%% file: main.tex
\documentclass[11pt]{article}
\usepackage[left=1in,right=1in,top=1in,bottom=1in]{geometry}

\usepackage[normalem]{ulem}
\usepackage[utf8]{inputenc} 
\usepackage[T1]{fontenc}    
\usepackage{hyperref}       
\usepackage{url}            
\usepackage{booktabs}       
\usepackage{amsfonts}       
\usepackage{nicefrac}       
\usepackage{microtype}      
\usepackage{xcolor}         
\usepackage{float}
\usepackage{graphicx, xcolor, ulem}
\usepackage{algorithm,algpseudocode}
\usepackage{multirow}
\usepackage{enumitem}
\usepackage{amsmath,amsfonts,latexsym,amssymb,amsbsy, amsthm}
\newcommand{\mb}[1]{\boldsymbol{#1}}

\newcommand \sg[1] {{\color{red}(#1)}}
\newcommand \hm[1] {{\color{blue}(Hassan: #1)}}

\newcommand \hme[1] {{\color{blue}#1}}
\newcommand \jme[1] {{\color{blue}#1}}

\definecolor{blue}{rgb}{0,0,0}
\definecolor{red}{rgb}{0,0,0}

\usepackage{mathtools}

\newtheorem{theorem}{Theorem}
\newtheorem{proposition}{Proposition}
\newtheorem{lemma}{Lemma}
\newtheorem{corollary}{Corollary}
\newtheorem{definition}{Definition}

\usepackage{thmtools, thm-restate}
\usepackage{setspace}
\usepackage{pifont}
\newcommand{\cmark}{\ding{51}}%
\newcommand{\xmark}{\ding{55}}%

\newcommand{\spn}{\mathrm{span}}
\newcommand{\vol}{\mathrm{vol}}

\newcommand{\cone}{\mathrm{cone}}
\newcommand{\relint}{\mathrm{relint}}
\newcommand{\R}{\mathbb{R}}
\newcommand{\ov}[1]{\overline{#1}}

\algnewcommand\INPUT{\item[{\textbf{Input:}}]}
\algnewcommand\RETURN{\item[{\textbf{Return:}}]}

\usepackage{authblk}
\DeclareMathOperator*{\argmin}{arg\,min}
\DeclareMathOperator*{\argmax}{arg\,max}

\usepackage[symbol]{footmisc}

\allowdisplaybreaks 
\newcommand{\innerprod}[2]{\left \langle #1,  #2 \right \rangle}

\newcommand{\AAFW}{A$^2$FW}
\title{Reusing Combinatorial Structure: Faster Iterative Projections over Submodular Base Polytopes}

\author[1]{Jai Moondra}
\author[1]{Hassan Mortagy}
\author[1]{Swati Gupta}

\affil[1]{\small Georgia Institute of Technology \protect \\
{\small \tt \{jmoondra3,hmortagy,swatig\}@gatech.edu}}
\date{}

\begin{document}
\maketitle
\begin{abstract}
Optimization algorithms such as projected Newton's method, FISTA, mirror descent, and its variants enjoy near-optimal regret bounds and convergence rates, but suffer from a computational bottleneck of computing ``projections'' in potentially each iteration (e.g., $O(T^{1/2})$ regret of online mirror descent). On the other hand, conditional gradient variants solve a linear optimization in each iteration, but result in suboptimal rates (e.g., $O(T^{3/4})$ regret of online Frank-Wolfe). Motivated by this trade-off in runtime v/s convergence rates, we consider iterative projections of close-by points over widely-prevalent submodular base polytopes $B(f)$. We first give necessary and sufficient conditions for when two close points project to the same face of a polytope, and then show that points far away from the polytope project onto its vertices with high probability. We next use this theory and develop a toolkit to speed up the computation of iterative projections over submodular polytopes using both discrete and continuous perspectives. We subsequently adapt the away-step Frank-Wolfe algorithm to use this information and enable early termination. For the special case of cardinality-based submodular polytopes, we improve the runtime of computing certain Bregman projections by a factor of $\Omega(n/\log(n))$. Our theoretical results show orders of magnitude reduction in runtime in preliminary computational experiments.
\end{abstract}

\section{Introduction}\label{sec:intro}
  
\input{intro}

\section{Preliminaries} \label{sec:prelims}
 
\input{prelims}

\section{Asymptotic Properties of Convex minimizers and Euclidean Projections on General Polytopes} \label{sec:proj}
\input{props}

\section{Bregman Projections over Cardinality-based Submodular Polytopes}\label{sec:card-based} 
\input{card-based-new}

\section{Toolkit to Adapt to Previous Combinatorial Structure}
\label{sec:recovering}
\input{recovering}

\section{Adaptive Away-steps Frank-Wolfe (A\texorpdfstring{$^2$}{ }FW)} \label{sec: AFW with adaptive learning}
 
\input{aafw.tex}

\section{Computations} \label{sec:computations}
\input{app_computations}


\section{Conclusion}
\label{sec:limitations}
We proposed speeding up iterative projections using combinatorial structure inferred from past projections. We developed a toolkit to infer tight inequalities, reuse active sets, restrict linear optimization to faces of submodular polytopes and round solutions so that errors in projections do not significantly impact the overarching iterative optimization. This work focuses on theoretical results for speeding up Bregman projections over submodular polytopes. \textcolor{blue}{We showed that online Frank-Wolfe (OFW) runs $4$ to $5$ orders of magnitudes faster on average than Online Mirror Descent (OMD) with AFW as a subroutine to compute projections while having significantly higher regret (i.e., around 30 times as much as OMD). However, OFW was only $2$ to $3$ orders of magnitude faster on average than OMD with our combinatorially enhanced \AAFW\ used as a subroutine to compute projections, which is a $2$ order of magnitude reduction in runtime and significant progress in bridging between OFW and OMD computationally.} Though we speed up OMD by orders of magnitude in our preliminary experiments, this still is a long way from closing the computational gap with Online Frank Wolfe. Our work inspires many future research questions, e.g., procedures to infer tight sets on non-submodular polytopes such as matchings and procedures to round iterates to the nearest tight face for combinatorial polytopes. Nevertheless, we hope that our results can inspire future work that goes beyond looking at projection subroutines as black boxes.

\section*{Acknowledgments.}
The research presented in this paper was partially supported by the Georgia Institute of Technology ARC TRIAD fellowship and NSF grant CRII-1850182.

\singlespacing
\bibliographystyle{IEEEtran}
\bibliography{related}

\newpage

\appendix

\section{Missing proofs in Section \ref{sec:proj}}\label{App: Missing proofs for general projections}
\input{proofs-general-projections}

\section{Missing proofs in Section \ref{sec:card-based} and the PAV algorithm} \label{App: card-based}
\input{card-based}

\section{Missing proofs in Section \ref{sec:recovering}} \label{App: missing proofs in section 3}
\input{proofs-sec-3}

\section{Results for the A\texorpdfstring{$^2$}{ }FW Algorithm} \label{missing proof AFW}
\input{AAFW_proof}

\end{document}


\textbf{Proof of Theorem 4}:

\begin{proof}
    There are elements $e^{(j)}, e^{(j + 1)} \in E$ be such that $x_{e^{(j)}} - y_{e^{(j)}} = c_j$ and $x_{e^{(j + 1)}} - y_{e^{(j + 1)}} = c_{j + 1}$.
    
    Let $\tilde{E}_j = \{\tilde{x}_e - \tilde{y}_e : x_e - y_e \le c_j \}$. Then, we'll show that if $\gamma \in \tilde{E}_j$ and $\delta \in E \setminus \tilde{E}_j$, then
    \[
        \gamma < \tilde{x}_{e^{(j + 1)}} - \tilde{y}_{e^{(j + 1)}} \le \delta.
    \]
    That is, every element of the set $\tilde{E}_j$ is smaller that every element of the set $E \setminus \tilde{E}_j$.

    For any $e \in E$, there is an $i$ such that $x_e - y_e = c_i$. Then,
    \begin{equation}\label{eq: gradient-difference}
        |(\tilde{x}_e - \tilde{y}_e) - c_i| = |(\tilde{x}_e - \tilde{y}_e) - (x_e - y_e)| \le |\tilde{x}_e - x_e| + |y_e - \tilde{y}_e| \le d(\tilde{x}, x) + d(\tilde{y}, y) \le 2d(\tilde{y}, y) < 2 \epsilon,
    \end{equation}
    where $d$ is the Euclidean distance. The first inequality is the triangle inequality, the second inequality follows from definition, and the third inequality is true because projection on convex sets is non-expansive.
    
    Therefore, for any $e$ such that $c_i \le c_j$,
    \[
        \tilde{x}_e - \tilde{y}_e < c_i + 2 \epsilon \le c_j + 2\epsilon < c_{j + 1} - 2 \epsilon < \tilde{x}_{e^{(j + 1)}} - \tilde{y}_{e^{(j + 1)}}.
    \]
    The first and last inequalities follow from (\ref{eq: gradient-difference}), and the second and third inequalities follow by assumption.
    
    Similarly, if $c_i > c_j$, then $\tilde{x}_e - \tilde{y}_e \ge \tilde{x}_{e^{(j + 1)}} - \tilde{y}_{e^{(j + 1)}}$.
    
    This implies the following: every element of the set $\{\tilde{x}_e - \tilde{y}_e : x_e - y_e \le c_j \} = \tilde{E}_j$ is smaller than every element of $\{ \tilde{x}_e - \tilde{y}_e : x_e - y_e > c_j \} = E \setminus \tilde{E}_j$.
    
    The result is then implied using Theorem \ref{foo_for_base}.
\end{proof}

\textbf{Proof of Theorem 5 and Algorithm 1}:
\begin{proof}
    This optimization problem is a linear problem:
    \begin{align*}
        \max c^\top x && \;\text{s.t.} \\
        x(T) &\le f(T) &\forall \; T \subseteq E, \\
        x(S_i) &= f(S_i) &\forall \; S_i \in \mathcal{S}, \\
        x(E) &= f(E).
    \end{align*}
    We'll use the dual problem and strong duality:
    \begin{align*}
        \min_{y} &\sum_{T \subseteq E} y_T f(T) &\text{s.t.} \\
        \sum_{T \ni e_j} y(T) &= c(e_j) &\forall \; j \in [1, n], \\
        y_T &\ge 0 \quad &\forall \; T \not\in \mathbf{S}.
    \end{align*}
    
    We show that the following vector $y^*$ is such that $\sum_{T \subseteq E} y^*_T f(T) = c^\top x^*$ (and that $y^*$ and $x^*$ are feasible), so optimality is implied by duality. Define $y^*$ as:
    \begin{align*}
        y^*_{U_j} &= c(e_j) - c(e_{j + 1})  & \forall \; j \in [1, n - 1], \\
        y^*_{U_n} &= c(e_n), \\
        y^*_{T} &= 0  & \forall \; T \subseteq E: T \not\in \{U_0, \ldots, U_n\}.
    \end{align*}
    
    Note that $\sum_{T \ni e_j} y^*_T = \sum_{\ell \in [j, n]} y^*_{U_\ell} = c(e(j))$. When $T \not\in \{U_1, \ldots, U_n\}$, $y^*_T \ge 0$ trivially. For some $j$, when $U_j \not\in \mathbf{S}$, $y^*_{U_j} \ge 0$ by definition of the order on $E$. Therefore $y^*$ is feasible.
    
    The feasibility of $x^*$ is essentially the same as in the proof of Edmonds' greedy algorithm, and we omit the details. Finally, a straightforward calculation verifies that $\sum_{T \subseteq E} y^*_T f(T) = c^\top x^*$, which proves our claim. 
\end{proof}

\textbf{Proof of Lemma 1}:
\begin{proof}
    The proof of this lemma utilizes the same ideas as those in the proof of Theorem 4.     There are elements $e^{(j)}, e^{(j + 1)} \in E$ be such that $\nabla h(x)_{e^{(j)}}  = c_j$ and $\nabla h(x)_{e^{(j + 1)}} = c_{j + 1}$.
    
    Let $E_j = \{\nabla h(x^\ast)_e : \nabla h(x)_{e} \le c_j \}$. Then, we'll show that if $\gamma \in {E}_j$ and $\delta \in E \setminus {E}_j$, then
    \[
        \gamma < {x^\ast}_{e^{(j + 1)}} - {y}_{e^{(j + 1)}} \le \delta.
    \]
    That is, every element of the set $\tilde{E}_j$ is smaller that every element of the set $E \setminus \tilde{E}_j$.

    For each $e \in E$, there is an $i$ such that $\nabla h(x)_e = c_i$. Then,
    \begin{equation}\label{eq: gradient-difference-same-y}
        |\nabla h(x^\ast)_e - c_i| = |\nabla h(x^\ast)_e - \nabla h(x)_{e}|  \le L \|x -  x^\ast \|_2 < L \epsilon.
    \end{equation}
    
    Therefore, for any $e$ such that $c_i \le c_j$,
    \[
        \nabla h(x^\ast)_e < c_i + \epsilon \le c_j + \epsilon < c_{j + 1} - \epsilon < \nabla h(x^\ast)_{e^{(j + 1)}}.
    \]
    The first and last inequalities follow from (\ref{eq: gradient-difference-same-y}), and the second and third inequalities follow by definition.
    
    Similarly, if $c_i > c_j$, then $\nabla h(x^\ast)_e \ge \nabla h(x^\ast)_{e^{(j + 1)}}$. Together, these imply the following: every element of the set $E_j = \{\nabla h(x^\ast)_e: h(x)_e \le c_j \}$ is smaller than every element of $\{ \nabla h(x^\ast)_e : h(x)_e > c_j \} = E \setminus E_j$. The result is then implied using Theorem 1.
\end{proof}

\textbf{Proof of Lemma 2}:
\begin{proof}
    Since $x^\ast$ also satisfies $x^\ast(S) = f(S)$ for all $S \in \mathcal{S}$, we have by definition
    \[
        h(x^\ast) \ge h(\tilde{x}) = \min_x \{ h(x) | x(S) = f(S) \; \forall S \in \mathcal{S}\}.
    \]
    If $\tilde{x} \in B(f)$, we have by definition
    \[
        h(\tilde{x}) \ge \min_{x \in B(f)} h(x).
    \]
    The two equations together imply that $h(\tilde{x}) = h({x^\ast})$.
\end{proof}

\textbf{Proof of Lemma 4 and Algorithm 5}:
\begin{proof}
    For brevity, denote $|E| = n$. From equation \ref{recovering projection}, we have for each $e \in F_i$,
    \[
        x_e = \frac{f(\cup_{j \in [i]} F_i) - f(\cup_{j \in [i - 1]}) - y(F_i)}{|F_i|} + y_e.
    \]
    Since $f, y$ are integral, we have $x_e \in Q$ for all $e \in E$. Further, note that
    \[
        \min_{x, y \in Q, x \neq y} |x - y| = \min_{\ell_1, \ell_2 \in [n], k_1 \ell_2 \neq k_2 \ell_1} \Big| \frac{k_1}{\ell_1} - \frac{k_2}{\ell_2} \Big| = \min_{\ell_1, \ell_2 \in [n], k_1 \ell_2 \neq k_2 \ell_1} \frac{|k_1 \ell_2 - k_2 \ell_1|}{\ell_1 \ell_2} \ge   
        \frac{1}{n^2}.
    \]
    Therefore, there is a unique element of $Q$ that is within a distance of less than $\frac{1}{2n^2}$ from $x^\ast_e$. But by assumption, we have $|x_e - x^\ast_e | \le \|x - x^\ast \|_2 < \frac{1}{2n^2}$ for all $e \in E$, which implies that $\argmin_{s \in Q}|x_e - s|$ is singleton, so that the rounding can be done uniquely. Further, note that for all $r \in \mathbb{R}$,
    \[
        \min_{s \in Q}|r - s| = \min_{k \in [n]} \min_{s \in \frac{1}{k} \mathbb{Z}} |r - s| = \min_{k \in [n]} \min_{t \in \mathbb{Z}} |k \cdot r - t|,
    \]
    which implies the correctness of the algorithm.
\end{proof}

%% file: intro.tex
Though the theory of discrete and continuous optimization methods has evolved independently over the last many years, machine learning applications have often brought the two regimes together to solve structured problems such as combinatorial online learning over rankings and permutations \cite{Helmbold2009,Koolen2010,Yasutake2011,Gupta2016}, shortest-paths \cite{gyorgy2007line} and trees \cite{cesa2010active,rahmanian2018online}, regularized structured regression \cite{Jaggi2013}, MAP inference, document summarization \cite{Bach2011} (and references therein). One of the most prevalent forms of constrained optimization in machine learning is the use of iterative optimization methods such as online stochastic gradient descent, mirror descent variants, projected Newton's method, conditional gradient descent variants, fast iterative shrinkage-thresholding algorithm (FISTA). These methods repeatedly compute two main subproblems: either a projection (i.e., a convex minimization) or a linear optimization in each iteration. The former class of algorithms is known as  projection-based optimization methods (e.g., projected Newton's method, see Table \ref{tab:alg examples}), and they enjoy near-optimal regret bounds in online optimization and near-optimal convergence rates in convex optimization compared to projection-free methods. These projection-based methods however suffer form high computational complexity per iteration due to the projection subproblem \cite{Nemirovski1983,Beck2003,beck2017first,Audibert2013,Bubeck2014}. E.g., online mirror descent is near-optimal in terms of regret (i.e., $O(\sqrt{T})$) for most online learning problems, however it is computationally restrictive for large scale problems \cite{Srebro2011}. On the other hand, online Frank-Wolfe is computationally efficient, but has a suboptimal regret of $O(T^{2/3})$ \cite{hazan2020faster}. 

Discrete optimizers, in parallel, have developed beautiful characterizations of properties of convex minimizers over combinatorial polytopes, which typically results in non-iterative exact algorithms (upto solution of a univariate equation) for such polytopes. This theory however has not been properly integrated within the iterative optimization framework. Each subproblem within the above-mentioned iterative methods is typically solved from scratch \cite{nesterov2003introductory, Karimi_2016}, using a black-box subroutine, leaving a significant opportunity to speed-up ``perturbed'' subproblems using combinatorial structure. Motivated by these trade-offs in convergence guarantees and computational complexity, we ask if:
\begin{center}
    \textit{Is it possible to speed up iterative subproblems of computing projections over combinatorial polytopes by reusing structural information from previous minimizers?}
\end{center}

This question becomes important in settings where the rate of convergence is more impactful than the time for computation, for e.g., regret impacts revenue for online retail platforms. However, the computational cost of solving a non-trivial projection sub-problem from scratch every iteration is the reason why these methods have remained of ``theoretical'' nature. We investigate if one can speed up iterative projections by reusing combinatorial information from past projections. Our techniques apply to iterative online and offline optimization methods such as Projected Newton's Method, Accelerated Proximal Gradient, FISTA, and mirror descent variants.
\begin{table}[t]
\centering
    \begin{tabular}{|p{6.5 cm}|p{6 cm}|p{2.7 cm}|} \hline
       \textbf{Algorithm}  & \textbf{Subproblem solved} &  {\textbf{Steps for $\epsilon$-error}} \\ \hline \hline
      Vanilla Frank-Wolfe \cite{Jaggi2013}& LO over polytope & $O\left( \frac{LD^2}{\epsilon}\right )$  \\ \hline
       Away-steps Frank-Wolfe \cite{Lacoste2015}& LO over polytope and active sets & $O\left( \kappa \left(\frac{D}{\delta}\right)^2\log \frac{1}{\epsilon}\right )$  \\ \hline
      *Projected gradient descent \cite{Karimi_2016} & Euclidean projection over polytope &  $O\left( \kappa \log \frac{1}{\epsilon}\right )$ \\ \hline
       *Mirror descent (MD) \cite{Radhakrishnan2020} & Bregman Projection & { $O\left( \kappa \nu^2 \log \frac{1}{\epsilon}\right )$} \\ \hline
       *Projected Newton's method \cite{Karimi_2016} & Euclidean projection over polytope scaled by (approximate) Hessian &  $O\left( (\kappa \beta)^3 \log \frac{1}{\epsilon}\right )$ \\ \hline
       *Accelerated Proximal Gradient \cite{nesterov2003introductory} & Euclidean projection over polytope & { $O\left( \sqrt{\kappa} \log \frac{1}{\epsilon}\right )$} \\ \hline
       *Fast Iterative Shrinkage-Thresholding Algorithm (FISTA) \cite{beck2009fast} & Euclidean projection over polytope & { $O\left( \sqrt{\kappa} \log \frac{1}{\epsilon}\right )$} \\ \hline
    \end{tabular}
    \caption{\small List of iterative optimization algorithms which solve a linear or convex optimization problem in each iteration. Here, $\kappa := L/\mu$ is the condition number of the main optimization, $\nu$ is condition number of the mirror map used in MD, $D$ is the diameter of the domain, $\delta$ is the pyramidal width, $\beta \geq 1$ measures on how well the Hessian is approximated. Starred algorithms have dimension independent optimal convergence rates.}
    \label{tab:alg examples}
\end{table}
To give an example setup of our iterative framework, we consider the overarching optimization problem of minimizing a convex function $h: \mathcal{P} \to \mathbb{R}^n$ over a constrained set ${\mathcal{P}}\subseteq \mathbb{R}^n$ be (P1), which we wish to solve using a regularized optimization method such as mirror descent and its variants. Typically, in such methods, iterates $x_t$ are obtained by taking an unconstrained gradient step, followed by a projection onto $\mathcal{P}$. We will refer to a subproblem of computing a single projection as (P2). Note that (P1) can be replaced by an online optimization problem as well, and similarly the iterative method to solve (P1) can be any one of those in Table \ref{tab:alg examples}. 
\vspace{-1 pt}
\begin{equation*} \label{setup}
\textbf{(P1)} \quad 
\left.
\begin{aligned}
    &\min h(x)\\
    &\text{s.t.} ~ x \in \mathcal{P}
\end{aligned} \hspace{10pt} \right\} \hspace{-1 pt}
 \substack{\text{(P1) can be solved~~~}\\ \text{ iteratively using,~~~~~}\\\text{e.g., mirror descent:}} \quad  
\begin{aligned}
    &1.~~ y_t = x_t - \gamma_t \nabla h(x_{t-1})~~~~~~~ \\
 &2.~~ x_t = \argmin_{z \in \mathcal{P}} D_\phi(z,y_t) \quad \textbf{(P2)}
\end{aligned}
\end{equation*}

\begin{figure}[t]
    \vspace{-8 pt}
    \centering
    \includegraphics[scale = 0.35]{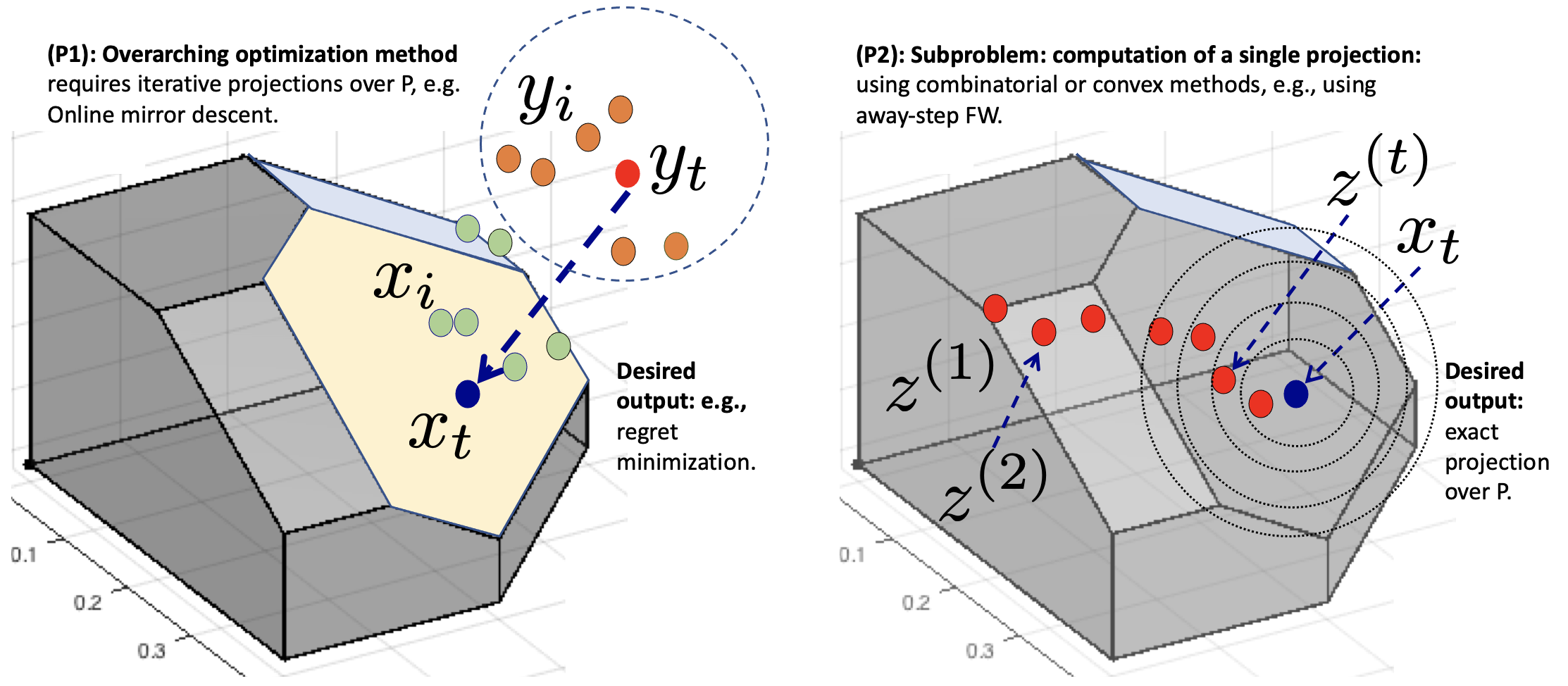}
    \caption{\footnotesize Left: {\bf (P1)} represents an iterative optimization algorithm that computes projections $x_i$ for points $y_i$ in every iteration (see Table \ref{tab:alg examples}). Right: {\bf (P2)} represents subproblem of computing a single projection of $y_t$ using an iterative method with easier subproblems, e.g., away-step Frank-Wolfe where $z^{(i)}$ are iterates during a single run of AFW and converge to projection $x_t$ (of $y_t$). The goal is speed up the subproblems using both past projections $x_1, \hdots, x_{t-1}$, as well as iterates $z^{(1)}, \hdots, z^{(k)}$. \vspace{-10 pt}}
    \label{fig:idea}
\end{figure}

To solve (P2), current literature aims to obtain arbitrary accuracy, to be able to bound errors in (P1) \cite{schmidt2011convergence}. We will refer to iterates in (P1) as $x_1, x_2, \hdots x_t$, and if (P2) is solved using an iterative method like Away-step Frank-Wolfe \cite{Marcotte1986}, we will refer to those iterates as $z^{(1)}, \hdots, z^{(k)}$ (depicted in Figure 1 (left, middle)). Our goal is to speed up the computation of $x_t$ by using the combinatorial structure of $x_1, \hdots, x_{t-1}, z^{(1)}, \hdots, z^{(k)}, y_1, \hdots, y_t$. To the best of our knowledge, we are the first to consider using the structure of previously projected points.

To capture a broad class of interesting combinatorial polytopes, we focus on submodular base polytopes. Submodularity is a discrete analogue of convexity, and captures the notion of diminishing returns. Submodular polytopes have been used in a wide variety of online and machine learning applications (see Table \ref{tab:problem examples} in appendix). A typical example is when $B(f)$ is permutahedron, a polytope whose vertices are the permutations of $\{1,\dots,n\}$, and is used for learning over rankings. Other machine learning applications include learning over spanning trees to reduce communication delays in networks, \cite{Koolen2010}), permutations to model scheduling delays \cite{Yasutake2011}, and $k$-sets for principal component analysis \cite{Warmuth2006}, background subtraction in video processing and topographic dictionary learning \cite{mairal2011convex}, and structured sparse PCA \cite{jenatton2010structured}. Other example applications of convex minimization over submodular polytopes include computation of densest subgraphs \cite{Nagano2011size}, bounds on the partition function of log-submodular distributions \cite{Djolonga2014} and distributed routing \cite{Krichene2015convergence}. 
 
Though (Bregman) projections can be computed efficiently in closed form for certain simple polytopes (such as the $n$-dimensional simplex), the submodular base polytopes pose a unique challenge since they are defined using $2^n$ linear inequalities \cite{Edmonds1971}, and there exist instances with exponential extension complexity as well \cite{Rothvoss2013} (i.e., there exists no extended formulation with polynomial number of constraints for some submodular polytopes). Existing combinatorial algorithms for minimizing separable convex functions over base polytopes typically require iterative submodular function minimizations (SFM) \cite{Groenevelt1991,Nagano2012,Gupta2016}, which are quite expensive in practice \cite{Jegelka2011,axelrod2020near}. However, these combinatorial methods highlight important structure in convex minimizers which can be exploited to speed up the continuous optimization methods. 

In this paper, we bridge discrete and continuous optimization insights to speed up projections. We first give a general characterization of similarity of cuts in cases where the points projected are close to the polytope as well when they are much further away (Section \ref{sec:proj}). We next focus on submodular polytopes, and show the following:
\begin{table}[t]
\centering
    \begin{tabular}{|p{6.5 cm}|p{6.5 cm}|p{2.5 cm}|} \hline
       \textbf{Problem}  &  \textbf{Submodular function, $S \subseteq E$ (unless specified)} &  { \textbf{Cardinality-based}} \\ \hline \hline
      $k$ out of $n$ experts ($k$-simplex), $E = [n]$& $f(S) = \min\{|S|,k\}$ & {\hspace{25pt} \cmark \hspace{25pt} } \\ \hline
      $k$-truncated permutations over $E = [n]$& $f(S) = (n - k)|S|$ for $|S| \leq k$, $f(S) = k(n -
k) + \sum_{j = k +1}^{|S|} (n + 1 - s)$ if $|S| >  k$  & {\hspace{25pt} \cmark \hspace{25pt} }  \\ \hline
 $k$-forests on $G = (V,E)$ & $f(S) = \min \{|V(S)| - \kappa (S), k\}$, $\kappa(S)$ is number of
connected components of $S$ & {\hspace{25pt} \xmark \hspace{25pt} }  \\ \hline
Matroids over ground set $E$: $M = (E, \mathcal{I})$ & $f(S) = r_M(S)$, the rank function of $M$ & {\hspace{25pt} \xmark \hspace{25pt} }  \\ \hline
Coverage of $T$: given $T_1, \dots , T_n \subseteq T$ & $f(S) = \left| \cup_{i \in S} T_i \right |$, $E = \{1,\dots , n\}$ & { \hspace{25pt} \xmark \hspace{25pt} }  \\ \hline
Cut functions on a directed graph $D = (V, E)$, $c:
E \to \mathbb{R}_+$ & $f(S) =c(\delta^{\mathrm{out}}(S))$, $S\subseteq V$ & {\hspace{25pt} \xmark \hspace{25pt} }  \\ \hline
    \end{tabular}
    \caption{ Problems and the submodular functions (on ground set of elements $E$) that give rise to them.}
    \label{tab:problem examples}
\end{table}
\begin{table}[t]
\vspace{-10 pt}
\centering
    \begin{tabular}{|p{5.5 cm}|p{5 cm}|p{5 cm}|} \hline
       \textbf{ Mirror Map $\phi(x) = \sum \phi_e(x_e)$}  & { $D_\phi(x,y)$} &  { \textbf{Divergence}}  \\ \hline \hline
      $\|x\|^2/2 $ & $\sum_{e } (x_e - y_e)^2$ & {Squared Euclidean Distance} \\ \hline
       $\sum_e x_e \log x_e - x_e $ & $\sum_e (x_e \log (x_e/y_e) - x_e + y_e)$ & Generalized KL-divergence \\ \hline
        $-\sum_e \log x_e $ & $\sum_e (x_e \log (x_e/y_e) - x_e + y_e)$ & Itakura-Saito Distance \\ \hline
         $\sum_e (x_e \log x_e + (1-x_e)\log (1- x_e))$ & $\sum_e (x_e \log (x_e/y_e) + (1-x_e) \log ((1-x_e)/(1-y_e))$ & Logistic Loss \\ \hline
    \end{tabular}
    \caption{ Examples of some popular uniform separable mirror maps and their corresponding divergences. \vspace{-12 pt}}
    \label{tab:diverexamples}
\end{table}
\begin{enumerate} [leftmargin = 12 pt]
\item[(i)]\textit{Bregman Projections over cardinality-based polytopes}: We show that the results of Lim and Wright \cite{lim2016efficient} on computing fast projections over the unit simplex in fact extend to {\it all} cardinality-based submodular polytopes (where $f(S) = g(|S|)$ for some concave function $g$). This gives an $O(n \log n)$-time algorithm for computing a Bregman projection, improving the current best-known $O(n\log n + n^2)$ algorithm \cite{Gupta2016}, in Section \ref{sec:card-based}. These are exact algorithms (up to the solution of a univariate equation), compared to iterative continuous optimization methods. 

\begin{figure}[t]
\vspace{-12 pt}
    \centering
    \includegraphics[scale = 0.3]{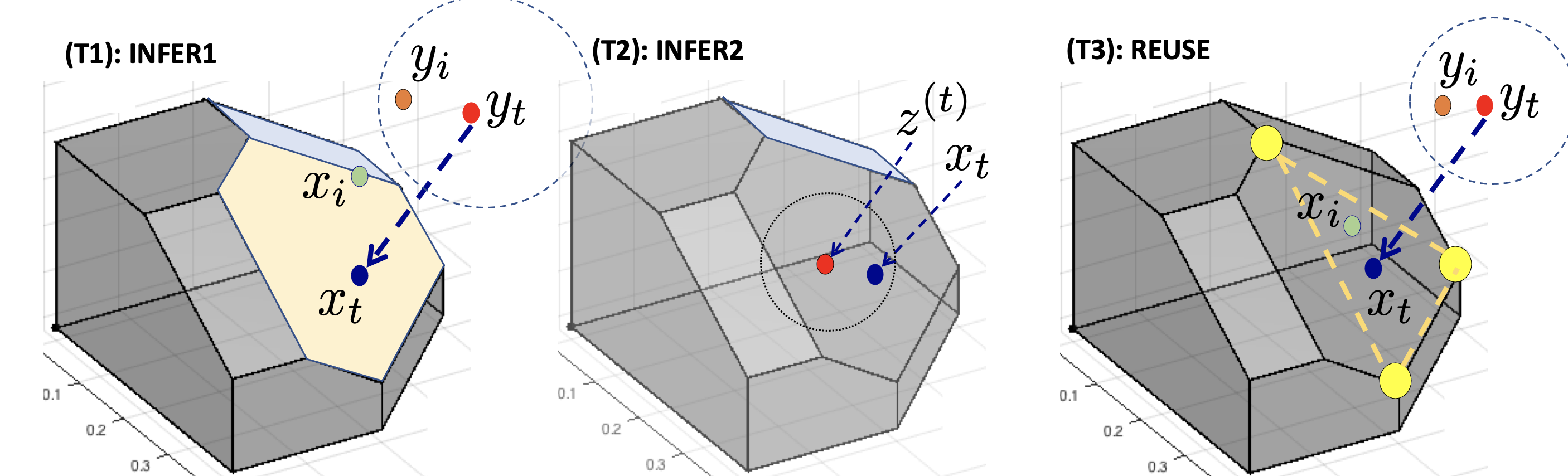}
    \includegraphics[scale = 0.3]{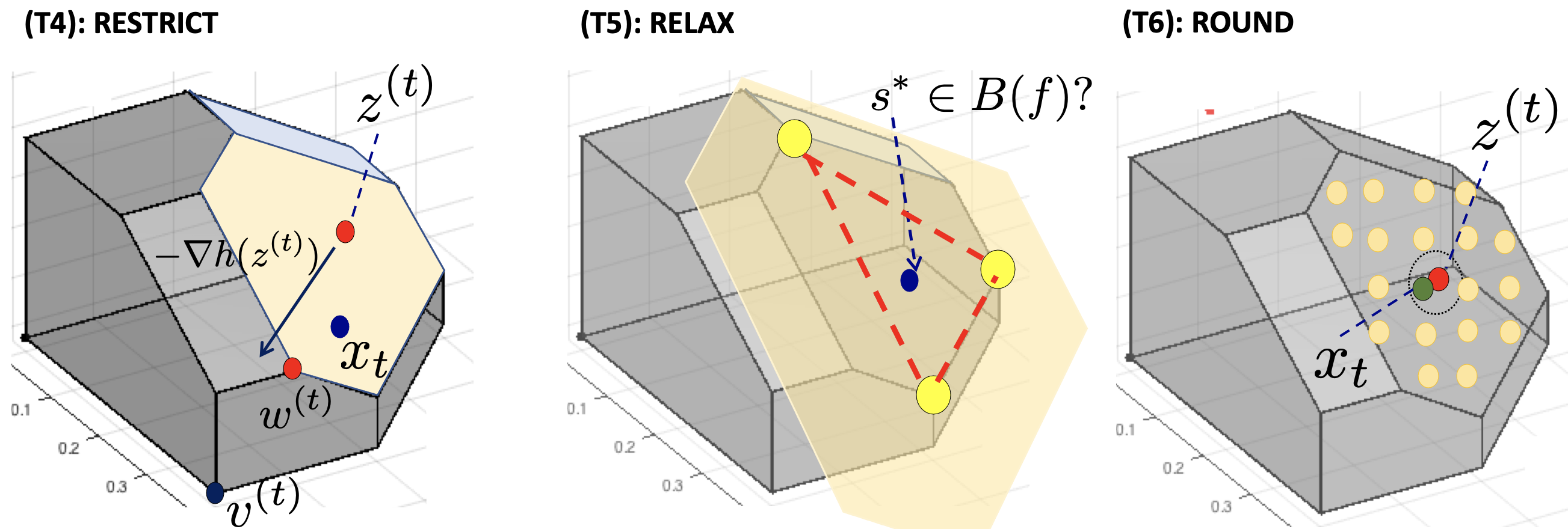}
    \caption{Toolkit to Speed Up Projections: \textsc{Infer1} {\bf (T1)} uses previously projected points to infer tight sets defining the optimal face of $x_t$ and is formally described by Theorem \ref{lemma: distant-gradients} (see also Figure \ref{fig:idea}-Right). On the other hand, \textsc{Infer2} {\bf (T2)} uses the closeness of iterates $z^{(t)}$ of an algorithm solving the projection subproblems (e.g. AFW) to the optimal $x_t$, to find more tight sets at $x_t$ (than those found by {\bf (T1)} (Lemma \ref{lemma: close-point rounding}). \textsc{ReUse} {\bf (T3)} uses active sets of previous projections computed using AFW (Lemma \ref{lemma: reusing active sets}). \textsc{Restrict} {\bf (T4)} restricts the LO oracle in AFW to the lower dimensional face defined by the tight sets found by {\bf (T1)} and {\bf (T2)} (Theorems \ref{lo over faces}, \ref{afw conv}). Note that the restricted vertex $w^{(t)}$ gives better progress than the orginal FW vertex $v^{(t)}$. \textsc{Relax} {\bf (T5)} enables early termination of algorithms solving projection subproblems (e.g. AFW) as soon as all tight sets defining the optimal face are found (Theorem \ref{lemma: comb-point rounding}). Finally, \textsc{Round} {\bf (T6)} gives an integral rounding approach for special cases (Lemma \ref{lemma: Integer_euclidean_rounding}). \vspace{-10 pt}}
    \label{fig:tools2}
\end{figure}

\item[(ii)] \textit{Toolkit for Exploiting Combinatorial Structure:} We next develop a toolkit (tools {\bf T1-T6}) of provable ways for detecting tight inequalities, reusing active sets, restrict to optimal inequalities and rounding approximate projections to enable early termination:
\begin{enumerate}  [leftmargin = 16 pt]
    \item[(a)] {\sc Infer}: We first show that for ``close'' points $y, \tilde{y}$ where the projection $\tilde{x}$ of $\tilde{y}$ on $B(f)$ is known, we can infer some tight sets for $x$ using the structure of $\tilde{x}$ without explicitly computing $x$ ({\bf T1}). Further, suppose that we use a convergent iterative optimization method to solve the projection subproblem (P2) for $y_t$ to compute $x_t$, then given any iterate $z^{(k)}$ in such a method, we know that $\|z^{(k)} - x_t\| \leq \epsilon_k$ is bounded for strongly convex functions. Using this, we show how to infer some tight sets (provably) for $x_t$ for small enough $\epsilon_k$ ({\bf T2}), in Section \ref{sec: recovering from previous}. 

    \item[(b)] {\sc ReUse}: Suppose we compute the projection $\tilde{x}$ of $\tilde{y}$ on $B(f)$ using AFW, and obtain an active set of vertices $A$ for $\tilde{x}$. Our next tool ({\bf T3}) gives conditions under which $A$ is also an active set for $x$. Thus, $x$ can be computed by projecting $y$ onto $\mathrm{Conv}(A)$ instead of $B(f)$ in Section \ref{sec: reuse and restrict}. 

    \item[(c)] {\sc Restrict}: While solving the subproblem (P2), we show that discovered tight inequalities for the optimum solution can be incorporated into the linear optimization (LO) oracle over submodular polytopes, in Section \ref{sec: reuse and restrict}. We modify Edmonds' greedy algorithm to do LO over any lower dimensional face of the submodular base polytope, while maintaining its efficient $O(n\log n)$ running time. Note that in general, while there may exist efficient algorithms to do LO over the entire polytope (e.g. shortest-paths polytope), restricting to lower dimensional faces may not be trivial. 

    \item[(d)] {\sc Relax} and {\sc Round}: We give two approaches for rounding an approximate projection to an exact one in Section \ref{sec: rounding}, which helps terminate iterative algorithms early. The first method uses {\sc Infer} to iteratively finds tight sets at projection $x_t$, and then checks if we have found all such tight sets defining the optimal face by projecting onto the affine space of tight inequalities. If the affine projection $x_0$ is feasible in the base polytope, then this is optimal projection. The second rounding tool is algebraic in nature, and applicable only to base polytopes of integral submodular functions. It only requires a guarantee that the approximate projection be within a (Euclidean) distance of  $1/(2n^2)$ to the optimal for Euclidean projections.
\end{enumerate}

\item[(iii)] {\it Adaptive Away-Step Frank-Wolfe (\AAFW):} We combine the above-mentioned tools to give a novel adaptive away-step Frank-Wolfe variant in Section \ref{sec: AFW with adaptive learning}. We first use {\sc Infer} ({\bf T1}) to detect tight inequalities using past projections of $x_{t-1}$. Next, we start away-step FW to compute projection $x_t$ in iteration $t$ by {\sc ReUsing} the optimal active set from computation of $x_{t-1}$. During the course of \AAFW, we {\sc Infer} tight inequalities iteratively using distance of iterates $z^{(t)}$ from optimal ({\bf T2}). To adapt to discovered tight inequalities, we use the modified greedy oracle ({\bf T4}). We check in each iteration if {\sc Relax} allows us to terminate early ({\bf T5}). In case of Euclidean projections, we also detect if rounding to lattice of feasible points is possible ({\bf T6}). We finally show three to six orders of magnitude reduction in running time of online mirror descent by using \AAFW~ as a subroutine for computing projections in Section \ref{sec:computations} and conclude with limitations in Section \ref{sec:limitations}.
\end{enumerate}

\paragraph{Related work.} Minimizing separable convex functions\footnote{Although our toolkit helps speed up iterative continuous optimization algorithms like mirror descent, the tools are general and can be used to speed up combinatorial algorithms like Groenvelt's Decomposition algorithm, Fujishige's minimum norm point, and Gupta et. al's Inc-Fix \cite{Fujishige1980b,Groenevelt1991,Gupta2016}. A special case of our rounding approach is used within the Fujishige-Wolfe minimum norm point algorithm to find approximate submodular function minimizers \cite{Chakrabarty2014}. } over submodular base polytopes was first studied by Fujishige \cite{Fujishige1980a} in 1980, followed by a series of results by Groenevelt \cite{Groenevelt1991}, Hochbaum \cite{Hochbaum1994}, and recently by Nagano and Aihara \cite{Nagano2012}, and Gupta et. al. \cite{Gupta2016}. Each of these approaches considers different problem classes, but uses $O(n)$ calls to either parametric submodular function or submodular function minimization, with each computation discovering a tight set and reducing the subproblem size for future iterations. Both subroutines, however, can be expensive in practice. Frank-Wolfe variants on the other hand have attempted at incorporating geometry of the problem in various ways: restricting FW vertices to norm balls \cite{Hazan2015,Garber_2013,lan_2013}, or restricting away vertices to best possible active sets \cite{Bashiri_2017}, or prioritizing in-face steps \cite{Grigas_2015}, or theoretical results such as \cite{Marcotte1986} and \cite{carderera2020second} show that FW variants must use active sets that containing the optimal solution after crossing a polytope dependent radius of convergence. These results, however, do not use combinatorial properties of previous minimizers or detect tight sets with provable guarantees and round to those. To the best of our knowledge, we are the first to adapt away-step Frank-Wolfe to consider combinatorial structure from previous projections, and accordingly obtain improvements over the basic AFW algorithm.
Although our \AAFW~ algorithm is most effective for computing projections (since we can invoke \emph{all} our toolkit for projections, i.e.({\bf T1-T6})), it is a \underline{standalone algorithm} for convex optimization over base polytopes that enables early termination with the exact optimal solution (compared to the basic AFW) via rounding ({\bf T5}) and improved convergence rates visa restricting ({\bf T4}). This might be of independent interest given the various applications mentioned above.

%% file: prelims.tex
\paragraph{Bregman Divergences.} Consider a compact and convex set $\mathcal{X} \subseteq \mathbb{R}^n$, and let $\mathcal{D} \subseteq \mathbb{R}^n$ be a convex set such that $\mathcal{X}$ is included in its closure. A differentiable function $h$ is said to be {strictly convex over domain $\mathcal{D}$ if $h(y) > h(x) + \innerprod{\nabla h (x)}{y - x}$} for all $x,y \in \mathcal{D}$. Moreover, a differentiable function $h$ is said to be $\mu$-strongly convex over domain $\mathcal{D}$ with respect to a norm $\|\cdot \| $ if $h(y) \geq h(x) + \innerprod{\nabla h (x)}{y - x} + \frac{\mu}{2} \| y - x\|^2$ for all $x,y \in \mathcal{D}$. A distance generating function $\phi: \mathcal{D} \to \mathbb{R}$ is a strictly (or $\mu$- strongly) convex and continuously differentiable function over $\mathcal{D}$, and satisfies additional properties of divergence of the gradient on the boundary of $\mathcal{D}$, i.e., $\lim_{x\to\partial \mathcal{D}} \|\nabla \phi (x)\| = \infty$ (see \cite{Nemirovski1983,beck2017first} for more details). We further assume that $\phi$ is \textit{uniformly} separable: $\phi = \sum_{e} \phi_e$ where $\phi_e : \mathcal{D}_e \to \mathbb{R}$ is the same function for all $e \in E$. We use $\| \cdot \|$ to denote the Euclidean norm unless otherwise stated. We say $\phi$ is $L$- smooth if $\| \nabla \phi(x) - \nabla \phi (z)\| \leq L \|x - z\|$ for all $x,z \in \mathcal{D}$. The Bregman divergence generated by the distance generated function $\phi$ is defined as $D_\phi(x, y) := \phi(x)-\phi(y)-\innerprod{\nabla {\phi}(y)}{x-y}$. For example, the Euclidean map is given by $\phi = \frac{1}{2}\|x\|^2$, for $\mathcal{D} = \mathbb{R}^E$ and is 1-strongly convex with respect to the $\ell_2$ norm. In this case $D_\phi(x, y) = \frac{1}{2}\|x - y\|_2^2$ reduces to the Euclidean squared distance (see Table \ref{tab:diverexamples}). We use $\Pi_{\mathcal{X}}(y):= \argmin_{x \in \mathcal{X}} \|x - y\|_2^2$ to denote the Euclidean projection operator on $\mathcal{X}$. The \emph{normal cone at a point ${x} \in \mathcal{X}$} is defined as $N_{
   \mathcal{X}}({x}) \coloneqq \{ {y} \in \mathbb{R}^n : \innerprod{ {y}}{ {z} -  {x}} \leq 0\; \forall  {z} \in \mathcal{X}\}$, which can be shown to be the cone of the normals of constraints tight at $ {x}$ in the case that $\mathcal{X}$ is a polytope. Let $\Pi_P( {y}) = \argmin_{ {x} \in P} \frac{1}{2}\| {x}- {y}\|^2$ be the \emph{Euclidean projection operator}.
Using first-order optimality, 
\begin{equation} \label{first-order optimality of projections}
    \innerprod{ {y}- {x}}{ {z}- {x}} \leq 0   \quad \forall  {z} \in \mathcal{X} \quad \Longleftrightarrow \quad \; ( {y}- {x}) \in N_{\mathcal{X}}( {x}),
\end{equation}
which implies that $ {x} = \Pi_{\mathcal{X}}( {y}) $ if and only if $( {y} -  {x}) \in N_{\mathcal{X}}( {x})$, i.e., moving any closer to $y$ from $x$ will violate feasibility in $\mathcal{X}$. It is well known that the Euclidean projection operator over convex sets is non-expansive (see e.g., \cite{beck2017first}): $\|\Pi_{\mathcal{X}}( {y}) - \Pi_{\mathcal{X}}( {x})\| \leq \| {y} -  {x}\|$ for all $ {x}, {y} \in \mathbb{R}^n$. Further, we denote the Fenchel-conjugate of the divergence by $D^*_\phi(z, y) = \sup_{x \in \mathcal{D}} \{\innerprod{{z}}{{x} } - D_\phi(x, y)\}$ for any $z \in \mathcal{D}^*$, where $\mathcal{D}^*$ is the dual space to $\mathcal{D}$ (in our case since $\mathcal{D} \subseteq \mathbb{R}^n$, $\mathcal{D}^*$ can also be identified with $\mathbb{R}^n$).

\noindent 
\paragraph{Submodularity and Convex Minimizers over Base Polytopes.} Let $f:2^E \to \mathbb{R}$ be a submodular function defined on a ground set  of elements $E$ ($|E| = n$), i.e. $f(A) + f(B) \geq f(A \cup B) + f(A \cap B)$ for all $A, B \subseteq E$. Assume without loss of generality that $f(\emptyset) = 0$, $f(A) > 0$ for $A \neq \emptyset$ and that $f$ is monotone (i.e. $f(A) \leq f(B)\; \forall A \subseteq B \subseteq E$- for any non-negative submodular function $f$; we can consider a corresponding monotone submodular function $\bar{f}$ such that $P(f) = P(\bar{f})$ (see Section 44.4 of \cite{Schrijver})). We denote by $EO$ the time taken to evaluate $f$ on any set. For $x \in \mathbb{R}^E$, we use the shorthand $x(S)$ for $\sum_{e \in S}x(e)$, and by both $x(e)$ and $x_e$ we mean the value of $x$ on element $e$. Given such a submodular function $f$, the polymatroid is defined as $P(f) = \{x \in \mathbb{R}^E_+: x(S) \leq f(S)\, \forall \,  S \subseteq E\}$ and the base polytope as $B(f) = \{x \in \mathbb{R}^E_+: x(S) \leq f(S)\, \forall \,  S \subset E, \; x(E) = f(E)\}$ \cite{Edmonds1970}. A typical example is when $f$ is the rank function of a matroid, and the corresponding base polytope corresponds to the convex hull of its bases (see Table \ref{tab:problem examples}).

Consider a submodular function $f : 2^E \to \mathbb{R}$ with $f(\emptyset) = 0$, and let $c \in \mathbb{R}^n$. Edmonds gave the greedy algorithm to perform linear optimization $\max c^Tx$ over submodular base polytopes for monotone submodular functions. Order elements in $E = \{e_1, \ldots, e_n\}$ such that $c(e_i) \geq c(e_j)$ for all $i < j$. Define $U_i = \{e_1, \ldots, e_i\}$, $i \in \{0, \ldots, n\}$, and let $x^\ast(e_j) = f(U_j) - f(U_{j - 1})$. Then, $x^\ast = \max_{x \in B(f)} c^\top x$. Further, for convex minimizers of strictly convex and separable functions, we will often use the following characterization of the convex minimizers.   

\begin{restatable}[Theorem 4 in \cite{Gupta2016}]{theorem}{foo_for_base}
\label{foo for base}%
Consider any continuously differentiable and strictly convex function $h : \mathcal{D} \to \mathbb {R}$ and submodular function $f:2^E \to \mathbb{R}$ with $f(\emptyset) = 0$. Assume that $B(f) \cap \mathcal{D} \neq \emptyset$. For any $x^* \in  \mathbb{R}^E$, let
$F_1, F_2, \dots , F_l$ be a partition of the ground set $E$ such that $(\nabla h(x^*))_e = c_i$ for all $e \in F_i$ and $c_i < c_l$ for $i < l$. Then $x^* = \argmin_{x \in B(f)} h(x)$ if and only if $x^*$ lies on the face $H^*$ of $B(f)$ given by $H^*:= \{x \in B(f) \mid x(F_1 \cup F_2 \cup \dots  \cup F_i) = f(F_1 \cup F_2 \cup \dots  \cup F_i) \, \forall \, 1 \leq i \leq l\}$.
\end{restatable}

To see why this holds, note that the first order optimality condition states that at the convex minimizer $x^*$, we must have that $\nabla h(x^*) = c$ is minimized as a linear cost function over $B(f)$, i.e., $c^Tx^* \leq c^Tz$ for all $z\in B(f)$. However, linear optimization over submodular base polytopes is given by Edmonds' greedy algorithm, which simply raises elements of minimum cost as much as possible. This gives us the levels of the partial derivatives of $x^*$ as $F_1, F_2, \hdots F_k$, which form the   optimal face $H^*$ of $x^*$. For separable convex functions like Bregman divergences (in Table \ref{tab:diverexamples}), we can thus compute $x^*$ by solving univariate equations in a single variable if the tight sets $F_1, \hdots, F_k$ of $x^*$ are known. We equivalently refer to corresponding inequalities $x(F_i) = f(F_i)$ as the optimal tight inequalities. 

We next characterize properties of convex minimizers and projections on general polytopes.

%% file: props.tex
 {\color{blue}
 We are interested in solving an overarching optimization problem using a projection-based method (e.g.  mirror descent) (P1), which repeatedly solves projection subproblems (P2) and our goal is to speed up these. In most projection-based methods, the initial descent in the algorithm involves ``big'' gradient steps, and the final iterations are typically ``smaller gradient steps''. For example, in projected gradient descent (PGD) steps, one can show that two consecutive points 
$y_{t}$ and $y_{t+1}$ to be projected in problem (P2) satisfy $\| y_{t+1} - y_{t}\|^2 \leq \delta$ after $O\left( \frac{L}{\mu} \log \left(\frac{D^2}{\delta } \right) \right)$ iterations, where $L/\mu$ is the condition number of the function and $D$ is the diameter of the polytope. So towards the end of PGD, the gradient steps $y_t$ are closer to each other.

In this section, we give necessary and sufficient conditions for when two close-by points, i.e., a fixed point $y$ and another perturbed point $\tilde{y} \coloneqq y + \epsilon$ obtained by adding any noise $\epsilon$ will project onto the same face of the polytope; see Figure \ref{fig: proj_perturb} for an exmple. In addition, we show that when $\epsilon$ is a random noise (with an arbitrary distribution) and $y$ and $\tilde{y}$ are at a large enough distance from the polytope, then $y$ and $\tilde{y}$ will project to the same face of the polytope with high probability. Furthermore, for the special case when $y$ and $\tilde{y}$ are sampled in a ball with a large radius compared to the volume of the polytope, we show that $y$ and $\tilde{y}$ will project to vertices of the polytope with high probability. 
To obtain the results in this section, we prove structural properties of Euclidean projections and convex minimizers on general polytopes that might be of independent interest. We will further develop these results by exploiting the structure of submodular base polytopes in subsequent sections. 

First, we introduce useful notation. Given some $z \in \R^n$ and $R \ge 0$, let $B_{z}(R) \subseteq \R^n$ denote the ball of radius $R \ge 0$ centered at $z$, and denote $\vol\big(B_{0}(1)\big) = v_n$. For a polytope ${\mathcal{P}} \subseteq \R^n$, let $\mathcal{V}(\mathcal{P}), \mathcal{F}(\mathcal{P})$ denote the set of vertices, faces of $\mathcal{P}$ respectively, and let $\relint(\mathcal{P})$ denote the relative interior of $\mathcal{P}$. Also let $\mathcal{F}_{\setminus \mathcal{V}}(\mathcal{P}) = \mathcal{F}(\mathcal{P}) \setminus \mathcal{V}(\mathcal{P})$. For a face $F$ of $\mathcal{P}$, define $\Theta_{\mathcal{P}}(F)$ to the set of points in ${y} \in \R^n$ such that $F$ is the minimal face of $\mathcal{P}$ containing $\Pi_{\mathcal{P}}(y)$. Notice that for $F \in \mathcal{F}_{\setminus \mathcal{V}}$, $\Theta_\mathcal{P}(F)$ is the set of all points whose projection on $\mathcal{P}$ lies in the relative interior of $F$.




For a face $F \in \mathcal{F}(\mathcal{P})$ and a measurable set $S \subseteq \R^n$, we define $r_F(S)$ to be the fraction of points in $S$ that are in  $\Theta_{\mathcal{P}}(F)$, i.e., $r_F(S) = \vol\left(S \cap \Theta_{\mathcal{P}}(F)\right)/\vol(S)$. This is well-defined since (as we show in Lemma \ref{face cond}) $\Theta_{\mathcal{P}}(F)$ is measurable for all $F \in \mathcal{F}(\mathcal{P})$. For $T \subseteq \mathcal{F}(\mathcal{P})$, we abuse notation slightly and define $r_T(S) = \sum_{F \in T} r_F(S)$. Since $\mathcal{P}$ is the disjoint union $\mathcal{V}(\mathcal{P}) \cup \left( \bigcup_{F \in \mathcal{F}(\mathcal{P})} \relint(F) \right)$, $\R^n$ is the disjoint union $\bigcup_{F \in \mathcal{F}(P)} \Theta_{\mathcal{P}} (F)$, and therefore $\sum_{F \in \mathcal{F}(\mathcal{P})} r_F(S) = 1$.

\begin{figure}[t]
\vspace{-8 pt}
    \centering
    \includegraphics[scale = 0.15]{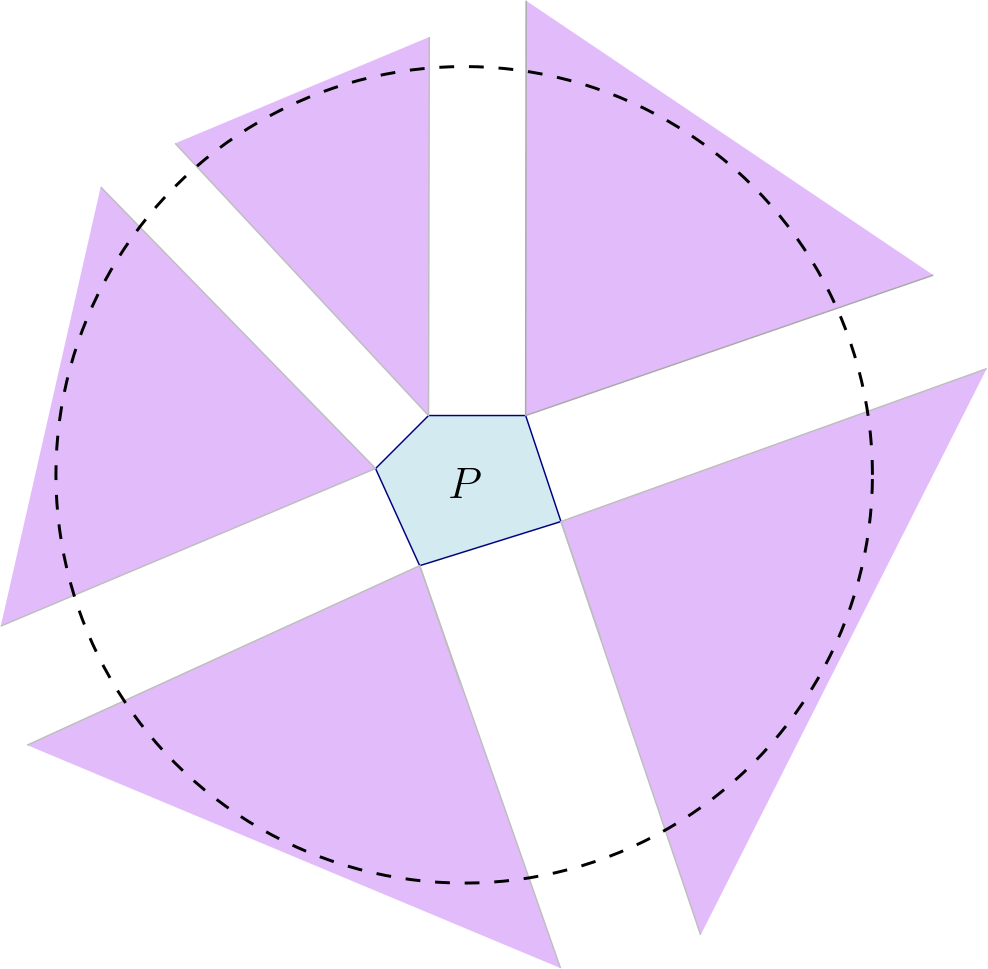}
    \hspace{10 pt}
    \includegraphics[scale = 0.2]{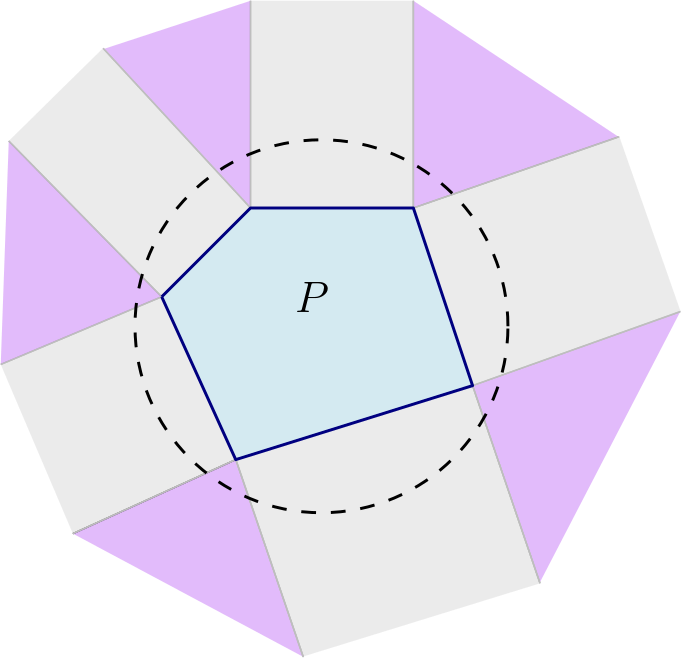}
    \includegraphics[scale = 0.35]{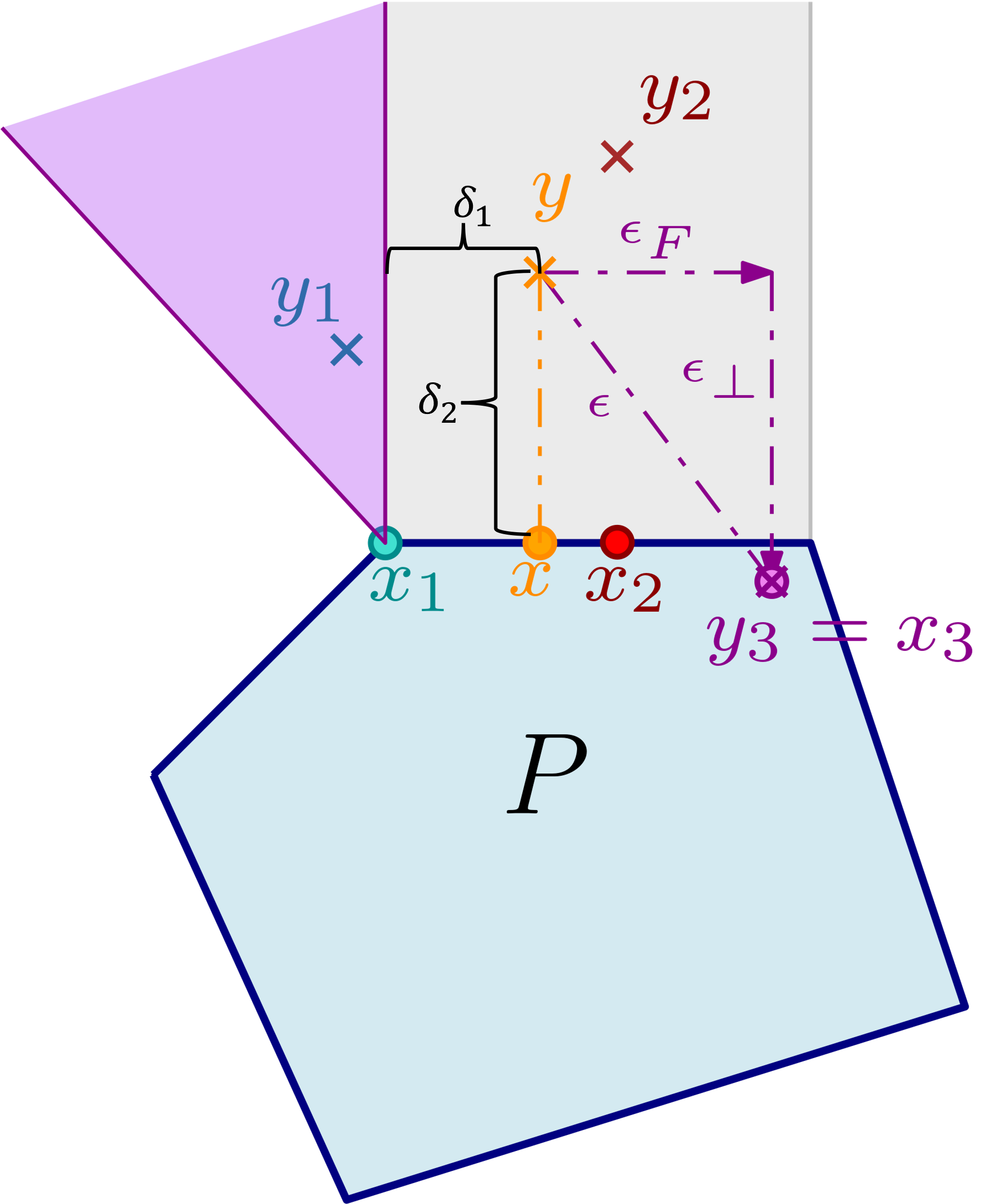}
    \caption{\textcolor{blue}{Left: An example showing that for a ball of far away points, most of those points would project to vertices of the polytope. In particular, as the radius of that ball goes to infinity, all points will project to a vertex in the limit. Middle: an example showing that this phenomenon is not true for points that are close to the polytope, where in fact most of these points would no longer project to vertices. 
    Right: An example showing the conditions of Theorem \ref{thm perturb}. We have an anchor point $y$ that projects to $x$, and perturbations $y_i$ and their projections $x_i$ for $i \in [3]$. In this example: $y_1$ satisfies condition (2) but not (1) and so we project to a different face; $y_2$ projects back to the same face because its perturbation is small enough as specified in Corollary \ref{corr face}; $y_3$ satisfies condition (1) but not (2). \vspace{-10 pt}}}
    \label{fig: proj_perturb}
\end{figure}

Our first lemma characterizes the points $\Theta_{\mathcal{P}}(F)$ that project to the face $F$ of polytope $\mathcal{P}$:

\begin{lemma} \label{face cond}
Let $\mathcal{P} = \{x \in \mathbb{R}^n: Ax \leq b\}$ be a polytope. Let $F = \{ x \in \mathcal{P} \mid A_{I} x =  b_I\}$ be a face of the polytope defined by setting the constraints in the index set $I$ to equality. Define $\mathrm{cone}(F):= \{A_I^\top \lambda \mid \lambda \geq 0\}$ to be the cone of the active constraints in $F$. Then, $\Theta_{\mathcal{P}}(F) = \mathrm{relint}(F) + \mathrm{cone}(F).$
\end{lemma}

\begin{proof}
    Using first-order optimality condition \eqref{first-order optimality of projections}, note that $x = \Pi_{\mathcal{P}}(y)$ if and only if $y - x \in \mathrm{cone}(F)$, where $\mathrm{cone}(F)$ is the normal cone at $x$. Thus $y \in \Theta_{\mathcal{P}}(F)  \Longleftrightarrow y - x \in \text{cone}(F)  \text{ for some $x \in \mathrm{relint}(F)$}.$
 \end{proof}

Consider points in any ball of radius $R$ in $\R^n$ (with an arbitrary center $z \in \R^n$). What fraction of the points in this ball project to the vertices of $\mathcal{P}$? The next theorem states that this fraction approaches $1$ as $R$ gets larger. This means that if one projects points on $\mathcal{P}$ within a large enough distance, most of them will project to the vertices\footnote{Most of the volume of a ball is concentrated near its boundary for large $n$.} of $\mathcal{P}$. We defer the proof of the following theorem to Appendix \ref{App: Missing proofs for general projections}.

\begin{restatable}{theorem}{farawaypoints}
\label{thm: most_far_away_points_project_to_vertices}
    For a fixed polytope $\mathcal{P} \subseteq \R^n$ and any constant $\delta \in (0, 1]$, there exists a radius $R_0 = \mu_{\mathcal{P}} \frac{n^2}{\delta}$, where $\mu_{\mathcal{P}}$ is a constant dependent on the polytope, such that for arbitrary $z \in \R^n$ we have
    \begin{equation*}
         r_{V(\mathcal{P})}\big(B_{z}(R)\big) \ge 1 - \delta \quad \text{ for all $R \ge R_0$.}
    \end{equation*}
\end{restatable}

Next, we show that perturbations of points at a large distance do not change the face that the corresponding projection lies on. Formally, consider the following experiment for any polytope $\mathcal{P} \subseteq \R^n$. Take any ball of radius $R \ge 0$ (with an arbitrary center $z \in \R^n$). Choose a $y$ uniformly from $B_{z}(R)$ and fix $\epsilon^\prime > 0$. Let $\epsilon$ be a random perturbation (with an arbitrary distribution) such that $\|\epsilon\| \leq \epsilon^\prime$ and set $\tilde{y} = y + \epsilon$. What is the probability that $\Pi_{\mathcal{P}}(y)$ and $\Pi_{\mathcal{P}}(\tilde{y})$ lie in the same minimal face of $\mathcal{P}$? Denote this probability by $\mathbb{P}_{\mathcal{P}}\left(B_z(R), \epsilon\right)$. That is, if some noise is added to $y$, does the projection on $\mathcal{P}$ still lie in the same minimal face of $\mathcal{P}$? Our next theorem helps answer this question for large enough $R$.

\begin{restatable}{theorem}{farawaypointsperturbation}
\label{thm: perturbations_dont_change_minimal_faces_at_large_distances}
    Let $\mathcal{P} \subseteq \R^n$ be a fixed polytope. For arbitrary $z \in \R^n, \epsilon > 0$, and $R \ge \sqrt{2 \pi} |\mathcal{F}(\mathcal{P})|^2 n^{3/2}\epsilon$, we have $\mathbb{P}_{\mathcal{P}}\left(B_z(R), \epsilon\right) \ge 1 - \frac{1}{n}$. In particular, $\lim_{R \to \infty} \mathbb{P}_{\mathcal{P}}\left(B_z(R), \epsilon\right) = 1$.
\end{restatable}

We next focus on the case of small perturbations. Consider a point $y \in \Theta_{\mathcal{P}}(F)$ for some face $F$ of a polytope. We now give a deterministic result that gives necessary and sufficient conditions under which a perturbed point $y^\prime := y + \epsilon$ projects back to the same face:

\begin{theorem} \label{thm perturb}
Let $\mathcal{P} = \{x \in \mathbb{R}^n: Ax \leq b\}$ be a polytope. Let $x$ be the Euclidean projection of some $y \in \mathbb{R}^n$ on $\mathcal{P}$. Let $I$ denote the index-set of active constraints at $x$ and $F = \{ x \in \mathcal{P} \mid A_{I} x =  b_I\}$ be the minimal face containing $x$. Define $\mathrm{cone}(F):= \{A_I^\top \lambda \mid \lambda \geq 0\}$ to be the cone of active constraints in $F$. Consider a perturbation of $y$ defined by $y^\prime := y + \epsilon$ where $\epsilon \in \mathbb{R}^n$, and let ${x}'$ be its Euclidean projection on $\mathcal{P}$. Decompose $\epsilon = \epsilon_F + \epsilon_\perp$, where $\epsilon_F$ and $\epsilon_\perp$ is the projection of $\epsilon$ onto the nullspace and rowspace of $A_I$, respectively. Then, $F$ will also be the minimal face containing $x'$ if and only if
\begin{enumerate}
    \item $x + \epsilon_F \in \mathrm{relint}(F)$, and
    \item $y - x + \epsilon_\perp \in \mathrm{cone}(F)$.
\end{enumerate}
Further, if $x^\prime \in F$ then $x^\prime = x + \epsilon_F$.
\end{theorem}

{It can be easy to see that if conditions 1 and 2 hold, then $x$ and $\tilde{x}$ lie on the same minimal face. However, the necessity of these conditions is harder to show since, in general, detecting the minimal face an optimal solution lies on is as difficult as finding the optimal solution. In particular, we show in Figure \ref{fig: proj_perturb} that even for very small perturbations $\epsilon$, the minimal face containing $x'$ might change, and so characterizing when these degenerate cases might happen is non-trivial.} To prove this theorem, we need the following result about minimizing strictly convex functions over polytopes, which states that if we know the optimal (minimal) face, then we can restrict the optimization to that optimal face, deferring its proof to the Appendix \ref{App: Missing proofs for general projections}. This result might be of independent interest.

\begin{restatable}[Reduction of optimization problem to optimal face]{lemma}{optimalface}\label{optimal face}
Consider any strictly convex function $h: \mathcal{D} \to \mathbb{R}$. Let $\mathcal{P} = \{x \in \mathbb{R}^n: Ax \leq b\}$ be a polytope and assume that $\mathcal{D} \cap \mathcal{P} \neq \emptyset$. Let $x^* = \argmin_{x \in \mathcal{P}} h(x)$, where uniqueness of the optimal solution follows from the strict convexity of $h$. Further, let $I$ denote the index-set of active constraints at $x^*$ and $\tilde{x} = \argmin_{x \in \mathbb{R}^n} \{h(x) \mid {A}_{I} x = {b}_{I}\}$. Then, we have that $x^* =  \tilde{x}$
\end{restatable}

%
We are now ready to prove our theorem:
\paragraph{Proof of Theorem \ref{thm perturb}.}
First, suppose that conditions (1) and (2) hold. Then, by summing conditions (1) and (2) (in the Minkowski sense) we have $y + \epsilon \in \mathrm{cone}(F) + \mathrm{relint}(F)$. Using Lemma \ref{face cond}, this implies that $F$ is the minimal face containing $x^\prime$.

Conversely, suppose that $F$ is the minimal face containing $x^\prime$, i.e. $x^\prime \in \mathrm{relint}(F)$. Thus, the index set of active constraints at $x^\prime$ is also $I$. Using Lemma \ref{optimal face} about reducing an optimization problem to the optimal face, we can reduce the projections of $y$ and $\tilde{y}$ on $\mathcal{P}$ to a projection onto the affine subspace containing $F$ as follows: 
\begin{equation}\label{reduced proj1}
x = \argmin_z \{\|y - z\|^2 \mid A_Iz = b\},  \qquad
x^\prime = \argmin_z \{\|y^\prime - z\|^2 \mid A_Iz = b\}.
\end{equation}
Assume without loss of generality that $A_I$ is full-row rank, since we can remove the redundant rows from $A_I$ without affecting our results. The solution to problems in \eqref{reduced proj1} could be computed in closed-form using standard linear algebra arguments for projecting a point onto an affine subspace:
\begin{align}
x &= y - A_I^\top (A_I A_I^\top)^{-1} (A_Iy - b_I) \label{reduced proj3} \\ 
x^\prime &= y^\prime - A_I^\top (A_I A_I^\top)^{-1} (A_Iy^\prime - b_I) = (y + \epsilon) - A_I^\top (A_I A_I^\top)^{-1} (A_I(y + \epsilon) - b_I). \label{reduced proj4}
\end{align}
where $(A_I A_I^\top)^{-1}$ is invertible since $A_I$ was assumed to be full-row rank. Thus, 
$$x^\prime - x =  \epsilon - A_I^\top (A_I A_I^\top)^{-1} A_I \epsilon = \epsilon_F,$$
where the last equality follows from the definition of $\epsilon_F$ being the projection of $\epsilon$ onto the nullspace of $A_I$. Thus, we have $x^\prime = x + \epsilon_F \in  \mathrm{relint}(F)$ and so condition (1) holds. 

We now show that condition (2) also holds. Using first-order optimality at $y^\prime$ we have
\begin{align*}
    y^\prime - x^\prime \in \mathrm{cone}(F) &\implies y + \epsilon_F + \epsilon_\perp - x^\prime \in \mathrm{cone}(F)  \implies y + \epsilon_F + \epsilon_\perp - (x + \epsilon_F) \in \mathrm{cone}(F)\\
     & \implies y + \epsilon_\perp - x \in \mathrm{cone}(F),
\end{align*}
so that condition (2) also holds.

As a corollary to this result, we show that if the point $y \in \Theta_{\mathcal{P}}(F)$ has some regularity conditions: the projection $x :=  \Pi_{\mathcal{P}}(y)$ has some distance to the boundary of $F$ and the normal vector $y - x$ is in the relative interior of $ \mathrm{Cone}(F)$, then points in small enough ball around $y$ will project back to $F$. We give an example of these conditions in Figure \ref{fig: proj_perturb} Right. In the corollary below, we use $\mathrm{rbd}(F) = F \setminus \relint(F)$ and $\mathrm{rbd}(\mathrm{cone}(F)\big) = \mathrm{cone}(F) \setminus \relint(\mathrm{cone}(F))$ to denote the relative boundary of $F$ and $\mathrm{cone}(F)$ respectively. Letting $I$ be the index set of active constraints at $F$ and $J$ be the index set of remaining indices, and noting that $F = \{x \mid A_Ix = b_I, A_J x \leq b\}$, we thus have $\relint(F) = \{x \mid A_Ix = b_I, A_J x < b\}$. Similarly, $\mathrm{cone}(F) = \{z \mid z = A_I^T \lambda, \lambda \geq 0\}$ and $ \relint(\mathrm{cone}(F)) = \{z \mid z = A_I^T \lambda, \lambda > 0$\}.

\begin{corollary} \label{corr face}
    Let $\delta_1 = d\big(x, \mathrm{rbd}(F)\big)$  and $\delta_2 = d\big(y - x, \mathrm{rbd}(\mathrm{cone}(F)\big)$. If $\|\epsilon\|_2 < \max \{\delta_1, \delta_2\}$, then $F$ is also the minimal face containing $x' = \Pi_P(y')$ where $y' = y + \epsilon$.
\end{corollary}

\begin{proof}
    As in Theorem \ref{thm  perturb}, write $\epsilon =\epsilon_F + \epsilon_\perp$, where $\epsilon_F$ is the component of $\epsilon$ on $\text{null}(F)$ and $\epsilon_\perp$ is the orthogonal component. We have that $\|\epsilon_F\|_2 \le \|\epsilon\|_2 < \delta_1$ and therefore $x + \epsilon_1$ lies in {$\text{relint}(F)$}. Similarly, $\|\epsilon_\perp\|_2 \le \|\epsilon\|_2 < \delta_2$, so that $y - x + \epsilon_\perp \in \mathrm{cone}(F)$. The result then follows by the theorem.
\end{proof}

We next focus on the special case when the submodular polytope is cardinality-based.}


%% file: card-based-new.tex
In this section, we improve the running time of exact combinatorial algorithms for computing uniform Bregman projections over cardinality-based submodular polytopes. The key observation that allows us to do that is the following generalization of Lim and Wright's result \cite{lim2016efficient}, which, to the best of our knowledge is the first result to explicitly state the relation between Bregman projections on general cardinality-based submodular polytopes and isotonic optimization: 

\begin{restatable}[Dual of projection is isotonic optimization] {theorem}{fencheldualityform}
\label{fenchel duality form}
Let $f : 2^E \to \mathbb{R}$ be a cardinality-based monotone submodular function, that is $f(S) = g(|S|)$ function for some nondecreasing concave function $g$. Let $c_i := g(i) - g(i-1)$ for all $ i\in [E]$. Let $\phi: \mathcal{D} \to \mathbb{R}$ be a strictly convex and uniformly seperable mirror map. Let $B(f) \cap \mathcal{D} \neq \emptyset$ and consider any $y \in \mathbb{R}^n$. 
{Let $\{e_1,\dots, e_n\}$ be an ordering of the ground set $E$ such that $y_1 \geq \dots \geq y_n$.} Then, the following problems are primal-dual pairs
\begin{equation} \label{primal dual}
(P) \quad 
\begin{aligned}
    &\min ~ D_\phi(x,y)\\
    &\text{subject to} ~ x \in B(f)
\end{aligned}
 \qquad  \qquad  (D) \quad 
\begin{aligned}
    &\max ~- D_\phi^*(z,y) + z^Tc\\
    &\text{subject to} ~ z_1  \leq \dots \leq z_n
\end{aligned}.
\end{equation}
Moreover, from a dual optimal solution $z^*$, we can recover the optimal primal solution $x^*$. 
\end{restatable}


To prove this result, we derive the Fenchel dual problem $(D)$ by using the structure of cardinality-based polytopes, and restricting the minimizer to the optimal face (see Appendix \ref{App: card-based}). Problem $(D)$ in \eqref{primal dual} is in fact a separable isotonic optimization problem (i.e. is of the form $\min \sum_{i=1}^n h_i(x_i)$ subject to $x_1 \leq x_2 \leq \dots \leq x_n$, where $h_i$ are univariate strictly convex functions), which highlights an interesting connection between projections on cardinality-based polytopes \cite{menon2012predicting,niculescu2005predicting,Bach2011}. In particular, when $\phi(x) = \frac{1}{2}\|x\|^2$, the dual problem $(D)$ in \eqref{primal dual} becomes the following $\min_z \{ \frac{1}{2} \| z - (c-y)\|^2 \mid  z_1  \leq \dots \leq z_n\}$ isotonic regression problem. 
Learning over projections is therefore dual to performing isotonic regression for perturbed data sets. Using the same algorithm as Lim and Wright's, i.e., the Pool Adjacent Violators (PAV) \cite{best2000minimizing}, we can solve the dual problem $(D)$ with a faster running time of $O(n \log n + n EO)$ compared to $O(n^2 + n EO)$ of \cite{Gupta2016}. We include the details about the algorithm and correctness in Appendix \ref{App: card-based}. It is worth noting that linear optimization over $B(f)$ also has a running time of $O(n \log n + nEO)$ using Edmonds' greedy algorithm \cite{Edmonds1971}. Therefore, for cardinality-based polytopes, {when solving the projection sub-problem (P2)}, it is better to use a combinatorial algorithm {(e.g. PAV)} than any iterative optimization method {(e.g. FW)}. Note that any FW iteration needs to sort the gradient vector (i.e., linear optimization over the base polytope) which is also $O(n\log n)$ in runtime. For cardinality-based polytopes, therefore, projection-based methods to solve (P1) are computationally competitive with conditional gradient methods.

%% file: recovering.tex
In the previous section, we gave an $O(n\log n)$ exact algorithm for computing  Bregman projections over cardinality-based polytopes. However, the pool-adjacent-violator algorithm is very specific to the cardinality-based polytopes and does not extend to general submodular polyhedra. To compute a projection over the challenging submodular base polytope, there are currently only two potential ways of doing so: (i) using Frank-Wolfe variants (due to simple linear sub-problems), (ii) using combinatorial algorithms such as those of \cite{Groenevelt1991,Nagano2012} (which typically rely on submodular function minimization for detecting tight sets). In this section, we construct a toolkit to speed up these approaches, and consequently speed up iterative projections over general submodular polytopes.

\subsection{\textsc{Infer} tight inequalities} \label{sec: recovering from previous}

We first present our {\sc Infer} tool {\bf T1} that recovers some tight inequalities of projection of $\tilde{y}$ by using the tight inequalities of the projection of a close-by perturbed point $y \in \mathbb{R}^n$. The motivation of this result stems from the fact that projection-based optimization methods often move slowly, i.e., points $y, \tilde{y}$ to be projected are often close to each other, and so are their corresponding projections $x, \tilde{x}$. Our first result is specifically for Euclidean projections.  

\begin{figure}[t]
    \centering
    \includegraphics[scale = 0.25]{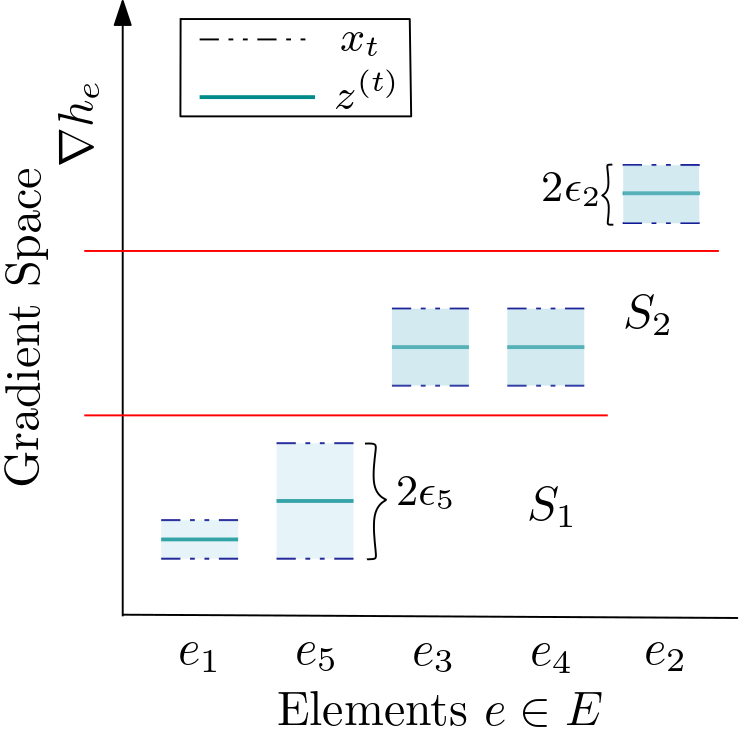}
    \hfill
    \includegraphics[scale=0.33]{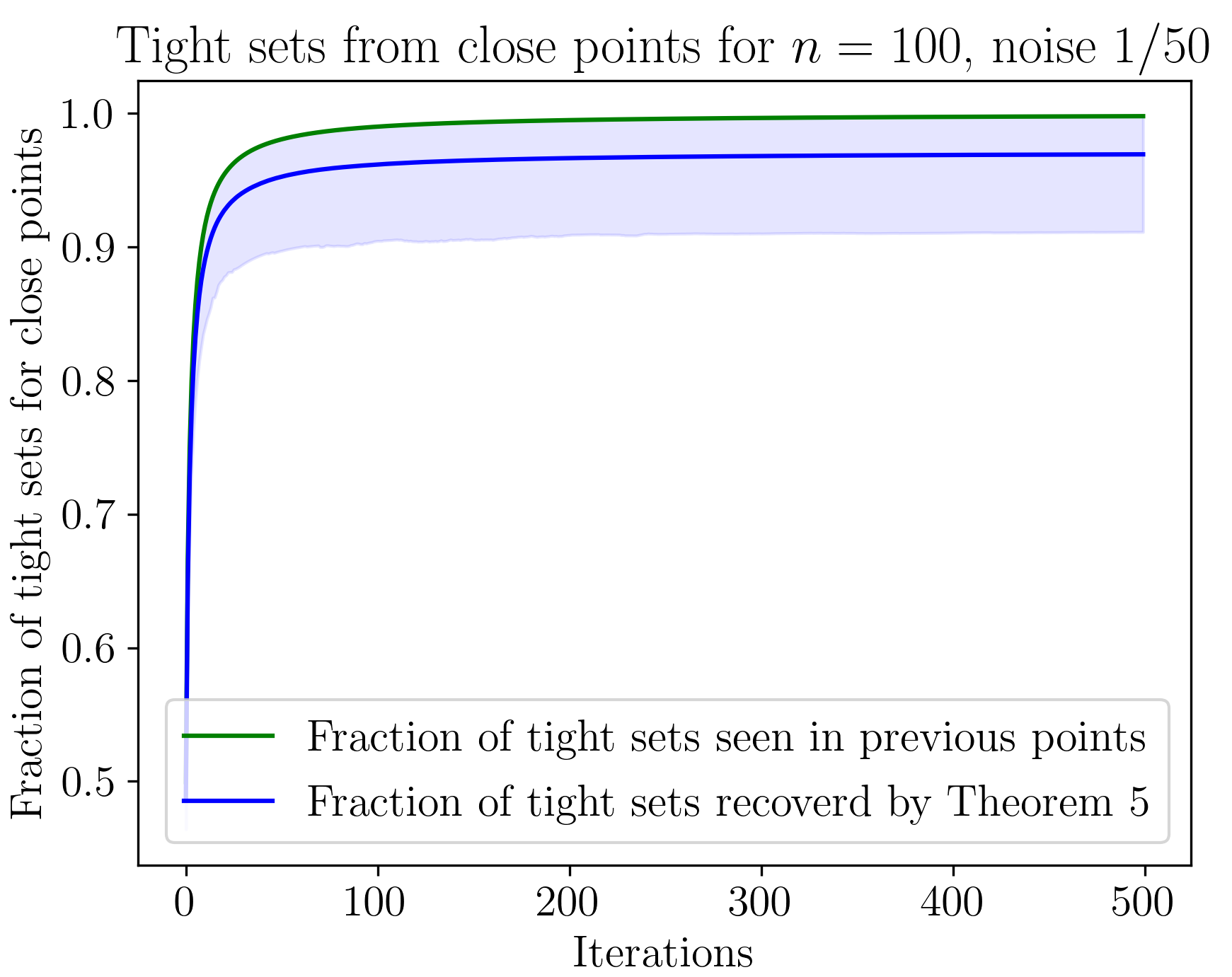}\label{fig: tight-sets-close-points}
    \caption{{Left:} 15-85\% percentile plot of fraction of tight sets inferred by using {\sc Infer1} (blue) v/s highest number of tight sets common for $i$th iterate compared to previous $i-1$ iterates (in green) for close points generated randomly using Gaussian noise, over $500$ runs. Right: 15-85\% percentile plot of fraction of tight sets inferred by using {\sc Infer1} (blue) v/s highest number of tight sets common for $i$th iterate compared to previous $i-1$ iterates (in green) for close points generated randomly using Gaussian noise, over $500$ runs.}
    \label{fig:grad}
\end{figure}

\begin{restatable}
[Recovering tight sets from previous projections {\bf (T1)}]{theorem}{distantgradients}
\label{lemma: distant-gradients}
Let $f:2^E \to \mathbb{R}$ be a monotone submodular function with $f(\emptyset) = 0$. Further, let $y$ and $\tilde{y} \in \mathbb{R}^E$ be such that $\|y -  \tilde{y}\| \leq \epsilon$, and $x, \tilde{x}$ be the Euclidean projections of $y, \tilde{y}$ on $B(f)$ respectively. Let $F_1, F_2, \dots , F_k$ be a partition of the ground set $E$ such that $x_e - y_e = c_i$ for all $e \in F_i$ and $c_i < c_l$ for $i < l$. If $c_{j + 1} - c_j > 4 \epsilon$ for some $j \in [k-1]$, then the set $S = F_1 \cup \dots \cup F_j$ is also a tight set for $\tilde{x}$, i.e. $\tilde{x}(S) = f(S)$.
\end{restatable}

Note that $x_e - y_e$ is the partial derivative of the distance function from $y$ at $x$. The proof shows that for $e \in E$, $\tilde{x}_e - \tilde{y}_e$ is close to $x_e - y_e$ and relies on the smoothness and non-expansivity of Euclidean projection. This helps us infer that the relative order of coordinates in $\tilde{x} - \tilde{y}$ (i.e., the coordinate-wise partial derivatives) is close to the relative order of coordinates in $x - y$. This relative order then determines tight sets for $x$, due to first-order optimality characterization of Theorem \ref{foo for base}. See Appendix \ref{app:missing in section 3.2} for a complete proof, where we also generalize the theorem to any Bregman projection that is $L$-smooth and non-expansive. As detailed in Section \ref{sec:computations}, we show that this theorem infers most of the tight inequalities computationally (see Figure \ref{fig:grad}-left).

Next, consider the subproblem (P2) of computing the projection $x_t$ of a point $y_t$. Let $z^{(k)}$ be the iterates in the subproblem that are convergent to $x_t$. The points $z^{(k)}$ grow progressively closer to $x_t$, and our next tool {\sc Infer} {\bf T2} helps us recover tight sets for $x_t$ using the gradients of points $z^{(k)}$.

\begin{restatable}
[Adaptively inferring the optimal face {\bf(T2)}]{theorem}{closepointrounding}
\label{lemma: close-point rounding}
Let $f:2^E \to \mathbb{R}$ be monotone submodular with $f(\emptyset) = 0$, $h: \mathcal{D} \to \mathbb{R}$ be a strictly convex and $L$-smooth function, where $B(f) \cap \mathcal{D} \neq \emptyset$. Let $x := \argmin_{z \in B(f)} h(z)$. Consider any $z \in B(f)$ such that $\|z - x\| \leq \epsilon$. Let $\tilde{F}_1, \tilde{F}_2, \dots , \tilde{F}_k$ be a partition of the ground set $E$ such that $(\nabla h(z))_e = \tilde{c}_i$ for all $e \in \tilde{F}_i$ and $\tilde{c}_i < \tilde{c}_l$ for $i < l$. Suppose $\tilde{c}_{j + 1} - \tilde{c}_j > 2L\epsilon$ for some $j \in [k-1]$. Then, $S = F_1 \cup \dots \cup F_j$ is tight for $x$, i.e. $x(S) = f(S)$.
\end{restatable}

\begin{algorithm}[t] \small
    \caption{Detect Tight Sets (\textbf{T2}): \textsc{Infer2}$(h,z,\epsilon)$}
\label{alg:tight sets}
\begin{algorithmic}[1]
\small
\INPUT Submodular function $f:2^E \to \mathbb{R}$, a function $h = \sum_{e \in E}h_e$, $Z \in B(f)$ such that $\|z- x^*\| \leq \epsilon$.
\State Initialize $\mathcal{S} = \emptyset$
\State   Let $\tilde{F}_1, \tilde{F}_2, \dots , \tilde{F}_k$ be a partition of $E$ such that $(\nabla h(z))_e = \tilde{c}_i \, \forall e \in F_i$ and $\tilde{c}_i < \tilde{c}_l$ for $i < l$. 
\For {$j \in [k-1]$}
\State \textbf{If} $\tilde{c}_{j + 1} - \tilde{c}_j > 2\epsilon$, \textbf{then} $\mathcal{S} = \mathcal{S} \cup\{F_1 \cup \dots \cup F_j\}$ \Comment{\texttt{we discovered a tight set at $x^*$}}
    \EndFor
\RETURN $\mathcal{S}$
\end{algorithmic}
\end{algorithm}

The proof of this theorem, similar to Theorem \ref{lemma: distant-gradients}, relies on the $L$-smoothness of $h$ to show that the relative order of coordinates in $\nabla h (x_t)$ is close to the relative order of coordinates in $\nabla h(z^{(k)})$, which helps infer some tight sets for $x$. See Appendix \ref{app:missing in section 3.2} for a complete proof and Figure \ref{fig:grad}-right for an example. Note that while Theorem \ref{lemma: distant-gradients} is restricted to Euclidean projections, Theorem \ref{lemma: close-point rounding} applies to any smooth strictly convex function.

\subsection{\textsc{ReUse} and \textsc{ Restrict}} \label{sec: reuse and restrict}\vspace{-0.2cm}

We now consider computing a single projection (P2) using Frank-Wolfe variants, that have two main advantages: (i) they maintain an active set for their iterates as a (sparse) convex combination of vertices, (ii) they only solve LO every iteration. Our first \textsc{Reuse} tool gives conditions under which a new projection has the same active set $A$ as a point previously projected, which allows for a faster projection onto the convex hull of $A$ (proof is included in Appendix \ref{app:missing in section 3.2}).
\begin{restatable}
[Reusing active sets \textbf{(T3)}]{lemma}{activesets}
\label{lemma: reusing active sets}
Let $\mathcal{P} \subseteq \mathbb{R}^n$ be a polytope with vertex set $\mathcal{V}(\mathcal{P})$. Let $x$ be the Euclidean projection of some $y \in \mathbb{R}^n$ on $\mathcal{P}$. Let $\mathcal{A} = \{v_1,\dots,v_k\} \subseteq \mathcal{V}(\mathcal{P})$ be an active set for $x$, i.e., $x = \sum_{i \in [k]}\lambda_i v_i$ for {$\|\lambda\|_1=1$} and $\lambda > 0$. Let $F$ be the minimal face of $x$ and $\Delta := \min_{v \in \partial\mathrm{Conv}(\mathcal{A})}\| x - v\|$ be the minimum distance between $x$ and the boundary of $\mathrm{Conv}(\mathcal{A})$. Then for all points $\tilde{y} \in \mathrm{cone}(F) + \mathrm{relint}(F)$ such $\tilde{y} \in \mathbb{B}_{\Delta}(y)$, where $\mathbb{B}_{\Delta}(y) = \{\tilde{y} \in \mathbb{R}^n \mid  \| \tilde{y} - y\| \leq \Delta\}$ is a closed ball centered at $y$, $\mathcal{A}$ is also an active set for the Euclidean projection of $\tilde{y}$.
\end{restatable}
 
In the previous section, we presented combinatorial tools to detect tight sets at the optimal solution. We now use our \textsc{Restrict} tool to strengthen the LO oracle in FW by restricting it to the lower dimensional faces defined by the tight sets we found (instead of doing LO over the whole polytope). Note that doing linear optimization over lower dimensional faces of polytopes, in general, is significantly harder (e.g., for shortest paths polytope). 
For submodular polytopes however, we show that we can do LO over any face of $B(f)$ efficiently using a modified greedy algorithm {(Algorithm \ref{alg:greedy})}. Given a set of tight inequalities, one can uncross these to form a {\it chain} of tight sets, i.e., any face of $B(f)$ can be written using a chain of subsets that are tight (see e.g. Section 44.6 in \cite{Schrijver2000combinatorial}). Given such a chain, our modified greedy algorithm then orders the cost vector in decreasing order so that it respects a given tight chain family of subsets. Once it has that ordering, it proceeds in the same way as in Edmonds' greedy algorithm \cite{Edmonds1971}. We include a proof of the following theorem in Appendix \ref{app:missing in section 3.2}.

\begin{algorithm}[t] \small
\caption{Greedy algorithm for faces of $B(f)$}
\label{alg:greedy}
\begin{algorithmic}[1]
\small
\INPUT Monotone submodular $f:2^E \to \mathbb{R}$, objective $c \in \mathbb{R}^n$, face $F = \{x \in B(f) \mid x(S_i) = f(S_i)$, where $S_1 \subset \dots \subset S_k = E \text{ where $S_i$ form a chain}\}$. 
\State Consider an ordering on the ground set of elements $E = \{e_1, \hdots, e_n\}$ such that (i) it respects the given chain, i.e., $S_i = \{e_1, \hdots, e_{s_i}\}$ for all $i$, and (ii) each set $S_i \setminus S_{i-1} = \{e_{s_{i-1} + 1}, \hdots, e_{s_{i}}\}$ is in decreasing order of cost, i.e.,
$c(e_{s_{i - 1} + 1}) \geq \ldots \geq c(e_{s_{i}})$. 
\State Let $x^*(e) := f(\{e_1, \ldots, e_j\}) - f(\{e_1, \ldots, e_{j - 1}\})$, for $i \in [n]$. 
\RETURN $x^* = \argmax_{x \in F} \innerprod{c}{x}$
\end{algorithmic}
\end{algorithm}

\begin{restatable}[Linear optimization over faces of $B(f)$ {\bf(T4)}]{theorem}{looverface} \label{lo over faces}
Let $f:2^E \to \mathbb{R}$ be a monotone submodular function with $f(\emptyset) = 0$. Further, let $F = \{x \in B(f) \mid x(S_i) = f(S_i) \text{ for } S_i \in \mathcal{S}\}$ be a face of $B(f)$, where $\mathcal{S} = \{S_1, \dots S_k | S_1 \subseteq S_2 \hdots \subseteq S_k\}$. Then the modified greedy algorithm (Alg. \ref{alg:greedy}) returns $x^* = \argmax_{x \in F} \innerprod{c}{x}$ in $O(n \log n + nEO)$ time.
\end{restatable} 

 \vspace{-0.2cm}
\subsection{\textsc{Relax} and \textsc{Round} for Early Termination} \label{sec: rounding}
\vspace{-0.2cm}
\input{rounding}

%% file: rounding.tex
Approximation errors in projection subproblems often impact (adversely) the convergence rate of the overarching iterative method unless the errors decrease at a sufficient rate \cite{schmidt2011convergence}. Our goal in this section is to detect if all tight sets at the optimum have been inferred, and enable early termination by computing the exact minimizer. In 2020, \cite{garber2020revisiting} gave primal gap bounds after which away-step FW reaches the optimal face, assuming strict complementarity assumption which need not hold even for computing a Euclidean projection. Further, \cite{diakonikolas2020locally}, showed that there exists some convergence radius $R$ such that for any iterate $z^{(t)}$ of AFW, if $\|z^{(t)} - x^*\| \leq R$, then any active set for $z^{(t)}$ must contain $x^*$, but {the} parameter $R$ existential and is non-trivial to compute. We complement these results by rounding our approximate projections to an exact one based on structure in partial derivatives.

\begin{algorithm}[t] \small
\caption{Combinatorial relaxed rounding (\textbf{T5}): \textsc{Relax}$(\mathcal{S},\mathcal{V})$}
\label{alg:round}
\begin{algorithmic}[1]
\small 
\INPUT Submodular function $f:2^E \to \mathbb{R}$, a function $h = \sum_{e \in E}h_e$, a chain of tight sets $\mathcal{S} = \{S_1,\dots,S_k\}$ where $S_1 \subset \dots \subset S_k = E$, and a set of vertices $\mathcal{V} = \{v_1, \dots v_l\}$ where $v_i$ is a vertex of $B(f)$.
\State Initialize $Flag = False$
\State Let $\tilde{x} := \argmin \{h(x) \mid x(S) = f(S) \, \forall S \in \mathcal{S}\}$ \Comment{\texttt{{could} be solved using Theorem \ref{foo for base}}}
\State \textbf{If} $\tilde{x} \in {\mathrm{Conv}}(\mathcal{V})$, \textbf{then}  $Flag = True$ \Comment{\texttt{we guessed optimal solution: $\tilde{x} = x^*$}}
\RETURN $\tilde{x}$, Flag
\end{algorithmic}
\end{algorithm}
\begin{algorithm}[t] \small
\caption{Integer-function rounding (\textbf{T6}): \textsc{Round}$(\mathcal{S},\mathcal{V})$}
\label{alg: int-round}
\begin{algorithmic}[1]
\small 
\INPUT Submodular function $f:2^E \to \mathbb{Z}$, a point $y \in \mathbb{Z}^E$, $x \in B(f)$ such that $|x_e - x^\ast_e| < \frac{1}{2|E|^2}$ for all $e \in E$, where $x^\ast = \Pi_{\mathcal{P}}(y)$ is the Euclidean projection of $y$ on $\mathcal{P}$.

\For{each $e \in E$}
    \State $z^{(i)} := \argmin_{s \in \frac{1}{i} \mathbb{Z}}|s - x_e|$, for each $i \in \{1, \ldots, |E|\}$.
    \State $z_e := \min_i z^{(i)}$
\EndFor
\State \textbf{Return} $z$
\end{algorithmic}
\end{algorithm}

Suppose that we have a candidate chain $\mathcal{S} = \{S_1, \dots S_k\}$ of tight sets (e.g., using {\sc Infer}). We observe that if the affine minimizer over $\mathcal{S}$, i.e., $\tilde{x} := \argmin \{h(x) \mid x(S) = f(S) \, \forall S \in \mathcal{S}\}$ is feasible in $B(f)$, then this is indeed the optimum solution $\tilde{x} = x^*$.
\begin{restatable}
[Rounding to optimal face \textbf{(T5)}]{lemma}{combpointrounding}
\label{lemma: comb-point rounding}
Let $f:2^E \to \mathbb{R}$ be a monotone submodular function with $f(\emptyset) = 0$. Let $h: \mathcal{D} \to \mathbb{R}$ be a strictly convex, where $B(f) \cap \mathcal{D} \neq \emptyset$. Let $x^* := \argmin_{x \in B(f)} h(x)$, and let $\mathcal{S} = \{S_1, \dots S_k\}$ contain some of the tight sets at $x^*$, i.e. $x^*(S_i) = f(S_i)$ for all $i \in [k]$. Further, let $\tilde{x} := \argmin \{h(x) \mid x(S) = f(S) \, \forall S \in \mathcal{S}\}$ be the optimal solution restricted to the face defined by the tight set inequalities corresponding to $\mathcal{S}$. {Then, $x^* = \tilde{x}$ iff $\tilde{x}$ is feasible in $B(f)$. In particular, if $\mathcal{S}$ contains all the tight sets at $x^*$, then $x^* = \tilde{x}$.}
\end{restatable}
The proof of this lemma can be found in Appendix \ref{app:missing in section 3.3}, and as a subroutine in Algorithm \ref{alg:round}. {We note that this holds for \emph{any polytope}: if we know that tight inequalities at the minimizer we can restrict the optimization problem to the face defined by those tight inequalities and ignore the other constraints defining the polytope (see Lemma \ref{optimal face}).} To check whether $\tilde{x} \in B(f)$ in general requires an expensive submodular function minimization, but instead we just check whether $\tilde{x}$ is in the convex hull of $\{v^{(1)},\dots,v^{(t)}\}$, where $v^{(i)}$ are the FW vertices of $B(f)$  that we have computed in Line 3 of Algorithm \AAFW~ up to iteration $t$.  Using \cite{diakonikolas2020locally}, we know that there will be a point at which the optimal solution is contained in the current active set.

We now present our second rounding tool \textsc{Round} for base polytopes of integral submodular functions. It only requires a guarantee that the approximate projection be within a (Euclidean) distance of $1/(2|E|^2)$ to the optimal projection. {This generalizes the  robust version of Fujishige’s theorem given in \cite{Chakrabarty2014}, connecting the MNP over $B(f)$ and the set minimizing the submodular function value}.

\begin{restatable}
[Combinatorial Integer Rounding Euclidean Projections \textbf{(T6)}]{lemma}{Euclideanrounding}
\label{lemma: Integer_euclidean_rounding}
Let $f:2^E \to \mathbb{\mathbb{Z}}$ $(|E|=n)$ be a monotone submodular function with $f(\emptyset) = 0$. Consider $y \in \mathbb{Z}^E$ and let $h(x) = \frac{1}{2}\| x - y\|^2$. Let $x^* := \argmin_{x \in B(f)} h(x)$. Consider any $x \in B(f)$ such that $\|x - x^*\| < \frac{1}{2n^2}$. Define $Q := \mathbb{Z} \cup \frac{1}{2} \mathbb{Z} \cup \ldots  \cup \frac{1}{n} \mathbb{Z}$, and for any $r \in \mathbb{R}$, let $q(r) := \argmin_{s \in Q} |r - s|$. Then, $q(x_e)$ is unique for all $e \in E$, and the optimal solution is given by $x^*_e = q(x_e)$ for all $e \in E$. \end{restatable}
This rounding algorithm runs in time $O(n^2 \log n)$ and is given in Algorithm \ref{alg: int-round}. The proof proceeds by showing that $x^*_e \in S$ for all $e \in E$, and that the distance between two points in $S$ is at least $\frac{1}{|E|^2}$, so that one can always round to $x^*$ correctly (complete proof is in Appendix \ref{app:missing in section 3.3}).

%% file: aafw.tex
We are now ready to present our Adaptive AFW (Alg. \ref{alg: AFW}) by combining tools presented in the previous section. First using the \textsc{Infer1}, we detect some of the tight sets $\mathcal{S}$ at the optimal solution before even running \AAFW, and accordingly warm-start \AAFW~ with $z_0$ in the tight face of $\mathcal{S}$. \AAFW~ operates similar to the away-step Frank-Wolfe, but during the course of the algorithm it restricts to tight faces as it discovers them (using {\sc Infer2}), adapts the linear optimization oracle (using {\sc Restrict}), and attempts to round to optimum (using {\sc Round, Relax}). To apply \textsc{Infer2} (subroutine included as Algorithm \ref{alg:tight sets}), we consider an iteration $t$ of A$^2$FW, where we have computed the FW gap $g_t^{\text{FW}} := \max_{v \in B(f)}\innerprod{-\nabla h(z^{(t)})}{v - z^{(t)}}$ (see line 5 in Algorithm \ref{alg: AFW}). When $h$ is $\mu$-strongly convex:\vspace{-0.2cm}
\begin{equation}\label{strong wolfe}
    \frac{\mu}{2}\| z^{(t)} - x^*\|^2 \leq h(z^{(t)}) - h(x^*) \leq \max \innerprod{-\nabla h(z^{(t)})}{v - z^{(t)}} = g_t^{\text{FW}},
\end{equation}
and so $\| z^{(t)} - x^*\| \leq \sqrt{2g_t^{\text{FW}}/\mu}$. Let $\tilde{F}_1, \tilde{F}_2, \dots , \tilde{F}_k$ be a partition of the ground set $E$ such that $(\nabla h(z^{(t)}))_e = \tilde{c}_i$ for all $e \in F_i$ and $\tilde{c}_i < \tilde{c}_l$ for all $i < l$. If $\tilde{c}_{j + 1} - \tilde{c}_j > 
{2L\sqrt{{2 g_t^{\text{FW}}}/{\mu}} }$ for some $j \in [k-1]$, then Theorem \ref{lemma: close-point rounding} implies that $S = F_1 \cup \dots \cup F_j$ is tight for $x^*$, i.e. $x^*(S) = f(S)$.

Overall in \AAFW, we maintain a set $\mathcal{S}$ containing all such tight sets $S$ at the optimal solution that we have found so far. We use those tight sets as follows: (i) we restrict our LO oracle to the lower dimensional face we identified using the modified greedy algorithm (\textsc{Restrict}- \textbf{(T4)}). (ii) We use our \textsc{Relax} (\textbf{(T5)}) tool to check weather we have identified all the tight-sets defining the optimal face (Lemma \ref{lemma: comb-point rounding}). If yes, then we round the current iterate to the optimal face and terminate the algorithm early. For (Euclidean) projections over an integral submodular polytope, we can also use our \textsc{Round} \textbf{(T6)} tool to round an iterate close to optimal without knowing the tight sets. Whenever the algorithm detects a new chain of tight sets $\mathcal{S}_{new}$, it is restarted from a vertex in $F(\mathcal{S}_{new})$, which possibly has a higher function value than the current iterate. However, this increase in the primal gap is bounded as $h$ is finite over $B(f)$ and can happen at most $n$ times; thus, these restarts do not impact the convergence rate. The pseudocode of \AAFW~ is included in Algorithm \ref{alg: AFW}. 

\noindent 
\textbf{Convergence Rate:} As depicted in {\bf (T4)} in Figure \ref{fig:tools2}, restricting FW vertices to the optimal face results in better progress per iteration during the latter runs of the algorithm. The convergence rate of \AAFW~ depends on a geometric constant $\delta$ called the pyramidal width \cite{Lacoste2015}. This constant is computed over the worst case face of the polytope. By iteratively restricting the linear optimization oracle to optimal faces, we improve this worst case dependence in the convergence rate. To that end, we define the restricted pyramidal width:
\begin{definition}[Restricted pyramidal width]
Let $\mathcal{P} \subseteq \mathbb{R}^{n}$ be a polytope with vertex set $\mathrm{vert}(\mathcal{P})$. Let $F \subseteq \mathcal{P}$ be any face of $\mathcal{P}$. 
Then, the pyramidal width restricted to to $F$ is defined as 
\begin{equation} \label{facial1}
    \rho_{F} := \min_{\substack{F^\prime \in faces(F) \\ x \in F^\prime \\ r \in \mathrm{cone}(F^\prime-x) \setminus \{0\}: 
    }} \min_{A \in \mathcal{A}(x)} \max_{v \in F^\prime,  a \in A} \innerprod{\frac{r}{\|r\|}}{v - a},
\end{equation}
where $\mathcal{A}(x) := \{A \mid A \subseteq \mathrm{vert}(\mathcal{P})$ such that $x$ is a proper
convex combination of all the elements in $A\}$.
\end{definition}

\begin{algorithm}[t] \small
\caption{{Adaptive Away-steps Frank-Wolfe (\AAFW)}}
\label{alg: AFW}
\begin{algorithmic}[1]
\INPUT Submodular $f:2^E \to \mathbb{R}$, $(\mu, L)$-strongly convex and smooth { 
$h 
:B(f) \to \mathbb{R}$}, chain of tight cuts $\mathcal{S}$ (e.g., using \textsc{Infer1}), $z^{(0)}\in B(f) \cap \{x(S) = f(S), S \in \mathcal{S}\}$ with active set $\mathcal{A}_0$,  tolerance $\varepsilon$.
\State {Initialize $t = 0, g^{\text{FW}}_0 = +\infty, v^{(0)} = z^{(0)}$}
\While{$g^{\text{FW}}_{t} \ge \varepsilon$}
       \State $\mathcal{S}_{new} = \mathcal{S} \cup \textsc{Infer2}(h,z^{(t)}, 2L\sqrt{{2 g_t^{\text{FW}}}/{\mu}})$ \Comment{\texttt{{use toolkit to find new tight sets}}}
    \State $\tilde{x}, Flag = \textsc{Relax}(\mathcal{S}_{new},\{v^{(0)} \dots v^{(t)}\})$ 
    \State \textbf{if} $Flag = True$, \Return $\tilde{x}$
    \If{$|\mathcal{S}_{new}| > |\mathcal{S}|$ {\textbf{and} $z^{(t)} \not \in F(\mathcal{S}_{new})$}
    }  \Comment{\texttt{round and restart}}
    \State Set $z^{(t+1)}  \in \argmin_{v \in F(\mathcal{S}_{new})} \innerprod{\nabla h(z^{(t)})}{v}$ and $\mathcal{A}_{t+1} =  z^{(t+1)}$
    \Else \Comment{\texttt{{do iteration of AFW restricted to $F(\mathcal{S})$}}}
        \State Compute $v^{(t)} \in \argmin_{v \in F(\mathcal{S})} \innerprod{\nabla h(z^{(t)})}{v}$ \Comment{ \texttt{ use toolkit}
    \State Compute away-vertex $a^{(t)} \in \argmax_{v \in \mathcal{A}_{t}} \innerprod{\nabla h(z^{(t)})}{v}$}
    \If{$g_t^{\text{A}} : =    \innerprod{-\nabla h(z^{(t)})}{ z^{(t)} - a^{(t)}} \leq g_t^{\text{FW}}$} \Comment{\texttt{FW gap v/s away gap}}
    \State $d_t : = v^{(t)} - z^{(t)}$ and $\gamma_t^{\max} := 1$. 
    \Else
    \State $d_t : = z^{(t)} - a^{(t)}$ and $\gamma_t^{\max} := {\lambda_{a^{(t)}}}/{(1-\lambda_{a^{(t)}})}$. 
\EndIf
\State Let $z^{(t+1)}: = z^{(t)} + \gamma_t d_t$ for $\gamma_t=  \argmin_{\gamma \in [0,\gamma_{\max}]} h(z^{(t)} + \gamma d_t)$
\State Update $\lambda_{v}$ for all $v \in  B(f)$ and $\mathcal{A}_{t+1} = \{v \in B(f) \mid \lambda_v > 0 \}$
\EndIf
\State Update $t := t + 1$ and $\mathcal{S} = \mathcal{S}_{new}$
\EndWhile 
\RETURN $z^{(t)}$ 
\end{algorithmic}
\end{algorithm}

In contrast, the pyramidal width $\delta$ is defined in a similar way as in \eqref{facial1}, but the first minimum is taken over $F^\prime \in faces(\mathcal{P})$ instead of $F^\prime \in faces({F})$. Thus, since $F \subseteq \mathcal{P}$ we have that $\rho_F \geq \delta$. For example, for the probability simplex (a submodular polytope; see Table \ref{tab:problem examples}), the pyramidal width restricted to a face $F$ is $2/\sqrt{\mathrm{dim}(F)}$ (assuming $\mathrm{dim}(F)$ is even for simplicity) \cite{Pena2015}. To the best of our knowledge, we are the first to adapt AFW to tight faces as they are detected. We establish the following result (proof in Appendix \ref{missing proof AFW}): 
\begin{restatable}[Convergence rate of \AAFW]{theorem}{afwconver} \label{afw conv}
Let $f:2^E \to \mathbb{R}$ be a monotone submodular function with $f(\emptyset) = 0$ and $f$ monotone. Consider any 
{function $h
: B(f) \to \mathbb{R}$} that is $\mu$-strongly convex and $L$-smooth. Let $x^* := \argmin_{x \in B(f)} h(x)$. Consider iteration $t$ of \AAFW~ and let $\mathcal{S}$ be the tight sets found up to iteration $t$ and $F(\mathcal{S})$ be face defined by these tight sets. Then, the primal gap  $w(z^{(t+1)}): = h(z^{(t+1)}) - h({x}^*)$ of \AAFW~ decreases geometrically at each step that is not a drop step (when we take an away step with a maximal step size so that we drop a vertex from the current active set) nor a restart step:
\begin{equation} \label{convergence rate}
   w({z}^{(t+1)}) \leq \left(1- \frac{\mu \rho_{F(\mathcal{S})}^2}{4LD^2}\right)w(z^{(t)}),
\end{equation}
where $D$ is the diameter of $B(f)$ and $\rho_{F(\mathcal{S})}$ is the pyramidal width of $B(f)$ restricted to $F(\mathcal{S})$.
Moreover, in the worst case, the number of iterations to get an $\epsilon$-accurate solution is $O\left(n \frac{L}{\mu} \left( \frac{D}{\rho_{B(f)}} \right)^2 \log \frac{1}{\epsilon}\right)$.
\end{restatable}


\textcolor{blue}{
\paragraph{Discussion.} In the worst case, our global linear convergence rate depends on the pyramidal width of the whole polytope $B(f)$. Note from \eqref{facial1} that $\rho_{B(f)} \leq \rho_{F(\mathcal{S})}$ for any chain $\mathcal{S}$, since $F(\mathcal{S}) \subseteq B(f)$. However, in practice, our algorithm does much better as we demonstrate computationally in the next section. This is because of our algorithm's adaptiveness and ability to exploit combinatorial structure, as is evident from the contraction rate in the primal gap in \eqref{convergence rate}. Whenever our algorithm detects a new chain of tight sets $\mathcal{S}_{new}$, it is restarted from a vertex in $F(\mathcal{S}_{new})$ (as long as the current iterate $z^{(t)} \notin F(\mathcal{S}_{new})$), which (possibly) has a higher function value than the current iterate. However, this increase in the primal gap is bounded as $h$ is finite over $B(f)$. Moreover, at that point, the iterates of the algorithm do not leave $F(\mathcal{S}_{new})$ and the contraction rate in the primal gap \eqref{convergence rate} depends on $\rho_{F(\mathcal{S}_{new})}$ instead of $\rho_{B(f)}$ (see Figure \ref{fig:aafw} for an example). Therefore, the tight sets are provably detected quicker and quicker as the algorithm proceeds. As soon as the algorithm detects all the tight sets it terminates early with an exact solution (which is important for our purpose of computing projections)\footnote{Suppose ${z}^{(t)}$ is rounded to the new face $F(\mathcal{S}_{new})$ by computing the Euclidean projection ${z}^{(t)}$ onto $F(\mathcal{S}_{new})$, i.e. ${z}^{(t+1)} := \argmin_{x \in F(\mathcal{S}_{new})}\|x - {z}^{(t)}\|^2$, we can show a bounded increase in the primal gap (Lemma \ref{lemma: rounding} in Appendix \ref{app: rounding}). This approach, however, might be computationally expensive, without any theoretical improvement in convergence.}. We finally remark that for small enough perturbations where project back to the same face (Corollary \ref{corr face}) or the same active set (Lemma \ref{lemma: reusing active sets}), the \AAFW~ will terminate in one iteration with an exact solution and we save up on the cost of computing a projection.} 

\textcolor{blue}{{\paragraph{Implications for Submodular Function Minimization.} Typically, the most practical way to minimize a submodular function $f$ is to solve the minimum norm point (MNP) over $B(f)$: $x^* = \argmin_{x \in B(f)} \|x \|^2$ over the base polytope. 
It is known that $S_1^* = \{e \in E \mid x_e^* < 0\}$ and $S_2^* = \{e \in E \mid x_e^* \leq 0\}$ are the minimal and maximal minimizers of $f$ respectively (i.e. $S_1 \subseteq \argmin_{S \subseteq E} f(S) \subseteq S_2$) \cite{Fujishige1980b}. 
When the submodular function is non-integral, AFW and Fujishige-Wolfe solve MNP \emph{approximately}, and hence can only result in approximate submodular function minimizers that don't distinguish between minimal and maximal minimizers. However, many applications require obtaining exact maximal (or minimal) submodular minimizers \cite{Groenevelt1991, Nagano2012, Gupta2016}. Our \AAFW~ algorithm remedies that since it rounds to an exact solution, where our rounding generalizes that of the Fujishige-Wolfe MNP algorithm. Furthermore, our convergence rate is better than AFW and Fujishige-Wolfe due to improving dependence on the pyramidal width as more tight sets are inferred \cite{Chakrabarty2014, Lacoste2015}.
}}

We next demonstrate computationally that exploiting combinatorial structure in our AFW algorithm with combinatorial rounding (Algorithm \ref{alg: AFW}) significantly outperforms the basic AFW algorithm by orders of magnitude.

\begin{figure}[t]
\vspace{-10 pt}
    \centering
    \includegraphics[scale = 0.45]{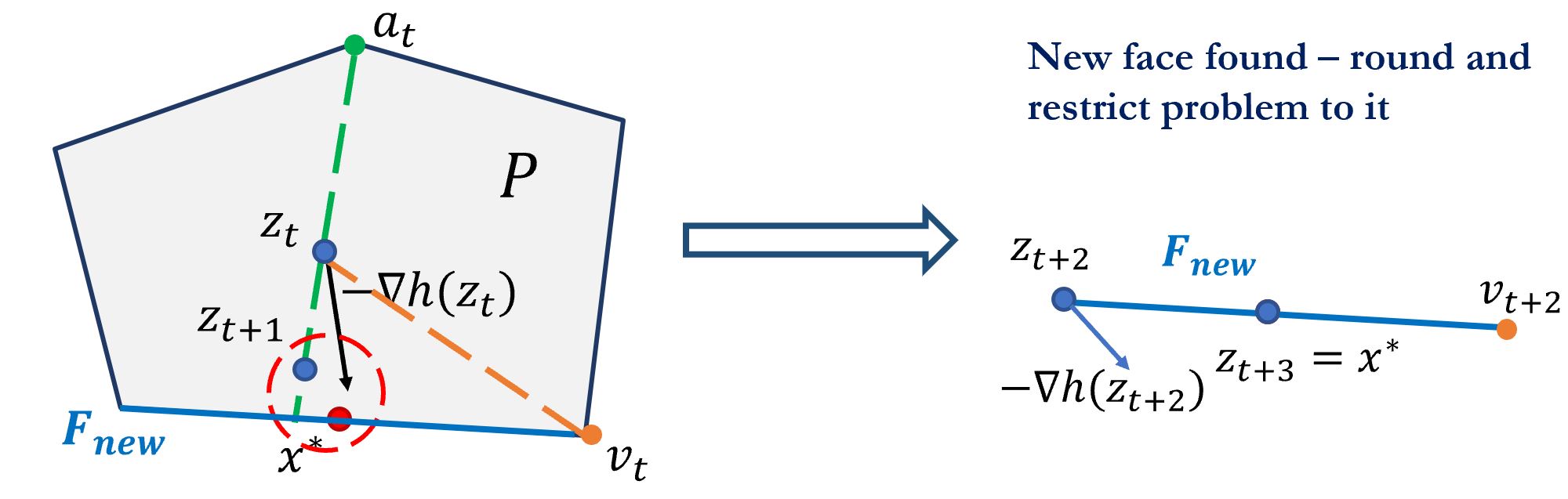}
    \caption{{\color{blue}{An example depicting the behavior of the iterates of the \AAFW~ algorithm, and how the algorithm constructs a sequence of lower dimensional faces containing the optimal solution $x^*$ and restricts to them. \vspace{-8 pt}}}}
    \label{fig:aafw}
\end{figure}

%% file: app_computations.tex
The code for our computations can be found on GitHub (\url{https://github.com/jaimoondra/submodular-polytope-projections}). We implemented all algorithms in Python 3.5+, utilizing numpy and scipy for some of our functions. We used these packages from the Anaconda 4.7.12 distribution as well as Gurobi 9 \cite{Gurobi2017} as a black
box solver for some of the oracles assumed in the paper. The first experiment was performed on a
16-core machine with Intel Core i7-6600U 2.6-GHz CPU and 256GB of main memory. The second experiment was performed by reserving 5 GB of memory for each run of the experiment on a 24-core Linux x86-64 machines (performed on the high-performance computing cluster of the Industrial and Systems Engineering department at the Georgia Institute of Technology).

We first show that computationally one can infer tight inequalities from past projections for perturbed candidate points using Theorem \ref{lemma: distant-gradients}. Next, we benchmark online mirror descent to learn over cardinality-based and general submodular polytopes. We show 2-6 orders of magnitude speedups using our toolkits in both settings (i.e., around $10^2$ to $10^4$ times faster over general submodular polytopes, and upto $10^6$ times faster over cardinality-based polytopes, compared to the away-step Frank-Wolfe.). Finally, we uncover some numerical issues with AFW variants and show how they can be mitigated, which might be of independent interest.

\subsection{First experiment: Recovering Tight Sets.} {We first show that one can infer tight inequalities from past projections for perturbed candidate points, and next that we can theoretically recover this approximately. To do this}, we consider $m = 500$ random points $y_1, \hdots, y_m$ obtained by perturbing a random $y_0 \in \mathbb{R}^{100}$ (where $y_0$ is itself sampled from a multivariate Gaussian distribution with mean $100$, standard deviation $100$) using multivariate Gaussian noise with mean zero and standard deviation $\epsilon = 1/50$. We compute the Euclidean projections of $y_0, y_1, \hdots, y_m$ (exactly) over the permutahedron. The results are plotted in Figure \ref{fig: experiment-main-plot}-left. Let $\mathcal{S}_i \subseteq 2^E$ represent the chain of tight sets for the projection of point $y_i$, where $E = \{e_1, \ldots, e_{100}\}$ is the ground set. The fraction of tight inequalities for each point $y_i$ that were already tight for some other previous point $y_0, \hdots, y_{i-1}$. The tight sets for the projection of $y_i$ that were also tight for a previous point in $y_1, \ldots, y_{i - 1}$ is then $\big|\mathcal{S}_i \cap \big(\bigcup_{j \in [i - 1]} \mathcal{S}_{j} \big) \big|$. {We plot the cumulative fraction of tight sets previously seen (i.e., $\frac{\sum_{i \in [k]}\big|\mathcal{S}_i \cap \big(\bigcup_{j \in [i - 1]} \mathcal{S}_{j} \big) \big|}{\sum_{i \in [k]} |\mathcal{S}_{i}|}$ 
against $k$, the number of points projected so far) to show that a large number of cuts are actually reused in future projections. We compare this with the fraction of tight sets inferred by Theorem \ref{lemma: distant-gradients}, and show that our theorem recovers near-optimal fraction of tight sets.}

The plots average over 500 independent runs of this experiment, while the shaded region is a 15-85 percentile plot across these runs. Note that our theoretical results give almost tight computational results, that is, we can recover most of the tight sets common between close points using Theorem \ref{lemma: distant-gradients}.

{\color{blue}
\subsection{Second experiment: Online Learning over  Cardinality-based Polytopes.} \label{sec: card experiments} Next, motivated by the trade-off in regret versus time for online mirror descent (OMD) and online Frank-Wolfe (OFW) variants, we conduct an online convex optimization experiment on the permutahedron (denoted by $B(f)$) with $n = 100$ elements. The loss functions in each iteration are (noisy) linear, and we use (i) Online Frank-Wolfe (OFW) and (ii) Online Mirror Descent (OMD) with the projection subproblem solved using Away-step Frank-Wolfe (AFW) and its variants enhanced by our toolkit.

We consider a time horizon of $T = 1000$, and consider two parameters $a,b$. We consider $a$ random permutations $\sigma_i$ ($i \in [a]$) close within a swap distance of $b$ from each other. We then define loss functions $\ell^{(t)}(x) = \innerprod{c^{(t)}}{x}$ for any $x\in B(f)$, where $c^{(t)}$ is the click-through-rate observed when $x$ is played in the learning framework. We construct $c^{(t)}$ randomly as follows: (i) sample a vector $v \sim [0, 1]^n$ uniformly at random, (ii) select a random $\sigma_i$ for $i \in [a]$, and sort $v$ for it to be consistent with $\sigma_i$, that is, $v_{\sigma_i^{-1}(n)} \ge v_{\sigma_i^{-1}(n - 1)} \ge \ldots \ge v_{\sigma_i^{-1}(1)}$, and (iii) let $c^{(t)} = v/\|v\|_1$. This $c^{(t)}$ mimics a random click-through-rate close to the random preferences (permutations) in $[a]$. 

We run this experiment for two settings: (i) $a = 1$ (single global optimal preference), and (ii) $a = 6, b = 6$ (mixture of multiple preferences). For this learning problem, we run Online Frank Wolfe (OFW) and Online  Mirror Descent (OMD) variants with the projection solved by using AFW and the toolkit proposed: (1) OMD-UAFW: OMD with projection using unoptimized (i.e., vanilla) Away-step Frank-Wolfe, (2) OMD-ASAFW: OMD with projection using AFW with reused active sets, (3) OMD-TSAFW: OMD with projection using AFW with \textsc{Infer}, \textsc{Restrict}, and \textsc{Rounding}, (4) OMD-\AAFW \:OMD with adaptive AFW, (5) OMD-PAV: OMD with projection using pool adjacent violators, and (6) OFW. We call the first four variants as OMD-AFW variants. In all the AFW variants, we stop and output the solution when the FW gap $g^{FW}$ is at most $\epsilon = 10^{-3}$. The OFW variant we implemented is that of Hazan and Minasyan \cite{hazan2020faster} developed in 2020, which is state-of-the-art and has a regret rate of $O(T^{2/3})$ for smooth and convex loss functions\footnote{\textcolor{blue}{The variant of Hazan and Minasyan \cite{hazan2020faster} is essentially a practical extension of the Follow-the-Perturbed-Leader (FPL) algorithm developed by Kalai and Vempala in 2005 \cite{Kalai2005} for the case of convex and smooth loss functions. However, for the case of linear loss (which is our case - see below), the OFW of \cite{hazan2020faster} recovers the original FPL algorithm, attaining an optimal regret rate of $O(\sqrt{T})$ \cite{hazan2020faster}.}}.

\begin{figure}[t]
    \centering
    \includegraphics[scale=0.5]{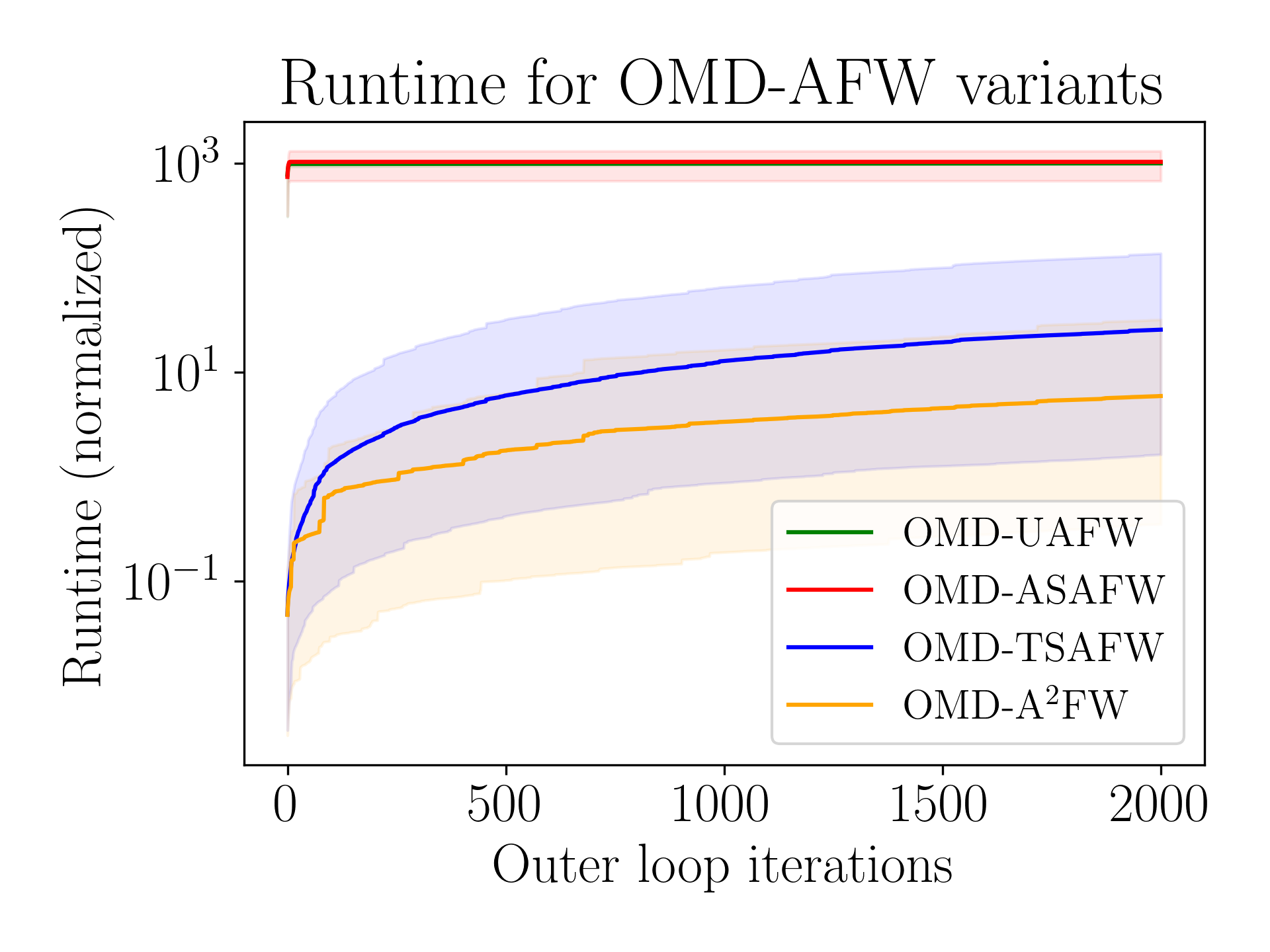}\label{fig: a2fw-1}
    \includegraphics[scale=0.5]{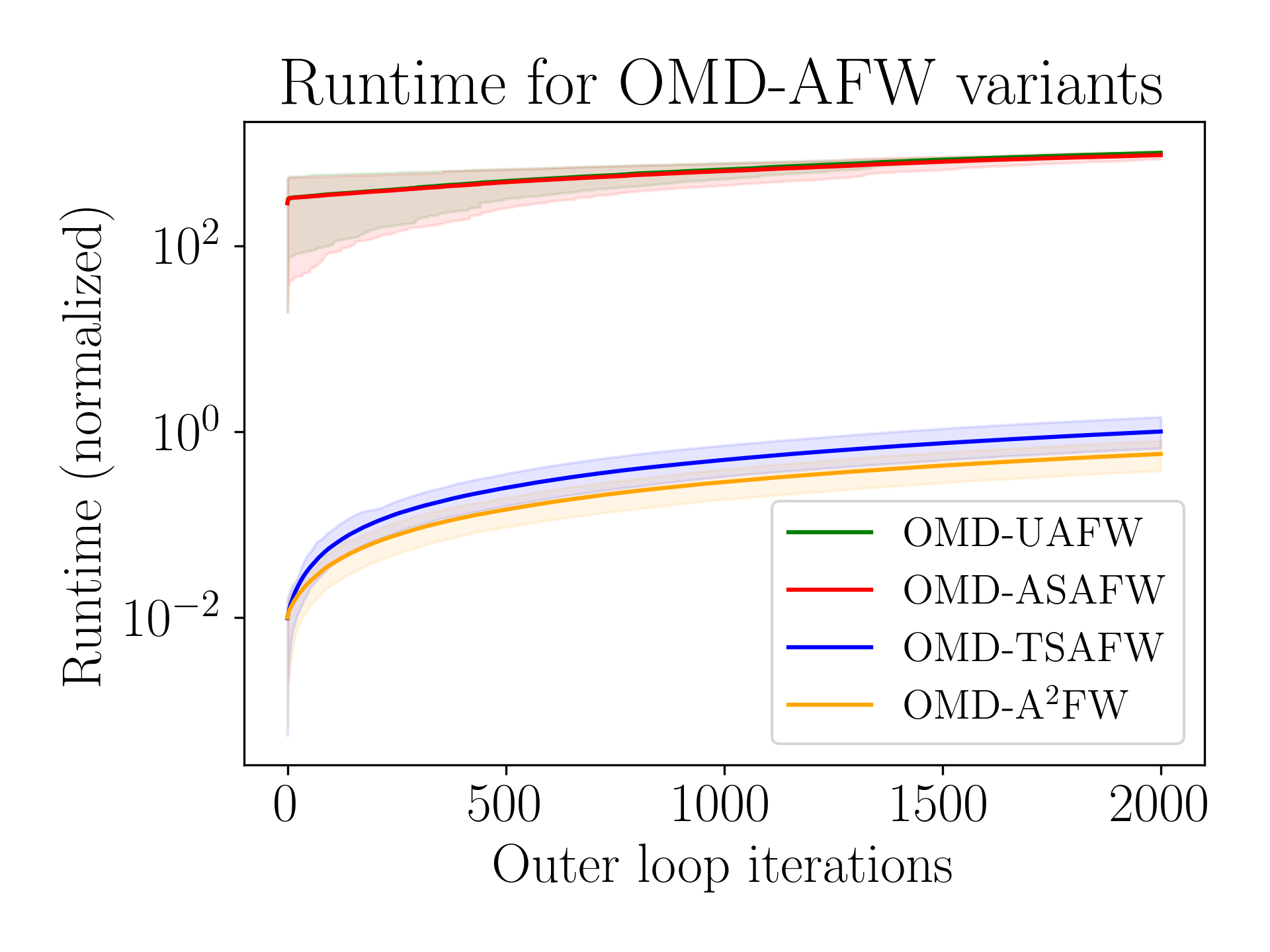}\label{fig: a2fw-2}
    \caption{Left: 25-75\% percentile plots of normalized run times for OMD-AFW variants for first loss setting averaged over 20 runs. Right: 25-75 percentile plots of normalized run times for OMD-AFW variants with second loss setting averaged over 20 runs with $n = 100$, in Section \ref{sec: card experiments}. \vspace{-5 pt}}
    \label{fig: experiment-main-plot}
\end{figure}

As stated previously, we run the experiment $20$ times each for (i) $a = 1$ and (ii) $a = 6, b = 6$. Since the run time varied across all runs, we normalized the run time for OMD-UAFW as $1000$ (other variants being normalized) in each run to take an equally-weighted average of run times.

\begin{figure}[t]
\vspace{0 pt}
    \centering 
    \includegraphics[scale=0.5]{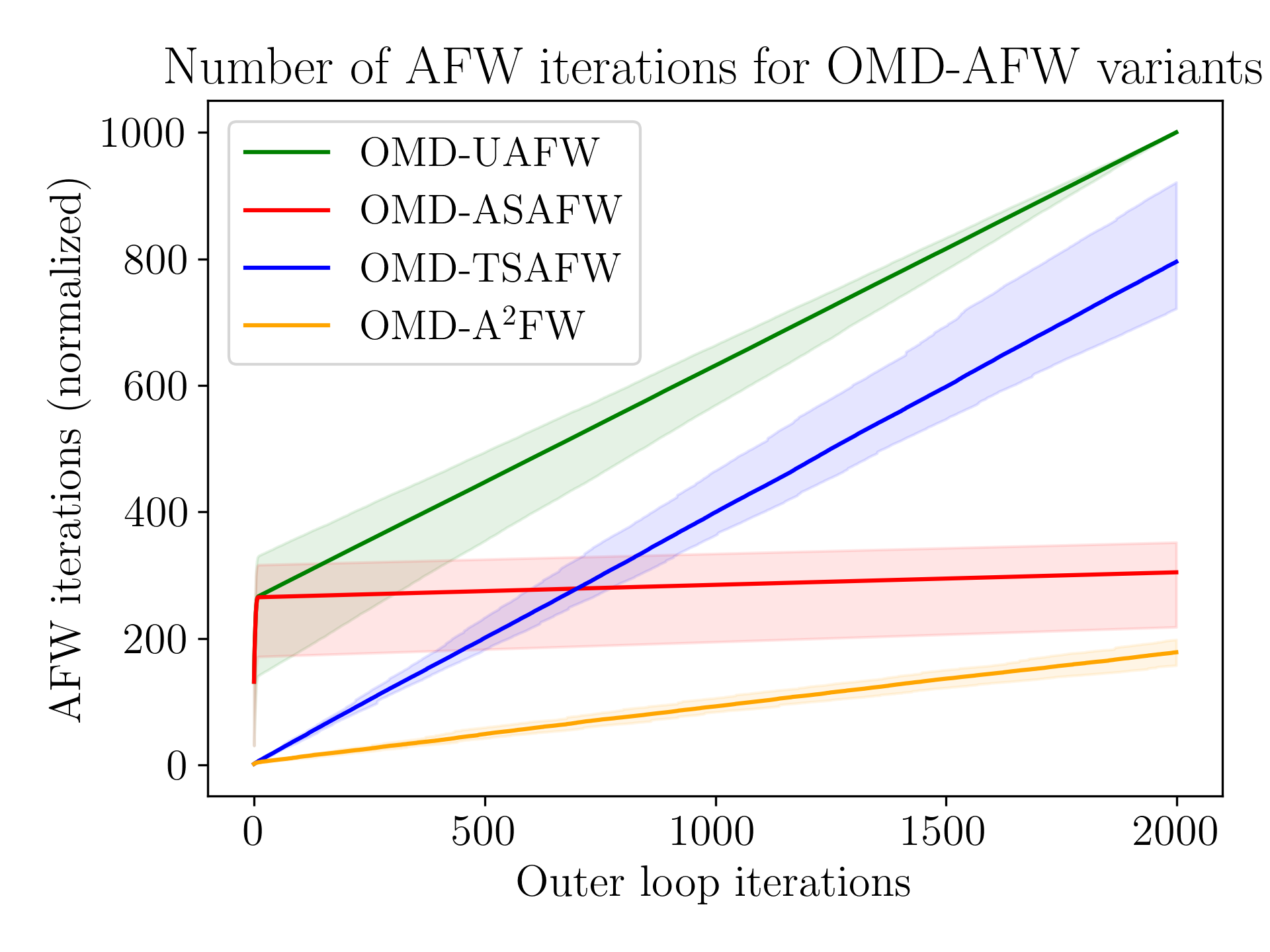}\label{fig: regret_nc_1}
    \includegraphics[scale=0.5]{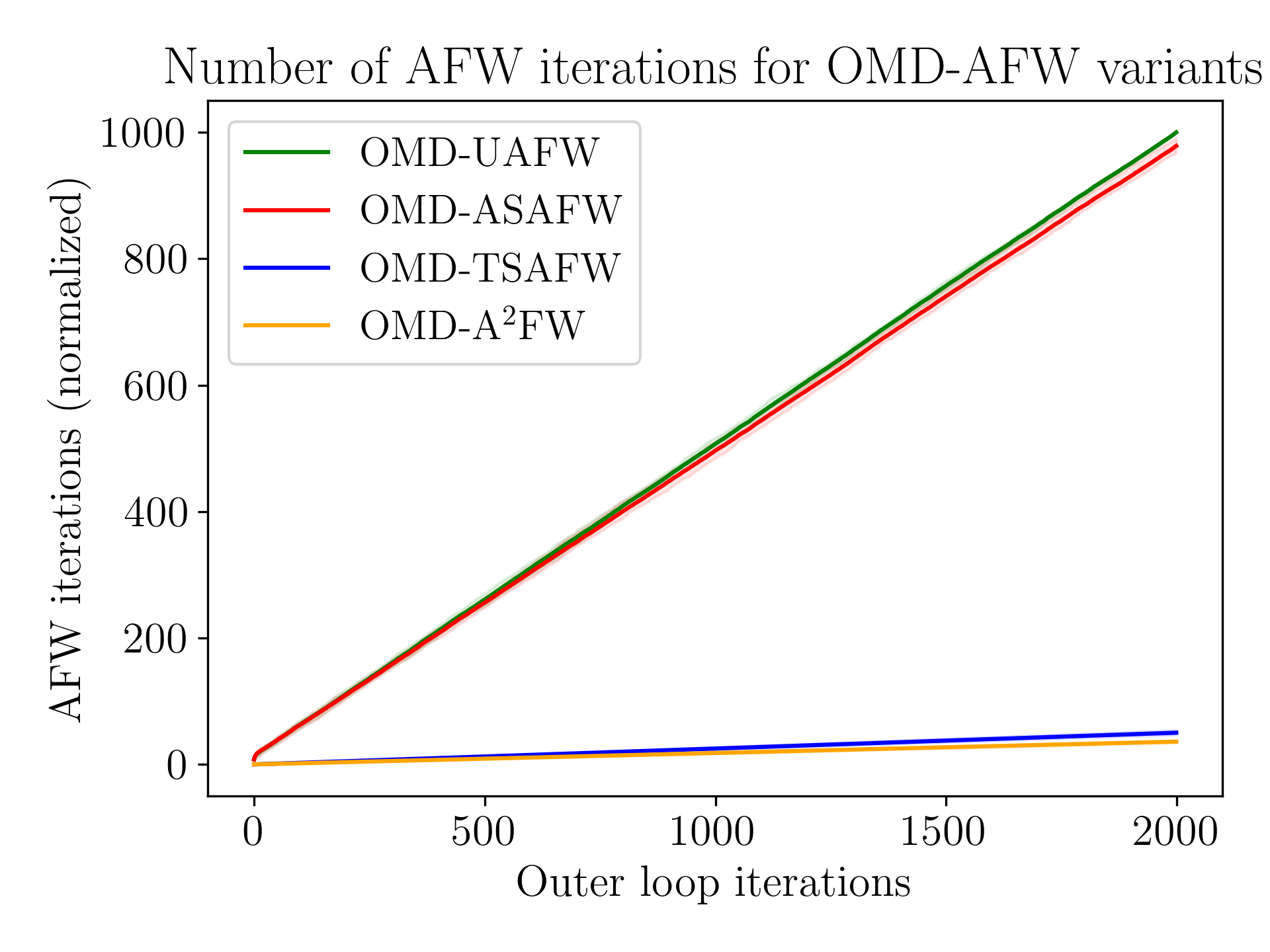}\label{fig: regret_nc_4}
    \caption{25-75\% percentile plots of number of AFW iterations (cumulative) for OMD-AFW variants over 20 runs for first loss setting (left) and second loss setting (right) for $n = 100$, in Section \ref{sec: card experiments}. \vspace{-10 pt}}
    \label{fig: experiment-regret-nc}
\end{figure}

Figures \ref{fig: experiment-main-plot}-middle and \ref{fig: experiment-main-plot}-right show improvements in run time for OMD-AFW variants, and show significant speed ups of the optimized OMD-AFW variants over OMD-UAFW. Each iteration of OMD involves projecting a point on the permutahedron, and the cumulative run times for these projections are plotted. We see more than three orders of magnitude improvement in run time for OMD-TSAFW and OMD-\AAFW\ compared to the unoptimized OMD-AFW. Both OMD-PAV and OFW run $4$ to $5$ orders of magnitudes faster on average than OMD-UAFW, {while only being $2$ to $3$ orders of magnitudes faster on average than OMD-\AAFW, which is around $2$ orders of magnitude reduction in runtime.} This is a significant effort in trying to bridge the gap between OMD and OFW. However, OMD-PAV suffers from the limitation that it only applies to cardinality-based submodular polytopes, while OFW has significantly higher regret in computations.

Figure \ref{fig: experiment-iterations} shows the total number of iterations of the inner AFW loop for the four OMD-AFW variants plotted cumulatively across the $T = 2000$ projections in the outer OMD loop. AFW for optimized variants that reuse active sets finishes in much fewer AFW iterations over the unoptimized variant, which contributes to a better running time and indicates that we are efficiently reusing information from AFW iterates. These results are summarized in Table \ref{tab: computations_permutahedron_summary}.


\begin{table}[t]
\vspace{-5 pt}
\centering
\begin{tabular}{|c||c|c|c|c|c|c|}
\hline
& \multicolumn{4}{|c|}{\textbf{OMD-AFW Variants}} & \multicolumn{2}{|c|}{}\\
\hline
 & \textbf{UAFW} & \textbf{ASAFW} & \textbf{TSAFW} & \textbf{\AAFW} & \textbf{OFW}       & \textbf{OMD-PAV}   \\
\hline
\hline
\multicolumn{7}{|c|}{\large {$\mb{a = 1}$}}      \\
\hline
Regret       & $1000$ & $1000$ & $1019$ & $1012$ & $1355914$ & $1000$ \\
\hline
Runtime      & $1000$ & $1034$ & $25.61$ & $5.936$ & $0.06451$ & $0.07695$ \\
\hline
AFW Iterates     & $1000$ & $384.4$ & $795.3$ & $178.0$ & - & - \\
\hline
\multicolumn{7}{|c|}{\large {$\mb{a = 6, b = 6}$}} \\
\hline
Regret       & $1000$ & $1000$ & $1000$ & $1000$ & $18340$ & $1000$ \\
\hline
Runtime      & $1000$ & $945.1$ & $0.9991$ & $0.5730$ & $0.003003$ & $0.003697$ \\
\hline
AFW Iterates     & $1000$ & $978.5$ & $50.24$ & $36.10$ & - & - \\
\hline
\end{tabular}
\vspace{0.2cm}
\caption{A comparison of total runtime, regret, and numbers of AFW iterates for computations {over the permutahedron} averaged over $20$ runs of the experiment for $n=100$. The corresponding values for OMD-UAFW are normalized to $1000$ and all numbers are reported to $4$ significant digits. \vspace{-10 pt}}
\label{tab: computations_permutahedron_summary}
\end{table}

\subsection{Second experiment: Online learning over general submodular polytopes}\label{sec: extra-computations}

We next benchmark online mirror descent over a set of general submodular polytopes as described below. 
We consider $n = 50$ elements in the ground set and build a submodular function $f: 2^n \to \mathbb{R}$. For a parameter $p \in [0, 1]$, create a random bipartite graph\footnote{Similar (deterministic) constructions for general submodular functions using bipartite graphs have been used in \cite{Jegelka2011} for image segmentation and speech recognition.} $G$ with bipartition $(U, V)$, where $U = V = [n]$ and each edge $uv, u \in U, v \in V$ is present independently with probability $p$. For each $T \subseteq U$, $f(T)$ is the number of neighbors of $T$ in $V$, that is, $f(T) = |\{v \in V: (u, v) \in E(G) \;\text{for some}\; u \in T \}|$. It can be shown that $f$ is submodular and is not cardinality-based in general. We fix $p = 0.2$ in our case.



The loss functions are generated in the same way as before with two parameters: (i) $a = 1$ and (ii) $a = 6, b = 6$. We do not consider OMD-PAV variant in this experiment because the PAV algorithm is restricted to cardinality-based submodular polytopes.

Figure \ref{fig: experiment-runtime-nc} shows significant speed ups of the optimized OMD-AFW variants over OMD-UAFW for $a = 1$ and for $a = 6,  b = 6$. We remark that OFW is much faster than the OMD-AFW variants; however, it has significantly higher regret (on average, 20 to 30 times as much as OMD-AFW variants for $a = 1$ and 6 to 7 times as much as OMD-AFW variants for $a = 6, b = 6$). Figure \ref{fig: experiment-regret-nc} shows mild improvements in regret for OMD-\AAFW \: over OMD-UAFW. This improvement in regret arises from our rounding procedure: AFW outputs only an approximate solution to the problem (depending on the FW gap stopping threshold $\epsilon$) but \AAFW \: can potentially round to the exact solution, resulting in lower regret. These results are summarized in Table \ref{tab: computations_noncardinality_summary}.

\begin{table}[t]
\vspace{-5 pt}
\centering
\begin{tabular}{|c||c|c|c|c|c|c|}
\hline
& \multicolumn{4}{|c|}{\textbf{OMD-AFW Variants}} &\\
\hline
 & \textbf{UAFW} & \textbf{ASAFW} & \textbf{TSAFW} & \textbf{\AAFW} & \textbf{OFW} \\
\hline
\hline
\multicolumn{6}{|c|}{\large {$\mb{a = 1}$}}     \\
\hline
Regret       & $1000$ & $1000$ & $855.9$ & $854.5$ & $27950$ \\
\hline
Runtime      & $1000$ & $952.6$ & $362.8$ & $31.65$ & $5.553$ \\
\hline
AFW Iterates     & $1000$ & $138.7$ & $937.6$ & $102.6$ & - \\
\hline
\multicolumn{6}{|c|}{\large {$\mb{a = 6, b = 6}$}} \\
\hline
Regret       & $1000$ & $1000$ & $949.1$ & $949.8$ & $12380$ \\
\hline
Runtime      & $1000$ & $853.0$ & $201.9$ & $108.9$ & $0.5998$ \\
\hline
AFW Iterates     & $1000$ & $766.6$ & $396.4$ & $206.2$ & - \\
\hline
\end{tabular}
\vspace{0.2cm}
\caption{A comparison of total runtime, regret, and numbers of AFW iterates for computations in Section \ref{sec: extra-computations} averaged over $20$ runs of the experiment when $n = 100$. The corresponding values for OMD-UAFW are normalized to $1000$ and all numbers are reported to $4$ significant digits. \vspace{-5 pt}}
\label{tab: computations_noncardinality_summary}
\end{table}

\begin{figure}[t]
    \centering 
    \includegraphics[scale=0.5]{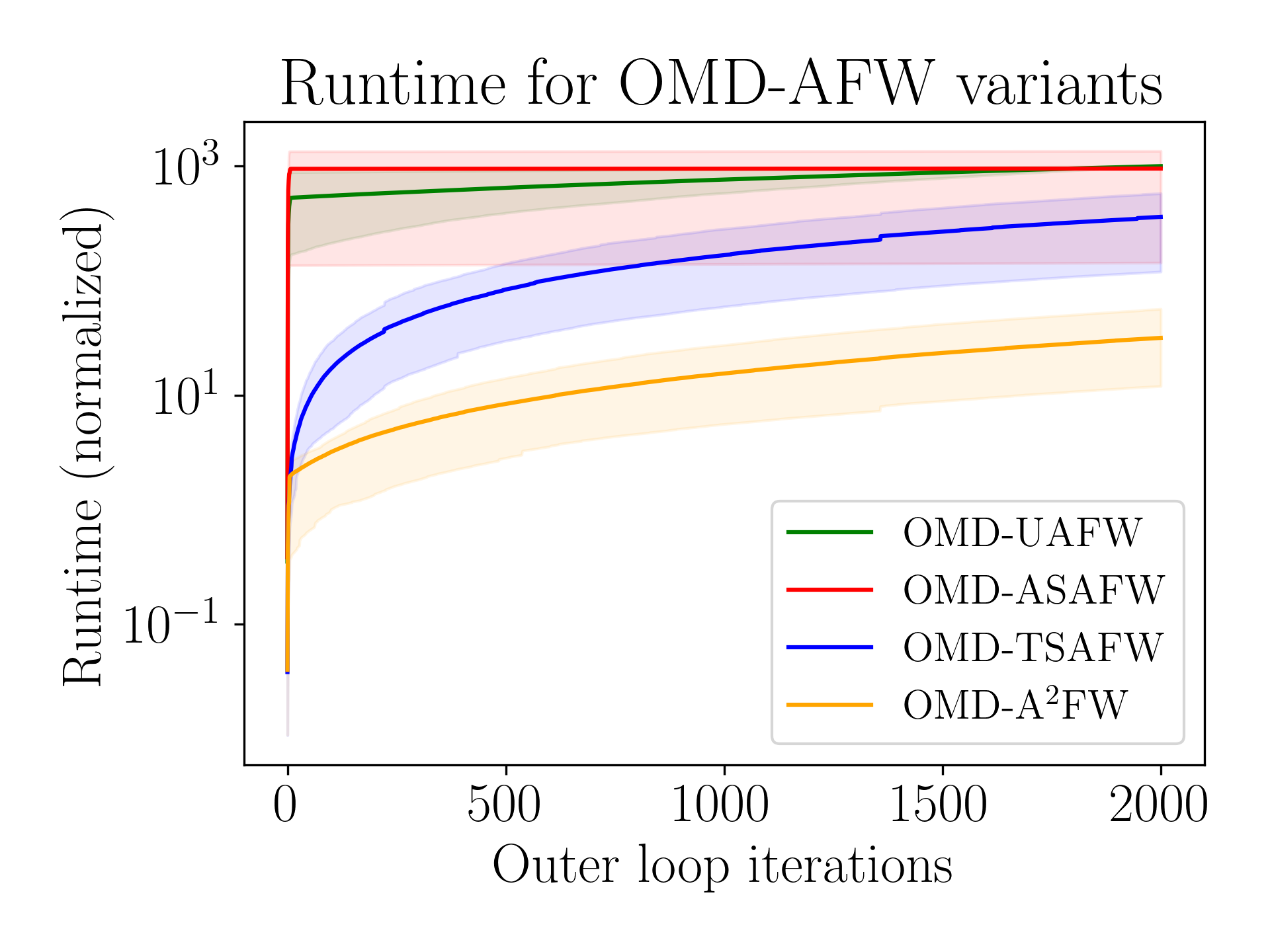}\label{fig: runtime_nc_1}
    \includegraphics[scale=0.5]{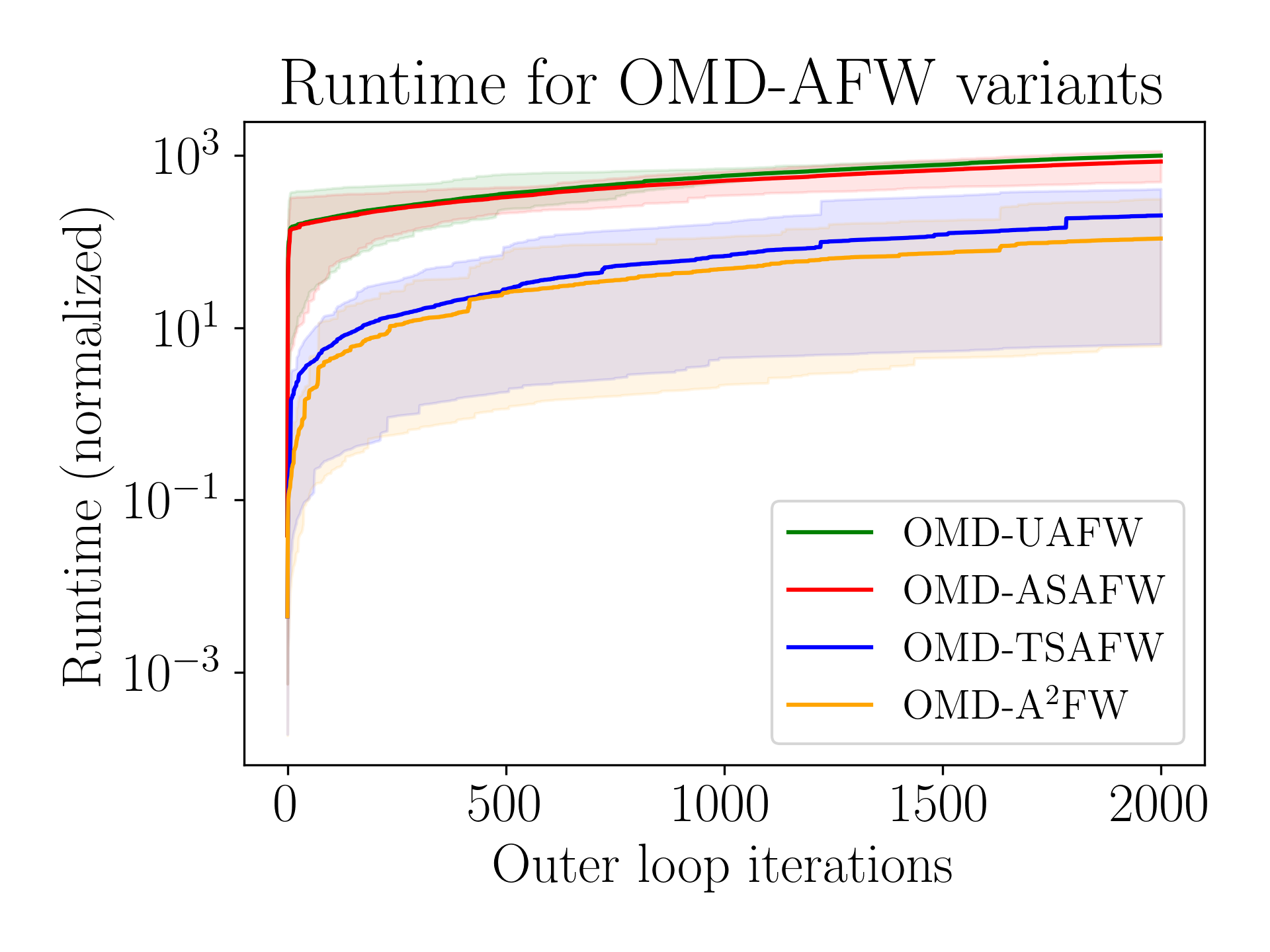}\label{fig: runtime_nc_4}
    \caption{25-75\% percentile plots of runtime for OMD-AFW variants over 20 runs for first loss setting (left) and second loss setting (right) for $n=100$, in section \ref{sec: extra-computations}.  \vspace{-0.5cm}}
    \label{fig: experiment-runtime-nc}
\end{figure}

The regret for all OMD variants (including OMD-PAV) was observed to be quite similar. OMD has a regret 1 to 2 orders of magnitude lower than OFW on average, thus bolstering the claim that we need to invest research to speed-up this optimal learning method and its variants. This drop in regret is {\it significant} in terms of revenue for an online retail platform. The regret for all OMD variants was observed to be nearly the same. Overall, speeding up OMD is an example of the impact of our toolkit, which can be applied in the broader setting of iterative optimization methods. 



\subsection{A Note on Precision Issues Arising in AFW} Several arithmetical computations in the Away-Step Frank-Wolfe (AFW) algorithm are prone to numerical precision error issues, which prevent the algorithm from converging. Any program that implements the AFW algorithm (or any of its variants) has limited precision numbers, and we observed that occasionally these errors add up to either slow down the algorithm significantly or make it enter an infinite loop where the algorithm is stuck at a vertex or a set of vertices. The first source of error is the line search step $\gamma_t = \argmin_{\gamma \in [0, \gamma_t^{\max}]} h(z^{(t)} + \gamma d)$ along the required descent direction $d$. For example, if the optimal step-size is 1, then \texttt{Scipy} python optimizer would return 0.9999940391390134, which is a significant error if a high optimality precision is needed. The leads to a second source of error when updating the convex combination of vertices for $z^{(t + 1)}$, because an incorrect computation of $\gamma_t$ (as in the previous example) may add stale vertices in the active set that prevents progress in subsequent iteration. These errors can add up across multiple iterations to generate a larger, more significant error.

\begin{figure}[t]
\vspace{-10 pt}
    \centering
    \includegraphics[scale=0.5]{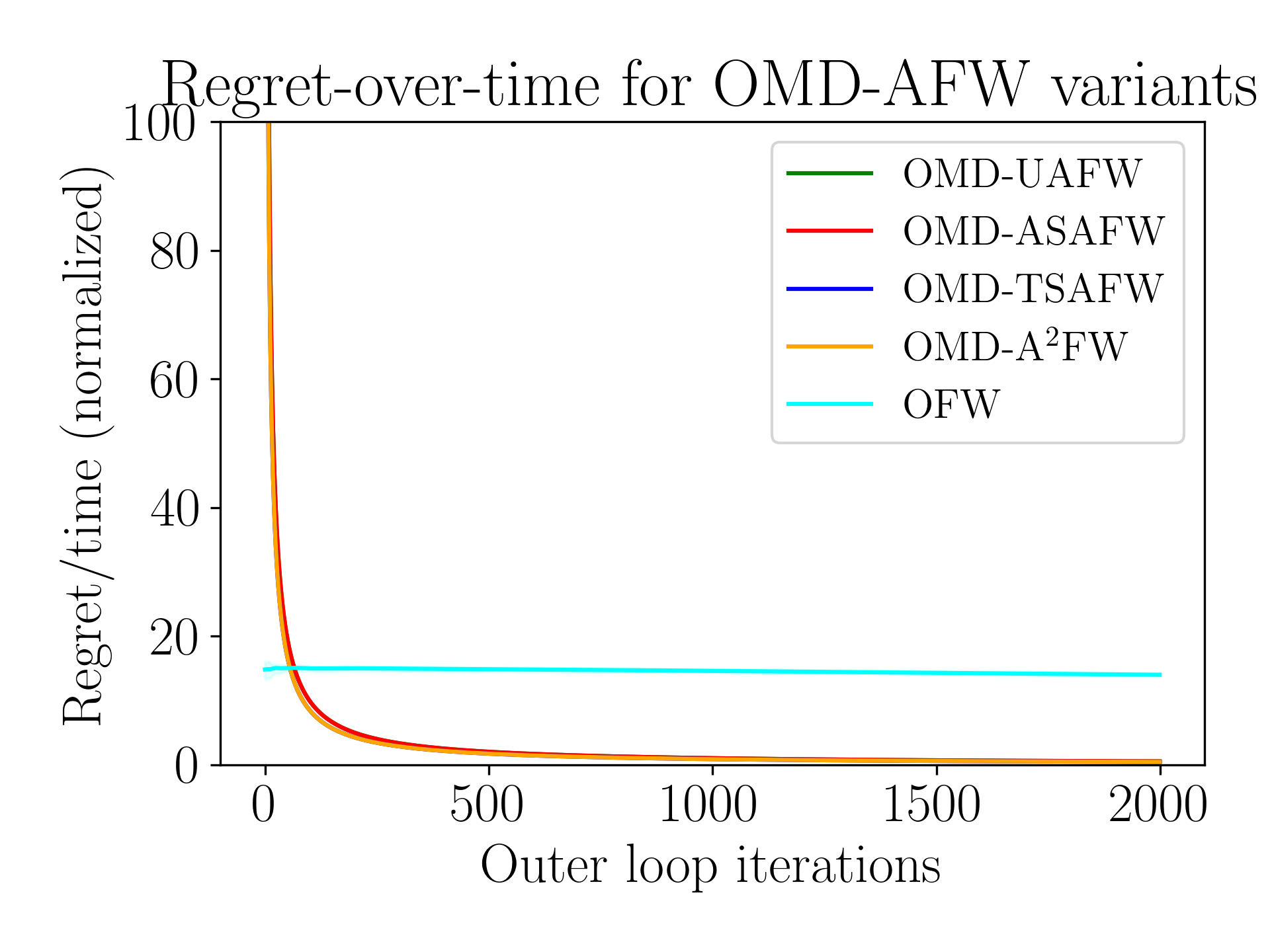}\label{fig: iterations_1}
    \includegraphics[scale=0.5]{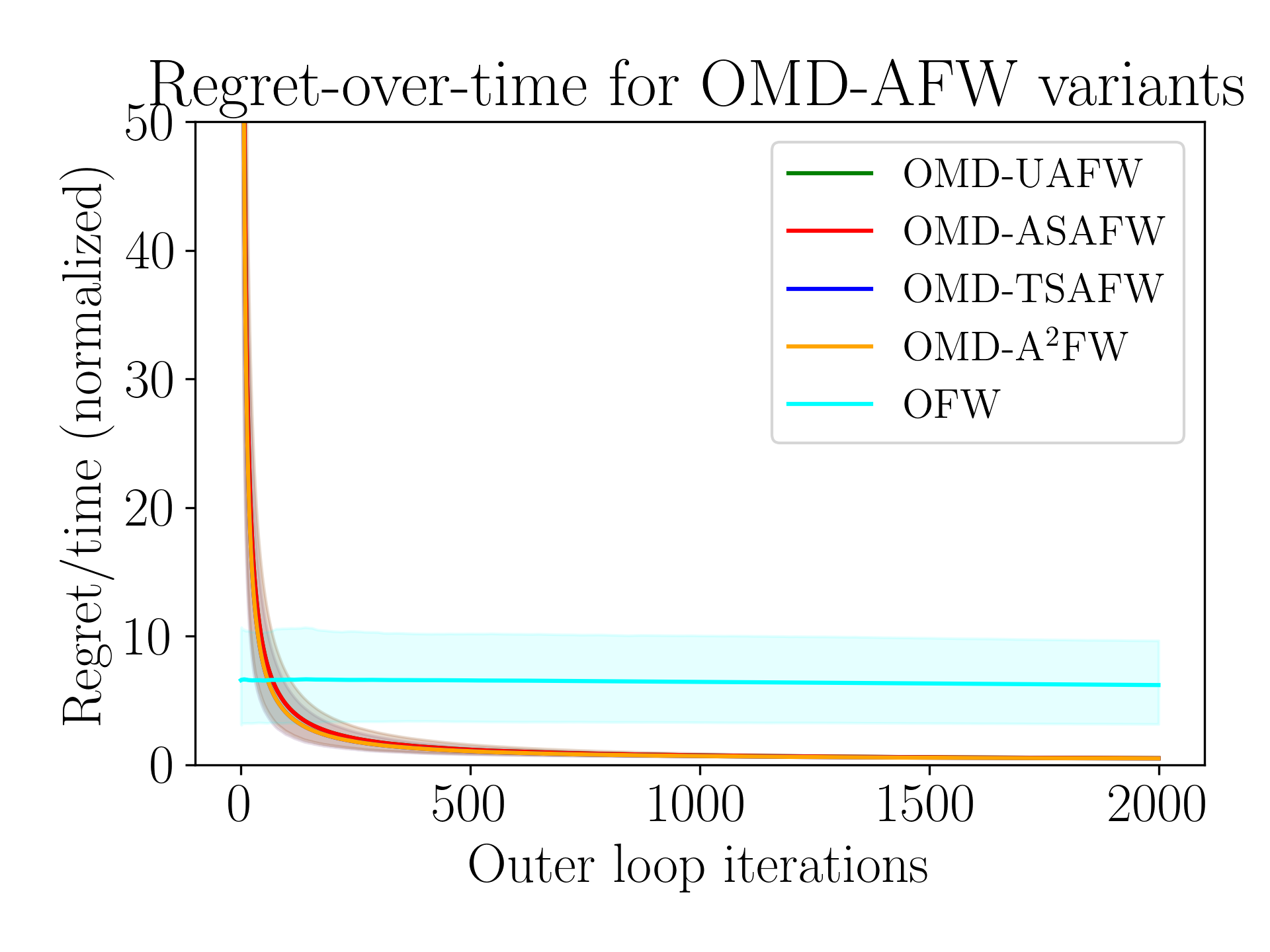}\label{fig: iterations_4}
    \caption{25-75\% percentile plots of regret over time for OMD-AFW variants over 20 runs for first loss setting (left) and second loss setting (right) for computations for $n = 100$, in Section \ref{sec: extra-computations}. \vspace{-10 pt}}
    \label{fig: experiment-iterations}
\end{figure}

To mitigate some of these issues, we added some mathematical checks in our implementation. Our implementation is in Python 3.5+, and uses the usual floating-point numbers in Python which have a precision of 12 decimal points. To avoid precision issues with Python optimization solvers when doing line-search $\gamma_t = \argmin_{\gamma \in [0, \gamma_t^{\max}]} h(z^{(t)} + \gamma d)$, we explicitly compute the optimal step size in closed form as follows. We check whether $\gamma_t = \gamma_t^{\max}$ separately using first-order optimality: If $\innerprod{d}{\nabla h(x + \gamma_t^{\max}d)} \leq 0$, then $\gamma_t = \gamma_t^{\max}$, otherwise $\gamma_t \in (0,\gamma_t^{\max})$ is strictly in the interior of the optimization interval. This crucially prevents an incomplete movement in the away direction. In particular, this deals with the issue of the algorithm being `stuck' at a stale vertex as mentioned above. We also explicitly drop any vertices in maximal away step (i.e. a drop step) that should have a coefficient of $0$, although its actual coefficient in python maybe of the order\jme{\footnote{In several runs of our experiments, we saw up to 40 percent of the vertices in active set with coefficients below $10^{-10}$. In most runs of our experiment, the fraction of such vertices was non-trivial before our fixes for precision issues.}} of $10^{-10}$. When $\gamma_t \in (0,\gamma_t^{\max})$ and $h(x) = \|x - y_t\|^2$ is the Euclidean projection, we can compute the optimal step size in closed form by differentiating $h$ and setting the derivative to zero to obtain $\gamma_t = - \frac{\innerprod{d_t}{x_t - y_t}}{\|d_t\|^2}$. We can also extend this approach for any $L-$smooth convex function.
}

%% file: proofs-general-projections.tex
We first prove the following result about minimizing strictly convex functions over polytopes, which states that if we know the optimal (minimal) face, then we can restrict the optimization to that optimal face:
\optimalface*
\begin{proof}
Let $J$ denote the index set of inactive constraints at $x^*$. We assume that $J \neq \emptyset$, since otherwise the result follows trivially. Now, suppose for a contradiction that $x^* \neq \tilde{x}$. Due to uniqueness of the minimizer of the strictly convex function over $\mathcal{P}$, we have that $\tilde{x}\notin \mathcal{P}$ (otherwise it contradicts optimality of $x^*$ over $\mathcal{P}$). We now construct a point $y \in \mathcal{P}$ that is a strict convex combination of $\tilde{x}$ and $x^*$ and satisfies $f(y) < f(x^*)$, which contradicts the optimality of $x^*$. Define
\begin{equation} \label{gamma def}
    \gamma := \min_{\substack{j \in J:\\ \innerprod{{A}_j}{\tilde{x} - x^*} > 0}} \frac{b_j - \innerprod{{A}_j}{x^*}}{\innerprod{{A}_j}{\tilde{x} - x^*}} > 0,
\end{equation}
with the convention that $\gamma = \infty$ if the feasible set of \eqref{gamma def} is empty, i.e. $\innerprod{{A}_j}{\tilde{x} - x^*} \leq 0$ for all $j \in J$. Select $\theta \in (0, \min\{\gamma,1\})$. Further, define $y :=  x^* + \theta (\tilde{x} - x^*) \neq x^*$ to be a strict convex combination of $x^*$ and $\tilde{x}$. We claim that that $(i)$ $y \in \mathcal{P}$ and $(ii)$ $h(y) < h(x^*)$, which would complete our proof:
\begin{itemize}
    \item [$(i)$]\textit{We show that $y \in \mathcal{P}$.}  Since all the tight constraints $I$ are satisfied at $y$ by construction, to show the feasibility of $y$ we just have to verify that any constraint $j \in J$ such that $\innerprod{{A}_j}{\tilde{x}} > b_j > \innerprod{{a_j}}{x^*}$ is feasible at $y$. Indeed, we have
\begin{align*}
   \innerprod{{A}_j}{y} &= \innerprod{{A}_j}{x^*} + \theta \innerprod{{A}_j}{\tilde{x} - x^*}
   \leq \innerprod{{A}_j}{x^*} + \gamma \innerprod{{A}_j}{\tilde{x} - x^*}\\
    &\leq \innerprod{{A}_j}{x^*} + b_j - \innerprod{{A}_j}{x^*}
   = b_j,
\end{align*}
where we used the fact that $\theta \leq \gamma$ in the first inequality, and the definition of $\gamma$ \eqref{gamma def} in the second. 
\item [$(ii)$] \textit{We show that $h(y) < h(x^*)$.} Observe that $h(\tilde{x}) \leq h(x^*)$ by construction. Since $x^* \neq \tilde{x}$, we have
$$h(y) = h((1 -\theta)x^* + \theta\tilde{x})
    < (1 -\theta)h(x^*) + \theta h(\tilde{x})
    \leq h(x^*),$$
where we used the fact $\theta \in (0,1)$ and the fact that $h$ is strictly convex in the first inequality, and the fact that $h(\tilde{x}) \leq h(x^*)$ in the second.
\end{itemize}
This completes the proof.  
\end{proof}

\subsection{Proof of Theorem \ref{thm: most_far_away_points_project_to_vertices}}

\jme{In order to prove the theorem, we first need a bound on the fraction of points in an arbitrary ball that project to a given face of the polytope:

\begin{lemma}\label{lem: bound_volume_of_pi_f}
    Let $F$ be a $d$-dimensional face of a polytope $\mathcal{P} \subseteq \R^n$, and let $\vol_d(F)$ denote the $d$-dimensional volume of $F$. For any $z \in \R^n$ and any $R > 0$,
    \[
        r_{F}\big(B_{z}(R)\big) \le \frac{\vol_d(F) \cdot v_{n-d}}{v_n} \cdot \frac{1}{R^d}.
    \]
\end{lemma}

\begin{proof}
    We will show that $\vol\big(\Theta_{\mathcal{P}}(F) \cap B_{z}(R)\big) \le \vol_d(F) \cdot v_{n-d} R^{n - d}$. Since $\vol\big(B_{z}(R)\big) = v_n R^n$, this implies the result.

    Let $\mathcal{P} = \{x: Ax \le b\}$. There exists some set of indices $I$ such that $F = \{x \in P: A_I x = b_I\}$ with the row rank of $A_I$ being $n - d$. Then $\cone(F) = \{A_I^T \lambda: \lambda \ge 0\}$. Since $\dim(F) = d$ and since volumes are preserved under translation and rotation, we can assume without loss of generality that $F \subseteq \spn\{e_1, \ldots, e_{d}\}$, and by the above, we have that $\cone(F) \subseteq \spn\{e_{d + 1}, \ldots, e_n\}$.
    

    Since $\Theta_{\mathcal{{P}}}(F) = \relint(F) + \cone(F)$ from Lemma \ref{face cond}, we get
    \begin{equation}\label{eqn: bound_on_volume_of_pi_F}
        B_{z}(R) \cap \Theta_{\mathcal{P}}(F) \subseteq \relint(F) + \big(B(z, R) \cap \cone(F)\big).
    \end{equation}
    
    Since $\cone(F) \subseteq \spn\{e_{d + 1}, \ldots, e_n\}$, we have that $\big(B(z, R) \cap \cone(F)\big) \subseteq B(z, R) \cap \spn\{e_{d + 1}, \ldots, e_n\}$, which is an $(n - d)$-dimensional ball of radius at most $R$. This helps us bound the desired volume:
    \begin{align*}
        \vol\Big(B_{z}(R) \cap \Theta_{\mathcal{P}}(F)\Big) &\le \vol\big(\relint(F) + \big(\spn\{e_{d + 1}, \ldots, e_n\} \cap B(z, R)\big)\big) \\
        &= \vol_d\big(\relint(F)\big) \cdot \vol_{n - d}\big(\spn\{e_{d + 1}, \ldots, e_n\} \cap B(z, R)\big) \\
        &= \vol_d(F) \cdot \vol_{n - d}\left(\spn\{e_{d + 1}, \ldots, e_{n} \cap B(z, R)\} \right) \\
        &= \vol_d(F) \cdot v_{n - d} R^{n - d}.
    \end{align*}
    
    
    The first inequality follows from equation \eqref{eqn: bound_on_volume_of_pi_F} and the following equality follows since $\relint(F) \subseteq \spn\{e_1, \ldots, e_d\}$ and $B(z, R) \cap \spn\{e_{d + 1}, \ldots, e_n\} \subseteq \spn\{e_{d + 1}, \ldots, e_n\}$.
    
    
  \end{proof}

We restate the theorem here for convenience.

\farawaypoints*


\begin{proof}
    Recall that $\sum_{F \in \mathcal{F}(\mathcal{P})} r_F\big(B_{z}(R)\big) = 1$, so that it is enough to prove that $\sum_{F \in \mathcal{F}_{\setminus \mathcal{V}}(\mathcal{P})} r_F\big(B_{z}(R)\big) < \delta$ for all $R \ge \mu_{\mathcal{P}}\frac{n^2}{\delta}$, where $\mu_{\mathcal{P}}$ is some constant dependent on the polytope that we specify later.
    
    We use the well-known formula $v_n = \frac{\pi^{n/2}}{\Gamma\big(\frac{n}{2} + 1\big)}$ \cite{apostol1969calculus} and the bound\footnote{This can be shown using the fact that $\Gamma(m) = (m - 1)!$ and $\Gamma(m + 1/2) = (m - 1/2)(m - 3/2) \ldots (1/2) \cdot \sqrt{\pi}$ for positive integer $m$.} $\frac{\Gamma\big(\frac{n}{2} + 1\big)}{\Gamma\big(\frac{n - d}{2} + 1\big)} \le \frac{n^d}{2^d}$, and we also write the sum by the dimension of the face:
    \begin{align}
        \sum_{F \in \mathcal{F}_{\setminus \mathcal{V}}(\mathcal{P})} r_F\big(B_{z}(R)\big) &\le \sum_{d \in [n]} \sum_{\substack{F \in \mathcal{F}(\mathcal{P}): \\ \dim(F) = d}} \frac{\vol_d(F) \cdot v_{n - d}}{v_n} \cdot \frac{1}{R^d} & (\text{Lemma } \ref{lem: bound_volume_of_pi_f}) \notag \\
        &= \sum_{d \in [n]} \frac{1}{R^d \pi^\frac{d}{2}} \cdot \frac{\Gamma\big(\frac{n}{2} + 1\big)}{\Gamma\big(\frac{n - d}{2} + 1\big)} \Big[\sum_{\substack{F \in F(P): \\ \dim(F) = d}} \vol_d(F) \Big] \notag \\
        &\le \sum_{d \in [n]} \frac{1}{R^d \pi^\frac{d}{2}} \cdot \frac{n^d}{2^d} \Big[\sum_{\substack{F \in \mathcal{F}(\mathcal{P}): \\ \dim(F) = d}} \vol_d(F) \Big] \label{eqn: bound-on-points-not-projecting-on-vertices}.
    \end{align}
    
    Let $V_d$ be the sum of $d$-dimensional volumes of all $d$-dimensional faces of $\mathcal{P}$, i.e, $V = \sum_{\substack{F \in \mathcal{F}(\mathcal{P}): \\ \dim(F) = d}} \vol_d(F)$. Then for each $d \in [n]$, for $R \ge \frac{(V_d)^{1/d} n^2}{2\sqrt{\pi}\delta}$, we can bound $\frac{n^d}{R^d \pi^\frac{d}{2} 2^d} V_d \le \frac{\delta^d}{n^d} \le \frac{\delta}{n}$,
    implying that the sum above in (\ref{eqn: bound-on-points-not-projecting-on-vertices}) is bounded by $\delta$ for $R \ge \frac{n^2}{2\sqrt{\pi}\delta} \max_{d \in [n]} (V_d)^{1/d}$ (so that we choose $\mu_{\mathcal{P}} = \frac{\max_d (V_d)^{1/d}}{2 \sqrt{\pi}}$).
  \end{proof}

\subsection{Proof of Theorem \ref{thm: perturbations_dont_change_minimal_faces_at_large_distances}}

We need two lemmas before we prove the theorem. In what follows, let $\ov{\Theta_{\mathcal{P}}(F)}$ denote the closure of $\Theta_{\mathcal{P}}(F)$.

\begin{lemma}\label{lem: thetas_intersect_in_lower_dimensions}
    Let $F_1, F_2$ be two distinct faces of a polytope $\mathcal{P} \subseteq \R^n$. Then,
    \begin{enumerate}
        \item If $F_1, F_2$ are not adjacent in $\mathcal{P}$, then $\ov{\Theta_{\mathcal{P}}(F_1)} \cap \ov{\Theta_{\mathcal{P}}(F_1)} = \emptyset$.
        \item If $F_1, F_2$ are adjacent in $\mathcal{P}$, then $\dim\Big(\ov{\Theta_{\mathcal{P}}(F_1)} \cap \ov{\Theta_{\mathcal{P}}(F_2)}\Big) \le n - 1$.
    \end{enumerate}
\end{lemma}

\begin{proof}
    \begin{enumerate}
        \item As noted previously, $\Theta_{\mathcal{P}}(F_1)$ and $\Theta_{\mathcal{P}}(F_2)$ are disjoint sets. Suppose $y \in \ov{\Theta_{\mathcal{P}}(F_1)} \cap \ov{\Theta_{\mathcal{P}}(F_1)}$, then there exist arbitrarily close points $y_1 \in \Theta_{\mathcal{P}}(F_1)$ and $y_2 \in \Theta_{\mathcal{P}}(F_2)$. Since $\Pi_P$ is a continuous operator, this implies that $\Pi_P(y_1)$ and $\Pi_P(y_2)$ are arbitrarily close, contradicting the fact that $F_1$ and $F_2$ are nonadjacent faces of $\mathcal{P}$, so that $\min_{x_1 \in F_1, x_2 \in F_2} \|x_1 - x_2\|_2 > 0$.
        
        \item From Lemma \ref{face cond}, $\Theta_{\mathcal{P}}(F_1), \Theta_{\mathcal{P}}(F_2)$ are both convex sets. Since they are disjoint, there exists a hyperplane $c$ separating them: 
        $c^\top x \ge \delta$ for all $x \in \Theta_{\mathcal{P}}(F_1)$ and $c^\top x < \delta$ for all $\Theta_{\mathcal{P}}(F_2)$. It is easy to see that this implies $c^\top x \ge \delta$ for all $x \in \ov{\Theta_{\mathcal{P}}(F_1)}$ and $c^\top x \le \delta$ for all $x \in \ov{\Theta_{\mathcal{P}}(F_2)}$. This gives $\ov{\Theta_{\mathcal{P}}(F_1)} \cap \ov{\Theta_{\mathcal{P}}(F_2)} \subset \{x: c^\top x = \delta\}$, implying the result.  
    \end{enumerate}
 \end{proof}

For any $S \subseteq \R^n$, denote $S_\epsilon$ to be the set of all points within distance $\epsilon$ of $S$, that is, $S_\epsilon = \{x \in R^n: \|x - \Pi_x(S) \|_2 \le \epsilon\}$. Our next lemma bounds the volume of $A_\epsilon$ in a ball for an affine space $A$.

\begin{lemma}\label{lem: volume_bound_affine_hyperplane_intersection_with_ball}
    Let $z \in \R^n$, and let $A$ be an affine space of dimension $m$. Then, $\vol\big(B_{z}(R) \cap A_\epsilon \big) \le v_m v_{n - m} R^m \epsilon^{n - m}$.
\end{lemma}

\begin{proof}
    Without loss of generality, we can assume that $A = \spn\{e_1, \ldots, e_m\}$, and that $\Pi_A(z) = 0$ by a translation and rotation.\footnote{To see this, let $x_0 = \Pi_A(z)$, then under the translation $\Phi: \R^n \to \R^n$ defined by $\Phi(x) = x - x_0$, $\Phi(A)$ is a linear subspace and $\Phi(x_0) = 0$.} That implies that the first $m$ coordinates of $z$ are all $0$. Also, 
    \[
        A_\epsilon = \{x: \|x - \Pi_A(x)\|_2 \le \epsilon\} = \Big\{x: \sum_{i \in [m + 1, n]} x_i^2 \le \epsilon^2\Big\},
    \]
    and consequently,
    \[
        A_\epsilon \cap B_{z}(R) = \Big\{x: \sum_{i \in [m + 1, n]} x_i^2 \le \epsilon^2, \|x - z\|_2 \le R\Big\} \subseteq \Big\{x: \sum_{i \in [m + 1, n]} x_i^2 \le \epsilon^2, \sum_{i \in [m]} x_i^2 \le R^2 \Big\}.
    \]
    The volume of this latter set is $(v_{n - m} \epsilon^{n - m}) \cdot (v_m R^m)$.
  \end{proof}

\begin{corollary}\label{cor: intersection_of_thetas_have_bounded_volume}
    For distinct adjacent faces $F_1, F_2$ of polytope $\mathcal{P} \subseteq \R^n$ and any $z \in \R^n$ and $R \ge 0$,
    \[
    \vol\Big( B_z(R) \cap \big(\ov{\Theta_{\mathcal{P}}(F_1)} \cap \ov{\Theta_{\mathcal{P}}(F_2)}\big)_\epsilon\Big) \le v_{n - 1} v_1 R^{n - 1} \epsilon.
    \]
\end{corollary}

\begin{proof}
    From Lemma \ref{lem: thetas_intersect_in_lower_dimensions}, $\ov{\Theta_{\mathcal{P}}(F_1)} \cap \ov{\Theta_{\mathcal{P}}(F_2)} \subseteq A$ for some hyperplane $A$, so that $\big(\ov{\Theta_{\mathcal{P}}(F_1)} \cap \ov{\Theta_{\mathcal{P}}(F_2)}\big)_\epsilon \subseteq A_\epsilon$. The result then follows directly from the Lemma.
  \end{proof}

We are now ready to prove the theorem, which we restate here for convenience.

\farawaypointsperturbation*


\begin{proof}
    If $y \in \Theta_{\mathcal{P}}(F)$ but $\tilde{y} \in \Theta_P(\tilde{F})$ for some faces $\tilde{F} \neq F$ of $\mathcal{P}$, we must have that $y$ lies within a distance $\|y - \tilde{y}\|_2$ of $\ov{\Theta_{\mathcal{P}}(F)} \cap \ov{\Theta_{\mathcal{P}}(\tilde{F})}$. Since $y - \tilde{y} = \epsilon'$ with $\|\epsilon'\|_2 \le \epsilon$, this means that $y$ must lie in  $\Big(\ov{\Theta_{\mathcal{P}}(F)} \cap \ov{\Theta_{\mathcal{P}}(\tilde{F})}\Big)_\epsilon$. Since $y$ is chosen uniformly from $B_{z}(R)$, this implies
        \begin{align*}
         1 - \mathbb{P}_{\mathcal{P}}\left(B_z(R), \epsilon\right) &\leq \frac{\sum_{F, \tilde{F} \in |\mathcal{F}(\mathcal{P})|} \vol\Big(B_z(R) \cap \big(\ov{\Theta_{\mathcal{P}}(F_1)} \cap \ov{\Theta_{\mathcal{P}}(F_2)}\big)_\epsilon\Big)}{\vol\big(B_{z}(R)\big)} \\
         & \leq \frac{\sum_{F, \tilde{F} \in |\mathcal{F}(\mathcal{P})|} v_{n - 1}v_1 R^{n - 1} \epsilon}{v_n R^n} \le \frac{ |\mathcal{F}(\mathcal{P})|^2 v_{n - 1}v_1 \epsilon }{v_n R} && (\text{using Corollary \ref{cor: intersection_of_thetas_have_bounded_volume}})
    \end{align*}
    Further, since $\frac{v_{n - 1} v_1}{v_n} \le \sqrt{2\pi n}$, we have $\mathbb{P}_{\mathcal{P}}\left(B_z(R), \epsilon\right) \ge 1 - \frac{1}{n}$ for $R \ge \sqrt{2\pi} \epsilon |\mathcal{F}(\mathcal{P})|^2 n^{3/2}$.
  \end{proof}
}

%% file: card-based.tex
We extend the proof of Lim and Wright \cite{lim2016efficient} and prove Theorem \ref{fenchel duality form}. To do that we need some more preliminaries. Consider any strictly convex and continuously differentiable separable function $h : \mathcal{D} \to \mathbb {R}$, defined over a convex set $\mathcal{D}$ such that $B(f) \cap \mathcal{D} \neq \emptyset$ and  $\nabla h(\mathcal{D}) = \mathbb{R}^E$ (this condition is not restrictive). Recall that the Fenchel-conjugate of $h$, that is $h^*({y}) = \sup_{x \in \mathcal{D} } \{\innerprod{{y}}{{x} } - h({x} )\}$ for any $y \in \mathcal{D}^*$. The subdifferential of $h$, i.e. the set of all subgradients of $h$, is defined by $\partial h = \{g \in \mathcal{D}^*: h(y) \geq h(x) + \innerprod{g}{y - x}\, \forall \, y \in \mathcal{D}\}$. Since $h$ is strictly convex and differentiable, the subdifferential is unique and given by $\partial h(x) = \nabla h(x)$ for all $x \in \mathcal{D}$. The conjugate subgradient theorem states that for any $x \in \mathcal{D}$, $y \in \mathcal{D}^*$, we have $\partial h(x) = \argmax_{\tilde{y} \in \mathcal{D}^*}\{\innerprod{x}{\tilde{y}} - h^*(\tilde{y})\} = \nabla h(x)$
and $\partial h^*(y) = \argmax_{\tilde{x} \in \mathcal{D}}\{\innerprod{y}{\tilde{x}} - h(\tilde{x})\} = \nabla h^*(y)$\footnote{$h^*$ is differentiable since $h$ is strictly convex (see Theorem 26.3 in \cite{rockafellar_1970}).} (see e.g. Corollary 4.21 in \cite{beck2017first}). We will need the Fenchel duality theorem, which states that (see e.g. Theorem 4.15 in \cite{beck2017first}):
\begin{equation} \label{fenchel duality thm1}
    \min_{x \in \mathcal{X}} h(x) = \max_{y \in \mathcal{D}^*} -h^*(y) + \min_{x \in \mathcal{X}} y^Tx.
\end{equation}
When $\mathcal{X} = B(f)$, the above result coincides with Proposition 8.1 in \cite{Bach2011}.

\subsection{Proof of Theorem \ref{fenchel duality form}}

To prove this Theorem, we need the following, which lemma shows that the ordering of the optimal solution is the same as the ordering of elements in $y$.
\begin{restatable}[Lemma 1 in \cite{Suehiro2012}] {lemma}{ordering}
\label{ordering}
Let $f : 2^E \to \mathbb{R}$ be any cardinality-based submodular function, that is $f(S) = g(|S|)$ function for some nondecreasing concave function $g$. Let $\phi: \mathcal{D} \to \mathbb{R}$ be a strictly convex and uniformly separable mirror map where $B(f) \cap \mathcal{D} \neq \emptyset$. Let $x^*:= \argmin_{x \in B(f)} D_{\phi}(x,y)$ be the Bregman projection of $y$. Assume that $y_1 \geq \dots \geq y_n$. Then, it holds that $x_1^* \geq \dots \geq x_n^*$.
\end{restatable}
We are now ready to prove Theorem \ref{fenchel duality form}:
\fencheldualityform*

\begin{proof}
Consider the problem of computing a Bregman projection of a point $y$ over a cardinality-based submodular polytope
\begin{equation}\label{breg proj}
\begin{aligned}
      \min ~D_{\phi}(x,y) \quad 
      \text{subject to} &~x(S) \leq g(|S|) & \forall S \subset E, \quad x(E) = g(|E|).
\end{aligned}
\end{equation}
Note that since, $y_1 \geq \dots \geq y_n$, using the previous lemma and Lemma \ref{optimal face} (about reducing the optimization problem to the optimal face), we can reduce the problem to only include the constraints that can be active under that ordering. That is, problem \eqref{breg proj} can be simplified to only have $n$ constraints as opposed to the original problem which had $2^n$ constraints:
\begin{equation}\label{breg proj1}
\begin{aligned}
      \min &~D_{\phi}(x,y) \quad 
      \text{subject to} ~ \sum_{i = 1}^jx_i \leq g(j) & \forall j \in [n-1], \quad \sum_{i = 1}^n x_i = g(n).
\end{aligned}
\end{equation}
Let $C$ denote the feasible region of the simplified optimization problem in \eqref{breg proj1}. Then, using the Fenchel duality theorem \eqref{fenchel duality thm1}, we have that the following problems are primal-dual pairs:
\begin{equation} \label{primal dual proof}
(P) \quad 
\begin{aligned}
    &\min ~ D_\phi(x,y)\\
    &\text{subject to} ~ x \in B(f)
\end{aligned}
 \qquad  \qquad  (D) \quad 
\begin{aligned}
    &\max_{z \in \mathbb{R}^n} ~- D_\phi^*(z,y) + \min_{x \in C} \innerprod{z}{x}
\end{aligned}
\end{equation}

Let us now focus on the $\min_{x \in C} \innerprod{z}{x}$ term in the dual problem $(D)$ above. If we let $Z_i =z_{i} - z_{i+1}$ for $i \in [n - 1]$ and $Z_n = z_n$, we have $z_i = \sum_{k=i}^n Z_k$. Recall that $c_i = g(i) - g(i-1)$ for $i = 1\dots,n$ and note that $c_i \leq c_{i-1}$ since $g$ is concave. This gives us
\begin{align}
    \innerprod{z}{x} &= \innerprod{z}{c} + \sum_{i=1}^n z_i(x_i - c_i) = \innerprod{z}{c} + \sum_{i=1}^n \left(\sum_{k=i}^n Z_k\right)(x_i - c_i) \nonumber \\
    &= \innerprod{z}{c} + \sum_{k=1}^n \left(\sum_{i=1}^k (x_i - c_i)\right) Z_k \label{derivation}
\end{align}
If $Z_k$ is larger than 0 for any $k \in [n - 1]$, then we claim that $\min_{x \in C} \innerprod{z}{x} = -\infty$. Indeed, we can set $x_i = c_i$ for all $i \notin \{k, k + 1\}$,
$x_k \to -\infty$ and $x_{k+1} = c_k + c_{k+1} - x_k$, where clearly such a solution is feasible in $C$. This means we require $Z_k \leq 0$ for all $k$ (i.e. $z_{i+1} \geq z_i$ for all $i$). Thus, since $\sum_{k=1}^n \left(\sum_{i=1}^k (x_i - c_i)\right) Z_k \geq 0 $ for all $k \in [n]$, it follows that $\min_{x \in C} \innerprod{z}{x} = \innerprod{z}{c}$ is obtained by setting $x_i = c_i$ for all $i$ in \eqref{derivation}. In other words, $\min_{x \in C} \innerprod{z}{x}$ is attained by vertex of $B(f)$ that corresponds to the ordering induced by the chain constraints. This proves our duality claim.

Furthermore, since $z^*$ is the optimal solution $z^*$ to the Fenchel dual $(D)$, we can use the conjugate subgradient theorem (given in the introduction of this section) to recover a primal solution using $\nabla_x  D_\phi(x^*,y) = \nabla \phi (x^*) - \nabla \phi(y) = z^*.$
\hfill $\square$ \end{proof}

\subsection{PAV Algorithm Implementation}
We now propose our algorithm, which solves the dual problem and then maps the dual optimal solution to a primal one using Theorem \ref{fenchel duality form}.  Best. al \cite{best2000minimizing} show that such problems could be solved exactly in $n$ iterations, using a well known algorithm called the Pool Adjacent Violators (PAV) Algorithm in $O(n)$ time (see Theorem 2.5 in \cite{best2000minimizing}). We adapt the algorithm here {in Algorithm \ref{alg:PAV}} to solve $(D)$.

The algorithm begins with the finest partition of the ground set $E$ whose blocks are single integers in $[E]$ and an initial solution (that is possibly infeasible and violates the chain constraints). Then, the algorithm successively merges blocks to reduce infeasibility through \textit{pooling} steps, obtaining a new, coarser partition of the ground set $E$ and an infeasible solution $z$, until $z$ becomes dual feasible. The pooling step is composed of solving an unconstrained version of the dual objective function restricted to a set $S$. We denote this operation by $\mathrm{Pool}_{\phi,y,c}(S) := \argmin_{\gamma \in \mathbb{R}} \sum_{i \in S}  D_{\phi_i}^*(\gamma,y_i) + \gamma c_i$, where the solution is unique by the strict convexity of $\phi_i$. We solve for $\gamma$ by setting the derivative to zero to obtain (see \cite{lim2016efficient} for more details):
\begin{equation} \label{pool}
    \sum_{i \in S} (\nabla{\phi}^{-1})(\gamma + \nabla{\phi}(y_i)) = \sum_{i\in S} c_i.
\end{equation}

Consider the case when $\phi(x) = \frac{1}{2}\|x\|^2$ so that our Bregman projection becomes a Euclidean projection. In this case, we have $\nabla{\phi}(x) = x = (\nabla{\phi})^{-1}(x)$ and \eqref{pool} reduces to computing an average: $\mathrm{Pool}_{\phi,y,c}(S) = \sum_{i\in S}(c_i - y_i)/|S|$
. On the other hand, when $\phi(x) =x \ln x -x $ so that our Bregman projection becomes the  generalized KL-divergence, we have $\mathrm{Pool}_{\phi,y,c}(S) = \ln \frac{\sum_{i\in S}c_i}{ \sum_{i\in S}z_i}$. Henceforth, we assume that the $\mathrm{Pool}_{\phi,y,c}$ operation can be done in $O(1)$ time using oracle access (which is a valid assumption for most widely-used mirror maps). We have thus arrived at the following result which gives the correctness and running time of the PAV algorithm:
\begin{restatable}{theorem}{pav_thm}
\label{pav thm}
 Let $f : 2^E \to \mathbb{R}$ be a cardinality-based submodular function, that is $f(S) = g(|S|)$ function for some concave function $g$. Let $\phi: \mathcal{D} \to \mathbb{R}$ be a strictly convex and uniformly seperable mirror map, where $B(f) \cap \mathcal{D} \neq \emptyset$. Then the output of the PAV algorithm (given in Algorithm \ref{alg:PAV}) is $x^*  = \argmin_{x \in B(f)} D_\phi(x,y)$. Moreover, the running-time of the algorithm is $O(n \log n + nEO)$.
\end{restatable}
\begin{proof}
The proof of this result follows from the fact that we need to sort $y$ in Theorem \ref{fenchel duality form} (which could be done in $O(n \log n)$ time)  and the fact that the PAV algorithm solves the dual problem exactly in $n$ iterations using Theorem 2.5 by Best. al \cite{best2000minimizing}, where each iteration takes $O(1)$ time. 
\hfill $\square$ \end{proof}

To explain the algorithm further and see it at work, consider the following example. Suppose we want to compute the Euclidean projection of $y = (4.8,4.6,2.7)$ onto the 1-simplex defined over the ground set $E = \{1,2,3\}$ {with the cardinality-based set function  $f(S) = \min \{|S|, 1\}$}. In this case we have $c_1 := 1$ and $c_i = 0$ for all $ i \in \{2,\dots,n\}$. The PAV algorithm initializes $z^{(0)} = c - y = (-3.8, -4.6, -2.7)$ using \eqref{pool}. Since $z_1 > z_2$, in the first iteration the algorithm will pool the first two coordinates by averaging them to obtain $z^{(1)} = c - y = (-4.2, -4.2, -2.7)$. Now we have $z^{(1)}_1 \leq z^{(1)}_2 \leq z^{(1)}_3$ and the algorithm terminates. Moreover, we recover the primal optimal solution using $x^* = z^{(1)}  + y = (0.6, 0.4, 0)$.

\begin{algorithm}[t] \small
\caption{Pool Adjacent Violators (PAV) Algorithm}
\label{alg:PAV}
\begin{algorithmic}[1]
\small
\INPUT Cardinality-based submodular function $f(S) = g(|S|):2^E \to \mathbb{R}$, strictly convex and uniformly separable mirror map $\phi: \mathcal{D} \to \mathbb{R}$ such that $B(f) \cap \phi \neq \emptyset$, and point to be projected $y \in \mathbb{R}^n$ where $y_1 \geq \dots \geq y_n$. 
\State Initialize $P \leftarrow \{{i} | i \in [E]\}$ and $z_i \leftarrow \mathrm{Pool}_{\phi,y,c}(i)$ for all $i \in [E]$.
\While{$\exists$ indices $i, i + 1 \in \mathcal{P}$ where $z_i > z_{i+1}$}
\State Let $K(i)$ and $K(i+1)$ be the intervals in $\mathcal{P}$ containing
indices $i$ and $i + 1$ respectively.
\State Remove $K(i), K(i+1)$ from $\mathcal{P}$ and add $K(i) \cup K(i+1)$.
\State set $z_{K(i) \cup K(i+1)} \leftarrow \mathrm{Pool}_{\phi,y,c} (K(i) \cup K(i+1))$. \Comment{\texttt{see equation \eqref{pool}}}
\EndWhile
\State Set $x^* \leftarrow \nabla{\phi}^{-1}(z + \nabla{\phi}(y))$ \Comment{\texttt{recover primal solution}}
\RETURN $x^*  = \argmin_{x \in B(f)} D_\phi(x,y)$. 
\end{algorithmic}
\end{algorithm}

%% file: proofs-sec-3.tex
\subsection{Missing proofs in Section \ref{sec: recovering from previous}} \label{app:missing in section 3.1}

\subsubsection{Proof of Theorem \ref{lemma: distant-gradients}}

\distantgradients*

\begin{proof} We show a more general result for uniformly separable divergences based on an $L$-smooth and strictly convex mirror map  $\phi$, so that the corresponding Bregman projection is nonexpansive, i.e., if $\|y-\tilde{y}\|\leq \epsilon$ then $\|x-\tilde{x}\| \leq \epsilon$. Let $F_1,F_2,...,F_k$ be a partition of $E$ such that $\nabla D_{\phi}(x,y)_e = c_i$ for all $e\in F_i$ and $c_i<c_l$ for $i < l$. {We now show that} if $c_{j+1}- c_j>4\epsilon L $ for some  $j\in [k-1]$, then the set $S=F_1 \cup \hdots \cup F_j$ is also a tight set for $\tilde{x}$. Let $\nabla D_\phi (x,y) = g$ and $\nabla D_{\phi}(\tilde{x},\tilde{y}) = \tilde{g}$ for brevity.

Let $e_j, e_{j+1} \in E$ be such that $g(e_j) = c_j$ and $g(e_{j+1}) = c_{j + 1}$. Consider the set of elements $S = \{e_1, \hdots, e_k\}$ that have a partial derivative at $x$ of value at most $c_j$, i.e., $S_j = \{e_i | g(e_i) \leq c_j\}$. Let $\tilde{C}_j := \{\tilde{g}(e_i) : e_i \in S_j\}$ and let $\tilde{C} := \{\tilde{g}(e) : e \in E\}$. Then, we will show every element of the set $\tilde{C}_j$ is smaller than every element of the set $\tilde{C} \setminus \tilde{C}_j$, by showing that $\max \tilde{C}_j \le \min \tilde{C} \setminus \tilde{C}_j$.

For any $e \in E$, consider $i$ such that $g(e) = c_{i}$. Then,
    \begin{align*}
        |\tilde{g}(e) - c_{i}| &= |\tilde{g}(e) - {g}(e) | \le \|\tilde{g} - g\|_\infty \le \|\tilde{g} - g\|_2 \\
        &= \|(\nabla \phi(\tilde{x}) -  \nabla \phi (\tilde{y})) - (\nabla \phi(x) - \nabla \phi (y))\|_2 \\
        &\le \|\nabla \phi(\tilde{x}) -  \nabla \phi (x)\|_2 + \|\nabla \phi(y) - \nabla \phi (\tilde{y})\|_2 \\
        &\le L \|\tilde{x} - x \|_2 + L \|\tilde{y} - y \|_2 \\
        &< 2L \epsilon.
    \end{align*}
    We use the result on gradient of Bregman projections for the first equality. The third inequality uses the triangle inequality, the fourth inequality uses $L$-smoothness, and the fifth inequality uses the non-expansiveness of the Euclidean (or Bregman) projection. 
    
    Therefore, if $e$ is such that $g(e) = c_{i} \le c_j$,
    \[
        \tilde{g}(e) < c_{i} + 2 L \epsilon \le c_j + 2 L \epsilon < c_{j + 1} - 2 L \epsilon < \tilde{g}(e_{j + 1}).
    \]
    The first and last inequalities follow from the inequality we established above, and the second and third inequalities follow by assumption. Similarly, if $i$ is such that $c_{i} > c_j$, then $\tilde{g}(e) \ge \tilde{g}(e_{j + 1})$.
    
    This implies the following: every element of the set $\tilde{C}_j = \{\tilde{g}(e) : e \in S \}$ is smaller than every element of $\tilde{C} \setminus \tilde{C}_j$ as claimed. {Since $\phi$ is $L$-smooth (and thus continuously differentiable) and strictly convex, the result then follows using Theorem \ref{foo for base}}. 
  \end{proof}

\subsubsection{Proof of Theorem \ref{lemma: close-point rounding}}

\closepointrounding*

\begin{proof}
    The proof of this theorem utilizes the same ideas as those in the proof of {Theorem} \ref{lemma: distant-gradients}. Consider elements $e_j, e_{j+1} \in E$ be such that $\nabla h(z)(e_j)  = \tilde{c}_j$ and $\nabla h(z)(e_{j + 1}) = \tilde{c}_{j + 1}$.
    
    Let $S_j$ be the set of elements at which $z$ has a partial derivative at most $c_j$. Let $C_j$ be the partial derivative values at $x$ at $S_j$, i.e., $C_j := \{\nabla h(x)_e: e \in S_j \}$ and let $C := \{\nabla h(x)_e : e \in E\}$. Then, we'll show that $\max {C_j} \le \min {C \setminus C_j}$.

    For each $e \in E$, there is an $i$ such that $\nabla h(z)_e = \tilde{c}_{i}$. Then, {using the $L$-smoothness of $h$ we have}
    \begin{equation}\label{eq: gradient-difference-same-y}
        |\nabla h(x)(e) - \tilde{c}_{i}| = |\nabla h(x)(e) - \nabla h(z)(e)|  \le L \|x -  z\|_2 < L \epsilon.
    \end{equation}
    
    Therefore, for any $e$ such that $\tilde{c}_{i} \le \tilde{c}_j$,
    \[
        \nabla h(x)(e) < \tilde{c}_{i} + L \epsilon \le \tilde{c}_j + L\epsilon < \tilde{c}_{j + 1} - L \epsilon < \nabla h(x)(e_{j + 1}).
    \]
    The first and last inequalities follow from the inequality established above, and the second and third inequalities follow by definition.
    
    Similarly, if $e$ is such that $\tilde{c}_i > \tilde{c}_j$, then $\nabla h(x)(e) \ge \nabla h(x)(e_{j + 1})$. Together, these imply the following: every element of the set $C_j = \{\nabla h(x)(e): e \in C_j \}$ is smaller than every element of $C \setminus C_j$. Since $h$ is $L$-smooth (and thus continuously differentiable) and strictly convex, the result then follows using Theorem \ref{foo for base}.
  \end{proof}

\subsection{Missing proofs in Section \ref{sec: reuse and restrict}} \label{app:missing in section 3.2}

\subsubsection{Proof of Lemma \ref{lemma: reusing active sets}}
To prove this lemma we first need the following result, which states for any $x$ in a polytope $\mathcal{P}$, a vertex in an active set for $x$ must like on the minimal face containing $x$:
\begin{lemma} \label{best away vertex}
Let $\mathcal{P} = \{x \in \mathbb{R}^n: Ax \leq b\}$ be a polytope with vertex set $\mathcal{V}(\mathcal{P})$. Consider any ${x} \in \mathcal{P}$. Let $I$ denote the index-set of active constraints at $x$ and $F = \{ x \in \mathcal{P} \mid A_{I} x =  b_I\}$ be the minimal face containing $x$. Let $\mathcal{A}(x) := \{S : S \subseteq \mathcal{V}(\mathcal{P}) \mid  \text{$x$ is a proper convex combination of all the elements in $S$}\}$ be the set of all possible active sets for $x$, and define $\mathcal{A} := \cup_{A \in \mathcal{A}(x)}A$. Then, we claim that $\mathcal{A} = \mathcal{V}(F)$.
\end{lemma}

\begin{proof}
We first prove that $\mathcal{A} \subseteq \mathcal{V}(F)$ by showing that any $A \in \mathcal{A}(x)$ must be contained in $\mathcal{V}(F)$. Write $x = \sum_{v \in A} \lambda_v v$. Consider any $A \in \mathcal{A}(x)$ and fix a vertex $y \in A$ arbitrarily. Define $z := \frac{1}{1- \lambda_y} \sum_{v \in A \setminus\{y\}} \lambda_{v} v \in \mathrm{Conv}(A)$ to be the point obtained by shifting the weight from $y$ to other vertices in $A \setminus \{y\}$. Then, we can write $x = \lambda_y y + (1 - \lambda_y) z$. Thus, if $\innerprod{{A}_i}{x} = b_i$, then we also have $\innerprod{{A}_i}{y} = b_i$, so that $y \in \mathcal{V}(F)$. 

elaTo show the reverse inclusion, we consider any $v \in \mathcal{V}(F)$ and will construct an active set $A \in \mathcal{A}(x)$ containing $v$. 
Define $\lambda^* \coloneqq \max \{\lambda \mid x + \lambda (x - v) \in F\}$ to be the maximum movement from $x$ in the direction $x-v$. Note  $\lambda^* > 0$ since $x$ is in the relative interior of $F$ by definition. Let $z \coloneqq x + \lambda^* (x - v) \in F$ to be the point obtained by moving maximally from $x$ along the direction $x - v$. Now observe that $(i)$ we can write $x$ as a proper convex combination of $z$ and $v$: $x = \frac{1}{1 + \lambda^*} z + \frac{\lambda^*}{1 + \lambda^*}v$; $(ii)$ the point $z$ lies in a lower dimensional face $\tilde{F} \subset F$ since it is obtained by line-search for feasibility in $F$. Letting $\tilde{A}$ be any active set for $z$ (where $\tilde{A} \subseteq \mathcal{V}(\tilde{F})$ by the first part of the proof), we have that $\tilde{A} \cup \{v\}$ is an active set for $x$.
  \end{proof}

We are now ready to prove our lemma:

\activesets*

\begin{proof}
    Consider any $\tilde{y} \in \mathbb{B}_\Delta(y)$ and let $\tilde{x}$ be its Euclidean projection. By the previous lemma we have that $\mathrm{Conv}(\mathcal{A}) \subseteq F$. Using non-expansiveness of projection operator we have that $\|x - \tilde{x}\| \leq \Delta$. Moreover, using Lemma \ref{face cond} we know that $\tilde{x}$ lies in $F$. Thus, since $\|x - \tilde{x}\| \leq \Delta$, we have that $\tilde{x} \in \mathrm{Conv}(\mathcal{A})$.
  \end{proof}

\subsubsection{Proof of Theorem \ref{lo over faces}}

\looverface*

\begin{proof}
    Our proof is an extension of the proof of the greedy algorithm by Edmonds \cite{Edmonds1971}. We follow the notation {given in Algorithm \ref{alg:greedy}}. The linear programming primal and dual pairs for our problem are:
    \begin{equation} \label{primal greedy}
        \begin{aligned}
        \max &~ \innerprod{c}{x} \\
        \text{s.t.} &~ x(T) \le h(T)  &\quad \forall \; T \subset E, \\
        &x(S_i) = h(S_i)  &\quad \forall \; S_i \in \mathcal{S}, \\
        &x(E) = h(E).
    \end{aligned}
    \qquad \qquad
    \begin{aligned}
        \min_{y} & \sum_{T \subseteq E} y_T h(T) \\
        \text{s.t.}&~ \sum_{T \ni e_j} y(T) = c(e_j)  &&\quad \forall \; j \in [1, n], \\
        &y_T \ge 0 &&\quad \forall \; T \not\in \mathcal{S} \cup \{E\}.
    \end{aligned}
    \end{equation}
    
    Define $U_j := \{e_1, \ldots, e_j\}$ (that is, the first $j$ elements of the order we have induced in the algorithm). Define $y^*$ as:
    \begin{align*}
        y^*_{U_j} &= c(e_j) - c(e_{j + 1})  & \forall \; j \in [1, n-1], \\
        y^*_{U_n} &= c(e_n), \\
        y^*_{T} &= 0  & \forall \; T \subseteq E: T \not\in \{U_0, \ldots, U_n\}.
    \end{align*}
    We will now show that $y^*$ is such that $\sum_{T \subseteq E} y^*_T h(T) = \innerprod{c}{x^*}$ (and that $y^*$ and $x^*$ are feasible), so optimality is implied by strong duality. 
    
    Note that $\sum_{T \ni e_j} y^*_T = \sum_{\ell \in [j, n]} y^*_{U_\ell} = c(e_j)$. When $T \not\in \{U_1, \ldots, U_n\}$, $y^*_T \ge 0$ trivially. For each $j$, when $U_j \not\in \mathcal{S}$, $y^*_{U_j} \ge 0$ by definition of the order we have induced on $E$. Therefore $y^*$ is feasible.
    
    The feasibility of $x^*$ is essentially the same as in the proof of the greedy algorithm for $B(f)$. We show that $x^*(T) \le f(T)$ for all $T \subseteq E$. We use induction on $T$. When $|T| = 0$, $T = \emptyset$ and $x(T) = f(T) = 0$. Assume now that $|T| > 0$, and let $e_j$ be the element of $T$ with the largest index. Then,
    \[
        x(T) = x(T \setminus \{e_j\}) + x(e_j)
        \le f(T \setminus \{e_j\}) + x(e_j)
        = f(T \setminus \{e_j\}) + f(U_j) - f(U_{j - 1})
        \le f(T).
    \]
    The first inequality follows from the induction hypothesis, and the last follows by submodularity. The equalities follow by definition. Finally, a straightforward calculation verifies that
    \begin{align*}
        \sum_{T \subseteq E} y^*_T h(T) &= \sum_{i=1}^n y_{U_i} h(U_i) =  \sum_{i=1}^{n-1} (c(e_j) + c(e_{j+1}))h(U_i) + c(e_j) h(U_n)\\
        &= \sum_{i = 1}^{n}c(e_i) (h(U_i) - h(U_{i - 1})) = \innerprod{c}{x^*},
    \end{align*}
    which proves our claim. 
  \end{proof}

\subsection{Missing proofs in Section \ref{sec: rounding}} \label{app:missing in section 3.3}

\subsubsection{Proof of Lemma \ref{lemma: comb-point rounding}}

\combpointrounding*
\begin{proof}
Let $\mathcal{S}^*$ be the set of all tight sets at $x^*$. If the optimal face is known, then we can restrict our original optimization problem to that optimal face by Lemma \ref{optimal face}, that is $x^* = \argmin \{h(x) \mid x(S) = h(S)\; \forall S \in \mathcal{S}^*\}$, which proves the last statement of the lemma. Since the feasible region
$\{x \mid x(S) = h(S)\; \forall S \in \mathcal{S}\}$ used to obtain $\tilde{x}$ contains the optimal face, i.e. 
$\{x \mid x(S) = h(S)\; \forall S \in \mathcal{S}^*\} \subseteq \{x \mid x(S) = h(S)\; \forall S \in \mathcal{S}\}$, it follows that $h(\tilde{x}) \leq h(x^*)$. Thus, if $\tilde{x} \in B(f)$, we must have $\tilde{x} = x^*$, otherwise we contradict the optimality of $x^*$ by the strict convexity of $h$. Conversely, if $\tilde{x} = x^*$, then we trivially have $\tilde{x} \in B(f)$.
  \end{proof}

\subsubsection{Proof of Lemma \ref{lemma: Integer_euclidean_rounding}}

\Euclideanrounding*
\begin{proof} 
    For brevity, denote $|E| = n$. First, if we are given all the tight sets at the optimal solution as defined in Theorem \ref{foo for base}, then we can recover the Bregman projection $\argmin_{x\in B(f)} \sum_{e} h_e(x_e)$ by solving the following univariate equation:
    {\begin{equation} \label{recovering projection}\sum_{e \in F_i}{(\nabla h_e)^{-1}(c_i)} = {f(F_1 \cup \dots  \cup F_i) - f(F_1 \cup \dots  \cup F_{i-1})} \qquad \forall \; i \in [k].\end{equation}}
    
    Using equation \eqref{recovering projection} {and noting that $(\nabla h_e)^{-1}(c_i) = c_i + y_e$ for all $e \in F_i$}, we have for each $e \in F_i$,

    \[
        x_e = \frac{f(\cup_{j \in [i]} F_i) - f(\cup_{j \in [i - 1]}) - y(F_i)}{|F_i|} + y_e.
    \]
    Since $f, y$ are integral, we have $x_e \in Q$ for all $e \in E$. Further, note that
    \[
        \min_{x, y \in Q, x \neq y} |x - y| = \min_{\ell_1, \ell_2 \in [n], k_1 \ell_2 \neq k_2 \ell_1} \Big| \frac{k_1}{\ell_1} - \frac{k_2}{\ell_2} \Big| = \min_{\ell_1, \ell_2 \in [n], k_1 \ell_2 \neq k_2 \ell_1} \frac{|k_1 \ell_2 - k_2 \ell_1|}{\ell_1 \ell_2} \ge   
        \frac{1}{n^2}.
    \]
    Therefore, there is a unique element of $Q$ that is within a distance of less than $\frac{1}{2n^2}$ from $x^\ast_e$. But by assumption, we have $|x_e - x^\ast_e | \le \|x - x^\ast \|_2 < \frac{1}{2n^2}$ for all $e \in E$, which implies that $\argmin_{s \in Q}|x_e - s|$ is singleton, so that the rounding can be done uniquely. Further, note that for all $r \in \mathbb{R}$,
    \[
        \min_{s \in Q}|r - s| = \min_{k \in [n]} \min_{s \in \frac{1}{k} \mathbb{Z}} |r - s| = \min_{k \in [n]} \min_{t \in \mathbb{Z}} |k \cdot r - t|,
    \]
    which implies the correctness of the algorithm.
  \end{proof}

%% file: AAFW_proof.tex
\subsection{Proof of Theorem \ref{afw conv}}
The proof of convergence for \AAFW~ follows simply from the iteration-wise convergence rate of Lacoste-Julien and Jaggi \cite{Lacoste2015}, and properties of convex minimizers over submodular polytopes. Once we detect a tight inequality, we can restrict the feasible region to a smaller face of the polytope. Since this happens only a linear number of times, we get linear convergence with \AAFW~ as well. 
To prove Theorem \ref{afw conv}, we need the following result:
\begin{restatable}{theorem}{pyramidal}[Extension of Theorem 3 in \cite{Lacoste2015}] \label{restricted Pyramidal conv}
Let $\mathcal{P}\subseteq \mathbb{R}^n$ be a polytope. Consider any strongly convex and smooth function $h: \mathcal{P} \to \mathbb{R}$. Further, Consider any suboptimal iterate $z^{(t)}$ of the \AAFW~ algorithm, and let $\mathcal{A}_t$ be its active set and $K$ be its minimal face. Let $x^* \coloneqq \argmin_{x \in \mathcal{P}} h(x)$ and $F$ be a face containing $x^*$ such that $F \supseteq K$. Further, denote $r \coloneqq - \nabla h(z^{(t)})$ and $\hat{e} \coloneqq {(z^{(t)} - x^*)}/{\|{z^{(t)} - x^*}\|}$. Define the pairwise FW direction at iteration $t$ to be $d_t^{\text{PFW}} \coloneqq v^{(t)} - a^{(t)}$, where recall that $v^{(t)}  = \argmax_{v \in F} \innerprod{{r}}{v - z^{(t)}}$ and $a^{(t)}  = \argmax_{a \in \mathcal{A}_t} \innerprod{{r}}{z^{(t)} - a}$. Then, we have
\begin{equation} \label{pyramidal progress}
   \frac{\innerprod{{r}}{d_t^{\text{PFW}}}}{\innerprod{r}{\hat{e}}} \geq \rho_{F},
\end{equation}
where $\rho_{F}$ is the pyramidal width of $\mathcal{P}$ restricted to $F$ as defined in $\eqref{facial1}$.
\end{restatable}

The proof of this result follows can be directly obtained by applying Theorem 3 in \cite{Lacoste2015} to the face $F$ instead of the whole polytope (since both $x^*$ and $z^{(t)}$ lie in $F$ and we are doing LO over $F$)

We are now ready to prove our convergence, which we state here for convenience:

\afwconver*


\begin{proof}
Recall that in the \AAFW we either take the FW direction $d_t = v^{(t)} - z^{(t)}$ or the {away} direction $d_t = z^{(t)} - a^{(t)}$ depending on which direction has a higher inner product with $-\nabla h(z^{(t)})$. Defining $d_t^{\text{PFW}} \coloneqq v^{(t)} - a^{(t)}$ to be the \emph{pairwise} FW direction, this implies the following key inequality
\begin{equation}\label{pair-wise}
    2 \innerprod {-\nabla h(z^{(t)})}{d_t} \geq \innerprod {-\nabla h(z^{(t)})}{v^{(t)} - z^{(t)}} + \innerprod{-\nabla h(z^{(t)})}{z^{(t)} - a^{(t)}} = \innerprod {-\nabla h(z^{(t)})}{d_t^{\text{PFW}}}.
\end{equation}

We proceed by cases depending on whether the step size {chosen by line search} is maximal or not, i.e. whether $\gamma_t = \gamma_t^{\max}$ or not:
\begin{enumerate} [leftmargin = 10pt]
    \item[] \textbf{Case 1:} \textit{The step size evaluated from line-search is not maximal, i.e.$\gamma_t < \gamma_t^{\max}$ so that we have `good' step.}  Note that the optimal solution of the line-search step is in the interior of the interval $[0, \gamma_t^{\max}]$. Define $z_\gamma \coloneqq z^{(t)} + \gamma d_t$. Then, because $h(z_\gamma)$ is convex in $\gamma$, we know that $\min_{\gamma \in [0,\gamma_t^{\max}]} h(z_\gamma) = \min_{\gamma \geq 0} h(z_\gamma)$ and thus $\min_{\gamma \in [0,\gamma_t^{\max}]} h(z_\gamma) = h(z^{(t+1)}) \leq h(z_\gamma)$ for all $ \gamma \geq 0$. In particular, if we define $\gamma_{d_t} \coloneqq \frac{\innerprod{-\nabla h({z}^{(t)})}{d_t}} {L \|{d}_t\|^2} \geq 0$ (non-negativity follows from $d_t$ being a descent direction), then we have $h(z^{(t+1)}) \leq h(z_{\gamma_{d_t}})$. Therefore,
    \begin{align}
    w(z^{(t)}) - w(z^{(t+1)}) &= h(z^{(t)}) - h(z^{(t+1)}) \geq  h(z^{(t)}) - h(z_{\gamma_{d_t}}) \nonumber\\
    & \geq \innerprod{-\nabla h(z^{(t)})}{z_{\gamma_{d_t}} - z^{(t)}} - \frac{L}{2} \|z_{\gamma_{d_t}} - z^{(t)}\|^2 \label{conv proof eq 2}\\
        & =\frac{\innerprod {-\nabla h(z^{(t)})}{d_t}^2} {2L D^2}  \label{conv proof eq 3}\\
         & \geq \frac{\innerprod {-\nabla h(z^{(t)})}{d_t^{\text{PW}}}^2} {8L D^2} \label{conv proof eq 4} \\
        & \geq \frac{\rho_{F(\mathcal{S})}}{8L D^2}  \frac{\innerprod {-\nabla h(z^{(t)})}{x^*-z^{(t)}}^2} { \|x^*-z^{(t)}\|^2}  \label{conv proof eq 5} \\
    &\geq \left(\frac{\rho_{F(\mathcal{S})}}{D}\right)^2 \frac{\mu}{4L} w(z^{(t)})  & \label{conv proof eq 6}.
\end{align}
We used the optimized smoothness inequality in \eqref{conv proof eq 2}, and the fact that of $z_{\gamma_{d_t}} = z^{(t)} + \gamma_{d_t} d_t$ in \eqref{conv proof eq 3}. The inequality in \eqref{conv proof eq 4} uses our key pairwise inequality \eqref{pair-wise}. In \eqref{conv proof eq 5}, we used the fact that $x^*, z^{(t)} \in F(\mathcal{S})$ by construction since $t$ is not a rounding iteration and away steps are in-face steps by Lemma \ref{best away vertex}. In other words, the minimal face $K$ containing $z^{(t)}$ satisfies $K \subseteq F(\mathcal{S})$. Thus, we can apply Theorem \ref{restricted Pyramidal conv} to go from \eqref{conv proof eq 4} to \eqref{conv proof eq 5}. Finally, \eqref{conv proof eq 6} uses the fact that since $h$ is strongly convex it satisfies $w(z^{(t)}) = h(z^{(t)}) - h(x^*) \leq \frac{\innerprod {-\nabla h(z^{(t)})}{x^*-z^{(t)}}^2} {2\mu \|x^*-z^{(t)}\|^2}$
. This shows the rate we claimed.
\item[] \textbf{Case 2:}  \textit{We have a boundary case: $\gamma_{t} = \gamma_t^{\max}$}. We further divide this case into two sub-cases:
 \begin{enumerate}
        \item[(a)] First assume that $\gamma_{t} = \gamma_t^{\max}$ and we take a FW step, i.e. $d_t = v^{(t)} - z^{(t)}$. In this case we have $\gamma_t^{\max} = 1$, and hence $z^{(t+1)} = z^{(t)} + d_t$. We can assume that the step size $\gamma_{d_t}$ is not feasible, i.e. $\gamma_{d_t} > \gamma_t^{\max}$ since otherwise we can use using same argument as above in Case 1 to again obtain a $(1 - \left(\frac{\rho_{F(\mathcal{S})}}{D}\right)^2 \frac{\mu}{4L})$-geometric rate of decrease. Observe that $\gamma_{d_t} = \frac{\innerprod {-\nabla h(z^{(t)})}{d_t}}{L \|d_t  \|^2} > \gamma_t^{\max} = 1$ implies that $\innerprod  {-\nabla h(z^{(t)})}{d_t} \geq L\|d_t  \|_2^2$. Hence, using the smoothness inequality we have
        \begin{align*}
            w(z^{(t)}) - w(z^{(t+1)}) &= h(z^{(t)}) - h(z^{(t+1)}) \geq  \innerprod  {-\nabla h(z^{(t)})}{d_t  } - \frac{L}{2} \|d_t  \|_2^2  && \text{(using $z^{(t+1)} = z^{(t)} + d_t$}) \\
            & \geq \frac{\innerprod {-\nabla h(z^{(t)})}{d_t}}{2}  && \text{(using $\gamma_t > \gamma_{d_t}^{\max} =1$)}\\
            & \geq \frac{\innerprod {-\nabla h(z^{(t)})}{v^{(t)} - z^{(t)}}}{2}  && \text{(using $d_t = v^{(t)} - z^{(t)}$)}\\
            & \geq \frac{w(z^{(t)})}{2} && \text{(using \eqref{strong wolfe})}.
        \end{align*}
        Hence, we get a geometric rate of decrease of 1/2. 
        \item[(b)] Finally, assume that $\gamma_{t} = \gamma_t^{\max}$ and we take an away step, i.e. $d_t = z^{(t)} - a^{(t)}$. In this case, we cannot show sufficient progress. However, it is known that the number of these drop steps (which we denote by ${Drop}_t$) can happen at most $t/2$ times up to iteration $t$ \cite{Lacoste2015}, i.e., $Drop_t\leq \frac{t}{2}$.
\end{enumerate}
\end{enumerate}

Note that $(i)$ $\rho_{B(f)} \leq \rho_{F(\mathcal{S})}$ for any chain $\mathcal{S}$ since $F(\mathcal{S}) \subseteq B(f)$; $(ii)$ anytime we restart the algorithm, we do so at a vertex of $B(f)$ and thus the increase in the primal gap resulting from the restart is bounded as $h$ is finite over $B(f)$. Thus, since ${Drop}_t\leq \frac{t}{2}$, and the number of rounding steps is at most $n$ (as the length of any chain of tight sets at $x^*$ is at most {$n$}), we have that the number of iterations to get an $\epsilon$-accurate solution is $O\left(n \frac{L}{\mu} \left( \frac{D}{\rho_{B(f)}} \right)^2 \log \frac{1}{\epsilon}\right)$ in the worst case. 
   \end{proof}

\subsection{Rounding to Optimal Face Using Euclidean Projections}\label{app: rounding}

We now give a different rounding approach that results in better constants in the linear convergence rate, but is computationally more expensive. Consider rounding ${z}^{(t)}$ to $F(\mathcal{S}_{new})$ by computing the Euclidean projection ${z}^{(t)}$ onto $F(\mathcal{S}_{new})$, i.e. ${z}^{(t+1)} \coloneqq \argmin_{x \in F(\mathcal{S}_{new})}\|x - {z}^{(t)}\|^2$, instead of restarting from a vertex in $F(\mathcal{S}_{new})$. Since $x^* \in F(\mathcal{S}_{new})$, by non-expansiveness of the Euclidean projection operator we have 
\begin{equation}\label{non expansive}
    \|{z}^{(t+1)} - x^*\| \leq \|{z}^{(t)} - x^*\|.
\end{equation}
We further \emph{explicitly} bound the increase in primal gap after the rounding step. 
\begin{restatable}{lemma}{afwround} \label{lemma: rounding}
Let $f:2^E \to \mathbb{R}$ be a monotone submodular function with $f(\emptyset) = 0$ and $f$ monotone. Consider any 
{function $h
: B(f) \to \mathbb{R}$} that is $\mu$-strongly convex and $L$-smooth. Let $x^* \coloneqq \argmin_{x \in B(f)} h(x)$. Let $t$ be an iteration of \AAFW~ in which we have discovered a new chain of tight sets $\mathcal{S}_{new}$ using an iterate $z^{(t)}$, and $F(\mathcal{S}_{new})$ be the face defined by these tight sets. Let the next iterate $z^{(t+1)}$ be obtained by projecting ${z}^{(t)}$ to $F(\mathcal{S}_{new})$, i.e. ${z}^{(t+1)} \coloneqq \argmin_{x \in F(\mathcal{S}_{new})}\|x - {z}^{(t)}\|^2$. Let $c$ be an upper bound on the norm of the gradient at optimum, i.e., $c \geq \|\nabla h(x^*)\|^2$. Then, the primal gap  $w(z^{(t+1)}): = h(z^{(t+1)}) - h({x}^*)$ satisfies
$$w({z}^{(t+1)}) \leq \frac{2L}{\mu^2} w({z}^{(t)}) + \frac{c}{\mu}.$$
\end{restatable}

\begin{proof}
Using the strong convexity of $h$ we have 
\begin{align}
    w({z}^{(t)}) &\geq \innerprod{\nabla h(x^*)}{{z}^{(t)} - x^*} + \frac{\mu}{2}\|{z}^{(t)} - x^*\|^2 \label{primal restart eq 1}\\
    &\geq \frac{\mu}{2} \|{z}^{(t+1)} - x^*\|^2 \label{primal restart eq 2}\\
    &\geq \frac{\mu}{2L} \|\nabla h({z}^{(t+1)})- \nabla h(x^*)\|^2 \label{primal restart eq 3}\\
    &= \frac{\mu}{2L}\left( \|\nabla h({z}^{(t+1)})\|^2 + \| \nabla h(x^*)\|^2 - 2\innerprod{\nabla h(x^*)}{\nabla h({z}^{(t+1)}}\right) \label{primal restart eq 4}\\
    &\geq \frac{\mu}{2L}\left( \frac{1}{2}\|\nabla h({z}^{(t+1)})\|^2 - \| \nabla h(x^*)\|^2\right) \label{primal restart eq 5}\\
    &\geq \frac{\mu}{2L}(\mu w({z}^{(t+1)}) - c). \label{primal restart eq 6}
\end{align}
We used the fact that $\innerprod{\nabla h(x^*)}{{z}^{(t)} - x^*} \geq 0$ by first order optimality and the fact that $\|{z}^{(t+1)} - x^*\| \leq \|{z}^{(t)} - x^*\|$ (see \eqref{non expansive}) in \eqref{primal restart eq 2}. The inequality in \eqref{primal restart eq 3} follows by the smoothness of $h$. Further, we used Young’s inequality\footnote{Young’s inequality states that for any two vectors ${a},{b} \in \mathbb{R}^n$ we have $2 \innerprod{{a}}{{b}} \leq w \|{a} \|^2 + \frac{1}{w}\|{b}\|^2$ for any scalar $w \in \mathbb{R}_{++}$.} in \eqref{primal restart eq 5}. Finally, we used the \emph{PL-inequality} (which states that $w(z^{(t+1)}) \leq {\|\nabla h(z^{(t+1)})\|^2}/ {2\mu}$) and the fact that $c  \geq \|\nabla h(x^*)\|^2$ in \eqref{primal restart eq 6}. Rearranging the above inequality gives $w({z}^{(t+1)}) \leq \frac{2L}{\mu^2} w({z}^{(t)}) + \frac{c}{\mu}$as claimed.
   \end{proof}

Although this approach gives explicit bounds on the primal gap increase after rounding, it is computationally expensive as it requires the computation of a Euclidean projection.